\documentclass{article}

\usepackage{blindtext}
\usepackage{amsmath,amsthm,amssymb}
\newtheorem{theorem}{Theorem}[section]
\newtheorem{lemma}[theorem]{Lemma}

\newtheorem{definition}[theorem]{Definition}
\newtheorem{proposition}[theorem]{Proposition}
\newtheorem{remark}[theorem]{Remark}
\newtheorem{assumption}{Assumption}
\newtheorem{example}{Example}

\providecommand{\norm}[1]{\left\lVert#1\right\rVert}

\providecommand{\R}{\mathbb{R}} %

\providecommand{\E}{{\mathbb E}}
\providecommand{\E}[1]{{\mathbb E}\left.#1\right. }        %
\providecommand{\Eb}[1]{{\mathbb E}\left[#1\right] }       %
\providecommand{\EEb}[2]{{\mathbb E}_{#1}\left[#2\right] } %

\DeclareMathOperator*{\argmin}{arg\,min}

\providecommand{\vv}{\mathbf{v}}

\providecommand{\xx}{\mathbf{x}}
\providecommand{\yy}{\mathbf{y}}
\providecommand{\zz}{\mathbf{z}}

\providecommand{\mA}{\mathbf{A}}
\providecommand{\mB}{\mathbf{B}}
\providecommand{\mC}{\mathbf{C}}

\providecommand{\mH}{\mathbf{H}}
\providecommand{\mI}{\mathbf{I}}

\providecommand{\mU}{\mathbf{U}}

\providecommand{\mX}{\mathbf{X}}

\providecommand{\mxi}{\boldsymbol{\xi}}

\providecommand{\momega}{\boldsymbol{\omega}}
\providecommand{\mdelta}{\boldsymbol{\delta}}
\providecommand{\mvarsigma}{\boldsymbol{\varsigma}}
\providecommand{\mnu}{\boldsymbol{\nu}}

\providecommand{\cD}{\mathcal{D}}

\providecommand{\cN}{\mathcal{N}}
\providecommand{\cO}{\mathcal{O}}

\providecommand{\cR}{\mathcal{R}}
\providecommand{\cS}{\mathcal{S}}

\providecommand{\cX}{\mathcal{X}}

\newenvironment{talign}
{\align}
{\endalign}
\newenvironment{talign*}
{\csname align*\endcsname}
{\endalign}
\usepackage{algorithm}
\usepackage[noend]{algpseudocode}

\newcommand{\mycaptionof}[2]{\captionof{#1}{#2}}

\errorcontextlines\maxdimen

\makeatletter
\newcommand*{\algrule}[1][\algorithmicindent]{\makebox[#1][l]{\hspace*{.5em}\thealgruleextra\vrule height \thealgruleheight depth \thealgruledepth}}%
\newcommand*{\thealgruleextra}{}
\newcommand*{\thealgruleheight}{.75\baselineskip}
\newcommand*{\thealgruledepth}{.25\baselineskip}

\newcount\ALG@printindent@tempcnta
\def\ALG@printindent{%
	\ifnum \theALG@nested>0
	\ifx\ALG@text\ALG@x@notext
	\else
		\unskip
		\addvspace{-1pt}
		\ALG@printindent@tempcnta=1
		\loop
		\algrule[\csname ALG@ind@\the\ALG@printindent@tempcnta\endcsname]%
		\advance \ALG@printindent@tempcnta 1
		\ifnum \ALG@printindent@tempcnta<\numexpr\theALG@nested+1\relax
			\repeat
		\fi
	\fi
}%
\usepackage{etoolbox}
\patchcmd{\ALG@doentity}{\noindent\hskip\ALG@tlm}{\ALG@printindent}{}{\errmessage{failed to patch}}
\makeatother

\newbox\statebox
\newcommand{\myState}[1]{%
	\setbox\statebox=\vbox{#1}%
	\edef\thealgruleheight{\dimexpr \the\ht\statebox+1pt\relax}%
	\edef\thealgruledepth{\dimexpr \the\dp\statebox+1pt\relax}%
	\ifdim\thealgruleheight<.75\baselineskip
		\def\thealgruleheight{\dimexpr .75\baselineskip+1pt\relax}%
	\fi
	\ifdim\thealgruledepth<.25\baselineskip
		\def\thealgruledepth{\dimexpr .25\baselineskip+1pt\relax}%
	\fi
	\State #1%
	\def\thealgruleheight{\dimexpr .75\baselineskip+1pt\relax}%
	\def\thealgruledepth{\dimexpr .25\baselineskip+1pt\relax}%
}
\usepackage{PRIMEarxiv}

\usepackage{url}
\usepackage{hyperref}
\usepackage[utf8]{inputenc}
\usepackage[T1]{fontenc}
\usepackage{booktabs}
\usepackage{amsfonts}       %
\usepackage{nicefrac}       %
\usepackage{microtype}      %
\usepackage{xcolor}
\usepackage[flushleft]{threeparttable}
\usepackage{color}
\usepackage{mathtools}

\usepackage{url}
\usepackage{float}
\usepackage{subfigure}
\usepackage{graphicx}
\usepackage{multirow}
\usepackage{xspace}
\usepackage{natbib}
\usepackage{enumitem}
\usepackage[font=small]{caption}
\usepackage{diagbox}
\usepackage{wrapfig}
\usepackage{cases}
\usepackage[toc, page, header]{appendix}
\definecolor{myblue}{rgb}{0,0.45,0.74}
\definecolor{myred}{rgb}{0.85,0.33,0.1}

\usepackage{xspace}  
\usepackage{bm}

\newcommand{\algopt}{CFL\xspace} 
\newcommand{\fedavg}{FedAvg\xspace}

\usepackage{lastpage}

\title{Towards Federated Learning on Time-Evolving Heterogeneous Data}

\author{YongXin Guo~$^*$ \\
CUHK (SZ), P.R. China \\
\texttt{yongxinguo@link.cuhk.edu.cn}
\And
Tao Lin \thanks{Equal contribution.}\\
Westlake University, P.R. China \\
\texttt{lintao@westlake.edu.cn} \\
\And
Xiaoying Tang \\
CUHK (SZ), P.R. China \\
\texttt{xiaoyingtang@cuhk.edu.cn}
}

\begin{document}


\newcommand{\fix}{\marginpar{FIX}}
\newcommand{\new}{\marginpar{NEW}}

\maketitle

\begin{abstract}
	Federated Learning (FL) is a learning paradigm that protects privacy by keeping client data on edge devices. However, optimizing FL in practice can be difficult due to the diversity and heterogeneity of the learning system. Despite recent research efforts to improve the optimization of heterogeneous data, the impact of time-evolving heterogeneous data in real-world scenarios, such as changing client data or intermittent clients joining or leaving during training, has not been studied well.

	In this work, we propose Continual Federated Learning (CFL), a flexible framework for capturing the time-evolving heterogeneity of FL. CFL can handle complex and realistic scenarios, which are difficult to evaluate in previous FL formulations, by extracting information from past local data sets and approximating local objective functions. We theoretically demonstrate that CFL methods have a faster convergence rate than \fedavg in time-evolving scenarios, with the benefit depending on approximation quality. Through experiments, we show that our numerical findings match the convergence analysis and that CFL methods significantly outperform other state-of-the-art FL baselines.
\end{abstract}


\section{Introduction}
Federated Learning (FL) has recently emerged as a critical distributed machine learning paradigm to preserve user/client privacy.
Clients engaged in the training process of FL only communicate their local model parameters, rather than their private local data, with the central server.

As the workhorse algorithm in FL, \fedavg~\citep{mcmahan2017communication} performs multiple local stochastic gradient descent (SGD) updates on the available clients before communicating with the server.
Despite its success, \fedavg suffers from the large heterogeneity (non-iid-ness) in the data presented on the different clients, causing drift in each client's updates and resulting in slow and unstable convergence~\citep{karimireddy2020scaffold}.
To address this issue, a new line of study has been suggested lately that either simulates the distribution of the whole data set using preassigned weights of clients~\citep{wang2020tackling,ReisizadehFPJ20robust,MohriSS19agnostic,LiSBS20fair} or adopts variance reduction methods~\citep{karimireddy2020scaffold,karimireddy2020mime,das2020faster,haddadpour2021federated}.

\begin{figure}[!t]
	\centering
	\includegraphics[width=.375\textwidth,]{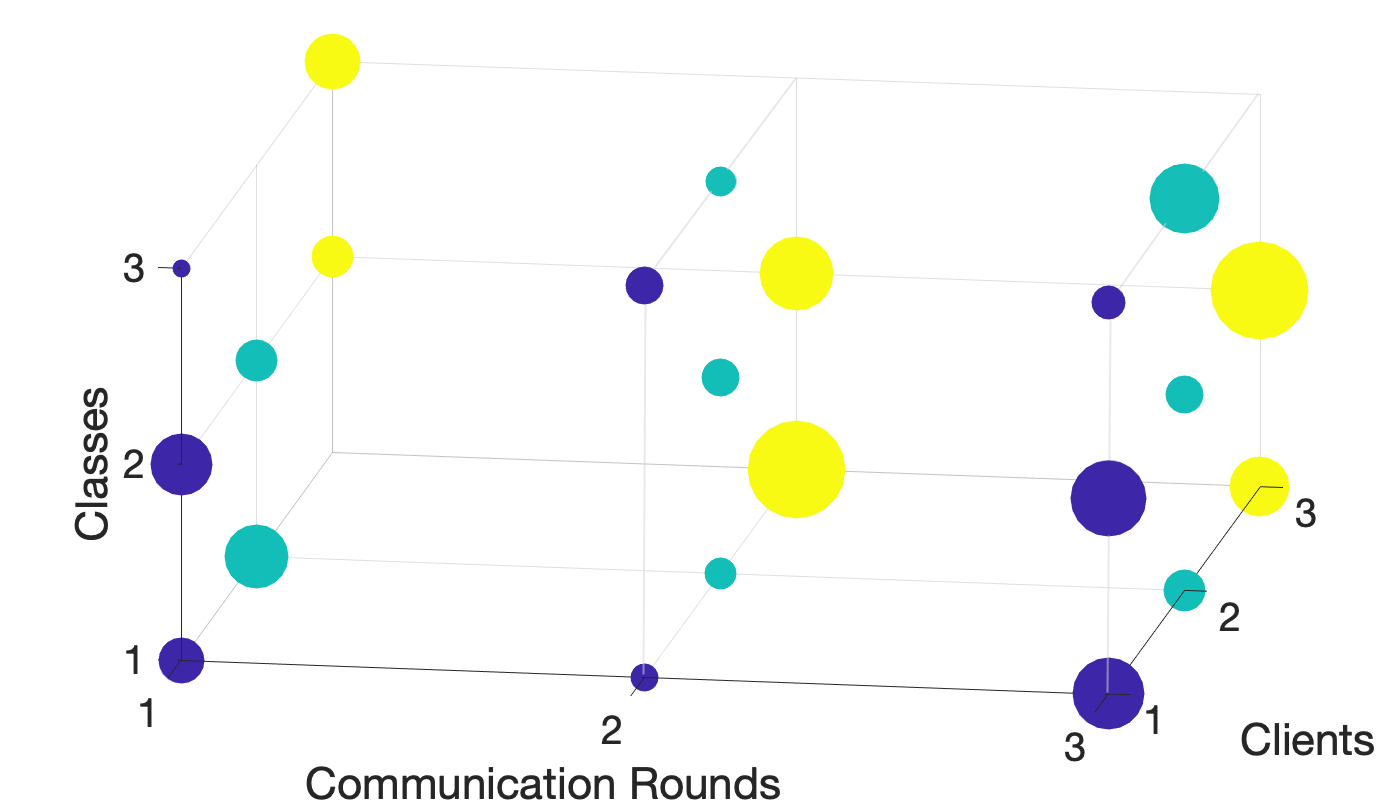}
	\caption{{Time-varying heterogeneous local data distributions.}
		The size of each ball represents the number of samples that correspond to the class.
		Different colors represent different clients.
	}
	\label{time evolving figure}
\end{figure}

However, the FL formulation in these approaches always assumes a fixed data distribution among clients throughout all training rounds,
while in practice this assumption does not always hold because of the complex, uncontrolled, and unpredictable client behaviors.
Local data sets, for example, often vary over time, and intermittent clients can join or depart during training without prior notice. As shown in Figure~\ref{time evolving figure}, the local distribution of one client can also change significantly for different rounds.

The difficulty in addressing the time-evolving heterogeneity of FL lies in the stateless nature of the local data sets, which includes the unpredictable future data sets and the impossibility of retaining all prior local data sets.
To this end, we first provide a novel Continual Federated Learning (CFL) formulation and then propose a unified CFL framework as our solution.
This framework encompasses a variety of design choices for approximating
the local objective functions in the previous training rounds,
with the difference between actual and estimated local object functions described as \emph{information loss} for further analysis.

To capture the time-evolving data heterogeneity in the CFL framework, we expand the theoretical assumption on the client drift---which has been extensively utilized recently
in previous studies~\citep{karimireddy2020scaffold,khaled2020tighter,li2020on}---to include both client drift and \emph{time drift}.
This allows us to quantify the difference between local client objective function and global objective function for the considered time-evolving heterogeneous clients.

We provide convergence rates for our unified CFL framework in conjunction with information loss and new models of both client and time drifts.
Our analysis reveals a faster convergence of CFL framework than \fedavg (under time-evolving scenario),
with the benefit dependent on approximation accuracy.
The rate of CFL can simply recover the rate of \fedavg (under traditional FL scenario) and Continual Learning: the rates of \fedavg obtained from our framework is consistent with previous work, while a similarly simplified rate on Continual Learning (CL)~\citep{french1999catastrophic,kirkpatrick2017overcoming} is novel.
Finally, in extensive empirical results, we demonstrate that CFL methods---stemmed from the CFL framework with different approximation techniques---significantly outperform the SOTA FL competitors on various simulation data sets and realistic data sets for FL settings with time-evolving heterogeneous data.
These numerical observations corroborate our theoretical findings.

We summarize our key contributions:
\begin{itemize}
	\item

	      We present a unified framework, termed Continual Federated Learning (CFL), together with a novel client and time drift modeling approach, to capture complex FL scenarios involving time-evolving heterogeneous data, including standard cross-device FL scenario, stateless scenario, and local data set overlapping.
	      This is the first theoretical study, to our knowledge, that describes the time-evolving nature of FL.

	\item We provide rigorous convergence analysis for the CFL methods.
	      Our theoretical analysis explains the faster and more stabilized optimization of the CFL methods over that of \fedavg under time-evolving scenarios, and conjecture their variance reduction effect.
	      In addition, we provide tight convergence rates for standalone CL methods: to the best of our knowledge, we are the first to provide such guarantees for CL on SGD.
	      %
	\item We thoroughly examining some conventional CL methods under the scope of CFL framework: the insights therein offer a valuable practical guideline. We believe filling such gaps is valuable to the FL community and has not been done yet.
	      We demonstrate the efficacy and necessity of CFL methods over the SOTA FL baselines across a range of time-evolving heterogeneous data scenarios and data sets.
\end{itemize}

\section{Related Work}
\subsection{Federated Learning}
\fedavg~\citep{mcmahan2017communication,lin2020dont} is the de facto standard FL algorithm, in which multiple local SGD steps are executed on the available clients to alleviate the communication bottleneck.
While communication efficient, heterogeneity, such as system heterogeneity~\citep{li2018federated,wang2020tackling,mitra2021achieving,diao2021heterofl} and statistical/objective heterogeneity~\citep{li2018federated,wang2020tackling,mitra2021achieving,lin2020ensemble,karimireddy2020scaffold,karimireddy2020mime}, results in inconsistent optimization objectives and drifted clients models, impeding federated optimization considerably.

A line of work has been proposed to address the heterogeneity in FL.
FedProx~\citep{li2018federated} adds the regularization term on the distance of local and global models when performing local training---similar formulations can be found in other recent FL works~\citep{hanzely2020federated,dinh2020personalized,li2021ditto} for various purposes.
To address the issue of objective heterogeneity e.g.\ caused by heterogeneous data, works like SCAFFOLD~\citep{karimireddy2020scaffold,mitra2021achieving} introduce the idea of variance reduction on the client local update steps.
FedNova~\citep{wang2020tackling} further proposes a general framework for unifying FedAvg and FedProx, and argues that while averaging, local updates should be normalized to minimize heterogeneity induced by different number of local update steps.
However, most of these prior works focus on the fixed heterogeneity across clients and throughout the entire optimization procedure; we instead consider the novel scenario with time-evolving data heterogeneity.

The theoretical study on the convergence of \fedavg can date back to the parallel SGD analysis on the identical functions~\citep{zinkevich2010parallelized} and recently is improved by~\cite{stich2018local,stich2020error,patel2019communication,khaled2020tighter,woodworth2020local}.
For the analysis of heterogeneous data,~\citet{li2020on} first give the convergence rate of FedAvg on non-iid data sets with random selection, assuming that the client optimum are $\epsilon$-close.
Recent studies~\citep{woodworth2020minibatch,khaled2020tighter} give tighter convergence rates under the assumption of bounded gradient drift.
All above works give a $\cO \left (1 / T \right)$ convergence rate for convex local objective functions.
More recently, a series of works~\citep{karimireddy2020scaffold,koloskova2020unified} give the convergence analysis of local SGD for non-convex objective functions under bounded gradient noise assumptions, and obtain a $\cO ( 1 / \sqrt{T} )$ convergence rate.
Our theoretical analysis framework covers more challenging time-evolving data heterogeneity in FL, which has not been considered in the community yet---our rate can be simplified to the standard FL scenario, matching the tight analysis in previous works~\citep{karimireddy2020scaffold}.

\subsection{Continual Learning}
Continual learning, also known as incremental learning or lifelong learning, aims to learn from (time-evolving) sequential data while avoiding the problem of \emph{catastrophic forgetting}~\citep{french1999catastrophic,kirkpatrick2017overcoming}.
There exists a large amount of works, from the perspectives of regularization~\citep{kirkpatrick2017overcoming,li2017learning,zenke2017continual}, experience replay~\citep{castro2018end,rebuffi2017icarl}, and dynamic architectures~\citep{maltoni2019continuous,rusu2016progressive}.
In this paper, we examine both regularization and experience replay based methods,
and compare their empirical performance in depth.  \\
Despite the empirical success, the theoretical analysis of CL is limited: only a recent preprint~\citep{yin2020optimization} provides a viewpoint of regularization-based continual learning, and only for the single worker scenario.
In this work, we provide tight convergence analysis for both CL and CFL on SGD.

\subsection{Continual Federated Learning}
To our knowledge, the scenario of CFL was originally described in~\citet{bui2018partitioned} in order to federated train Bayesian Neural Network and continually learn for Gaussian Process models---it is orthogonal to our optimization aspect in this paper.
FedCurv~\citep{shoham2019overcoming} blends EWC~\citep{kirkpatrick2017overcoming} regularization with \fedavg, and empirically shows a faster convergence.
FedWeIT~\citep{yoon2021federated} extends the regularization idea, with a focus on selective inter-client knowledge transfer for task-adaptive parameters.
In another study,~\citet{zhu2021diurnal} investigate task-specific parameters and only considers the scenario in which data undergo periodic changes.
In a very recent parallel (and empirical) work,
CDA-FedAvg~\citep{casado2021concept} uses locally maintained long-term data samples, with asynchronous communication.
FLwF-2T~\citep{usmanova2021distillationbased} employs a distillation-based approach and focuses on class incremental scenarios.
In~\citet{xu2022acceleration}, a distillation-based method is utilized, incorporating augmented data for distillation.
FedDrift~\citep{jothimurugesan2022federated} focuses on the concept drift problem and proposes a multi-model approach to address it.
Our theoretically sound CFL framework covers the regularization component of FedCurv and the core set part of CDA-FedAvg, and is orthogonal to the neural architecture manipulation idea in FedWeIT.

\section{Continual Federated Learning Framework}

\subsection{Formulation} \label{sec:Formulation}
\subsubsection{Conventional FL Formulation}
The standard FL typically considers a sum-structured distributed optimization problem as below
\begin{align}
	\textstyle
	f^{\star} = \min_{\momega \in \R^d} \left[ f(\momega) := \sum_{i=1}^{N} p_i f_i(\momega) \right] \, ,
	\label{Conventional FL Formulation}
\end{align}
where the objective function $f(\momega): \R^d \rightarrow \R$ is the weighted sum of the local objective functions $f_i(\momega) := \EEb{\cD_i}{F_i (\momega)}$ of $N$ nodes/clients, and $p_i$ is the weight of client $i$.

In practice, it may be infeasible to select all clients each round, especially for cross-device setup~\citep{mcmahan2017communication,kairouz2019advances}.
The standard \fedavg then randomly selects $S$ clients to receive the model parameters from the server ($S \leq N$) in each communication round and performs $K$ local SGD update steps in the form of $\momega_{t, i, k } = \momega_{t, i, k - 1} - \eta_l (\nabla f_i(\momega_{t, i, k-1}) + \mnu_{t, i,k-1})$, with local step-size $\eta_l$ and gradient noise $\mnu_{t, i,k-1}$.
The selected clients then communicate the updates $\Delta \momega_{t,i} = \momega_{t, i, K} - \momega_t$ with the server for the model aggregation: $\momega_{t+1} =\momega_t - \frac{\eta_g}{S} \sum_{i=1}^{S} \Delta \momega_{t,i}$.

\subsubsection{Continual FL Formulation}
Despite its wide usage, the standard FL formulation given in Equation~\eqref{Conventional FL Formulation} cannot properly reflect actual time-evolving scenarios, such as local data sets changing over time or intermittent clients joining or leaving during training.
To address this problem, we propose the Continual Federated Learning (CFL) formulation
\begin{align}
	\textstyle
	f^{\star} = \min_{\momega \in \R^d} \left[ f(\momega) := \sum_{t=1}^{T} \sum_{i \in \cS_{t}} p_{t, i} f_{t, i} (\momega) \right] \,,
	\label{CFL Formulation}
\end{align}
where $f_{t, i} (\momega)$ represents the local objective function of client $i$ at time $t$. $\cS_{t}$ is a subset of clients sampled from all clients set $\Omega$, where $|\cS_{t}| = S$.
For the time-evolving scenarios, client $i$ could have different local objective functions $f_{t, i} (\momega)$ due to the changing local data sets on different $t$.

\subsection{Approximation of CFL}

The challenge of addressing CFL formulation (Equation~\ref{CFL Formulation}) stems from the stateless nature of the local data sets,
which include unpredictable future data sets, as well as the difficulty of keeping all the previous local data sets.
The original definition of CFL formulation (Equation~\ref{CFL Formulation}) is theoretically and empirically infeasible.
To handle the second case (while ignoring the intractable former), a straightforward approach is to approximate the prior local objective functions~\citep{kirkpatrick2017overcoming,zenke2017continual,li2017learning}, which may be accomplished by retraining information from earlier rounds in accordance with privacy protection standards.
Thus, we can express the approximated CFL formulation as
\begin{align}
	\textstyle
	\tilde{f_t}^{\star} = \min_{\momega \in \R^d} \left[
		\sum_{i \in \cS_t} p_{t, i} f_{t, i} (\momega)
		+ \sum_{\tau=1}^{t-1} \sum_{i \in \cS_t} p_{\tau, i} \tilde{f}_{\tau, i}(\momega) \right] \, ,
	\label{Approximate CFL Formulation}
\end{align}
where $\tilde{f}_{t, i} (\momega)$ denotes the approximated local objective function of client $i$ at time $t$.
In practice, different approximation methods can be used to calculate $\tilde{f}_{\tau, i}(\momega)$, and we refer the detailed illustration and discussions of these approximation algorithms in Section \ref{Approximation techniques in CFL}, Section \ref{sec:results}, and Appendix \ref{sec:Approximation Methods Appendix}.

We give the formal definition of CFL framework in Algorithm \ref{CFL Algorithm Framework}.
CFL methods are a collection of methods that make use of different approximation techniques and are based on Algorithm \ref{CFL Algorithm Framework}.
We retrieve the objective function of 
Continual Learning (CL) by setting $S = 1$ in \eqref{Approximate CFL Formulation}.
\begin{algorithm}[t]

	\begin{algorithmic}[1]
		\Require{initial weights $\momega_0$, global learning rate $\eta_g$, local learning rate $\eta_l$, number of training rounds $T$.}
		\Ensure{ trained weights $\momega_T$.}
		\For{round $t = 1, \ldots, T$}
		\myState{\textit{communicate} $\momega_t$ to the chosen clients.}
		\For{\textit{client} $i \in \cS_t$ \textit{in parallel}}
		\myState{initialize local model $\momega_{t,i,0} = \momega_t$.}
		\For{$k = 1, \ldots, K$}
		\myState{Calculate stochastic gradient $\tilde{g}_{t,i,k}$ by \eqref{Approximate CFL Formulation}.}
		\myState{$\momega_{t,i,k} \gets \momega_{t,i,k-1} - \eta_l \tilde{g}_{t,i,k}$.}
		\EndFor
		\myState{\textit{communicate} $\Delta \momega_{t,i} \gets \momega_{t,i,K} - \momega_t$.}
		\EndFor
		\myState{$\Delta \momega_t \gets \frac{\eta_g}{S} \sum_{i \in \cS_t} \Delta \momega_{t,i}$.}
		\myState{$\momega_{t+1} \gets \momega_t + \Delta \momega_t$.}
		\EndFor
	\end{algorithmic}
	\mycaptionof{algorithm}{ Approximated CFL Framework}
	\label{CFL Algorithm Framework}
\end{algorithm}

Take note that in most cases, the previous local object functions cannot be properly approximated.
Due to the fact that such approximation impairs optimization, we define the information loss below.
\begin{definition}[Information Loss]
	Let $\Delta_{t, i}(\momega)$ be the information loss between the approximated local objective function $\tilde{f}_{t, i} (\momega)$ and the real local objective function $f_{t, i} (\momega)$. This information loss is defined as
	\begin{align}
		\Delta_{t, i}(\momega) = \nabla f_{t, i}(\momega) - \nabla \tilde{f}_{t, i}(\momega) \,. \label{information loss definition}
	\end{align}
\end{definition}
In practice, we use $\norm{\Delta_{t, i}(\momega)}_2$ to measure the information loss of different approximation methods.
We show large information loss can impede the convergence theoretically (c.f.\ Theorem \ref{Convergence rate of CFL}) and empirically (e.g.\ Figure~\ref{Information Loss Figure} in Section \ref{sec:results}).

\subsection{Gradient Noise Model}
To analyze Equation~\eqref{Approximate CFL Formulation} and Algorithm \ref{CFL Algorithm Framework} in-depth, we propose the Gradient Noise Model (Equation~\ref{Gradient Noise in CFL}) to capture the dynamics of performing SGD with data heterogeneity.
We first recap the standard definition of gradient noise in previous works~\citep{ref1,gower2019sgd}.

\textit{Gradient noise in SGD.} For objective function $f(\momega)$, the gradient with stochastic noise can be defined as
$\nabla f(\momega) = g(\momega) + \mnu$,
where $g(\momega)$ is the stochastic gradient, and $\mnu$ is a zero-mean noise.

\textit{Client drift in FL.}
To analyze the impact of data heterogeneity, recent works~\citep{karimireddy2020scaffold,khaled2020tighter,li2020on} similarly use the gradient noise to capture the distribution drift between local objective function and global objective function
\begin{align*}
	\textstyle
	\nabla f(\momega) = \nabla f_i(\momega) + \mdelta_{i} \, ,
\end{align*}
where $\mdelta_i$ is a zero-mean random variable which measures the gradient noise of client $i$.

\textit{Client drift and time drift in CFL framework.}
Considering the distribution drift in the dimension of client and time, we further modify the gradient noise model in FL as
\begin{align}
	\textstyle
	\nabla f(\momega) = \nabla f_{t,i}(\momega) + \mdelta_{i} + \mxi_{t,i} \, . \label{Gradient Noise in CFL}
\end{align}
Here we extend the gradient noise to two terms: $\mdelta_{i}$ and $\mxi_{t, i}$.
$\mdelta_{i}$ is a time independent (of $t$), zero-mean random variable that measures the drift of client $i$.
The zero-mean $\mxi_{t, i}$ measures the drift of client $i$ at time $t$:
We assume a fixed underlying client distribution for each client $i$, and the time drift is only the additive noise for the given client drift. For each client $i$, an time-drift will be sampled for each time step $t$ and added on top of the fixed client drift to capture the time-evolving nature of the clients’ local data sets. Notice that the random variables are nature to be zero-mean because we assume the global objective $f(\momega) = \Eb{f_{t,i}(\momega)}$ as in most FL works~\citep{karimireddy2020scaffold,karimireddy2020mime,mcmahan2017communication}. We use Assumption \ref{Bounded gradient noise of CFL assumption} to construct the non-iidness.

\begin{remark}
	To simplify the study and address the cross-device FL settings, we assume that different $\mxi_{t, i}$ are independent.
	A more complicated change pattern of local data may arise in cross-silo FL scenarios in terms of time-dependent $\mxi_{t, i}$, resulting in overlapped local data sets.
	The formulation of these time-dependent scenarios would need some scenario-specific assumptions, and we give some analysis in Section~\ref{sec:time_dependent_drifts}. Besides, to narrow the gap, in Section~\ref{sec:results}, we experimentally analyze the performance of CFL approaches with overlapped time-evolving local data, and the findings are consistent with our theoretical analysis.
	\label{shortcoming remark}
\end{remark}

\subsection{Assumptions}
To ease the theoretical analysis of CFL (Equation~\ref{CFL Formulation}) framework, we use the following widely used assumptions.

\begin{assumption}[Smoothness and convexity]
	\label{Smoothness and convexity assumption}
	Assume local objective functions $f_{t, i} (\momega)$ are $L$-smooth and $\mu$-convex. This means that the following inequality holds for all $\momega_1, \momega_2$: $\frac{\mu}{2} \norm{\momega_1 - \momega_2}^2 \le f_{t, i} (\momega_1) - f_{t, i} (\momega_2) - \langle \nabla f_{t, i} (\momega_2), \momega_1 - \momega_2 \rangle \le \frac{L}{2} \norm{\momega_1 - \momega_2}^2$.
\end{assumption}

A widely used corollary is that if a function $f_{t, i}$ is both $L$-smooth and $\mu$-convex, then it satisfies $\frac{1}{2L} \norm{ \nabla f_{t, i}(\xx) - \nabla f_{t, i}(\yy) }^{2} \le f_{t, i}(\xx) - f_{t, i}(\yy) - \nabla f_{t, i}(\xx)^{T}(\xx - \yy)$, and $L \ge \mu$.

\begin{assumption}[Bounded noise in stochastic gradient]
	\label{Bounded noise in stochastic gradient assumption}
	Let $g_{t, i, k}(\momega) = \nabla f_{t, i, k}(\momega) + \mnu_{t, i, k}$, where $\mnu_{t, i, k}$ is the stochastic noise of client $i$ on round $t$ at $k$-th local update step.
	We assume that $\Eb{ \mnu | \momega } = 0$, and $\Eb{ \norm{ \mnu }^{2} | \momega} \le \sigma^{2}$.
\end{assumption}

Assumption~\ref{Smoothness and convexity assumption} and~\ref{Bounded noise in stochastic gradient assumption} are common assumptions in FL~\citep{karimireddy2020scaffold,li2020on}.
A relaxed assumption, such as bounded noise at the optimum, as discussed in recent works~\citep{khaled2020tighter}, is an interesting direction for future research.
Besides, we introduce the following assumptions to better characterize the client and time drifts in CFL framework (Equation~\ref{Gradient Noise in CFL}).

\begin{assumption}[Bounded gradient drift of CFL framework]
	\label{Bounded gradient noise of CFL assumption}
	Let $f(\momega) = \nabla f_{t,i}(\momega) + \mdelta_{i} + \mxi_{t,i}$, where $\mdelta_{i}$  independent on client $i$ measures client drift, and $\mxi_{t,i}$ independent on time $t$ and indicates the time drift of $\{t, i\}$.
	We assume that $\Eb{ \mdelta | \momega } = 0$, $\Eb{ \mxi | \momega } = 0$, and thus $\Eb{ \norm{ \mdelta }^2 | \momega} \le G^2 + B^2 \E{ \norm{\nabla f(\momega) }^2 }$ and $\Eb{ \norm{ \mxi }^2 | \momega} \le D^2 + A^2 \E{ \norm{ \nabla f(\momega) }^2 }$.
\end{assumption}

Assumption~\ref{Bounded gradient noise of CFL assumption} assumes the bounded $\Eb{\norm{ \mdelta }^2 | \momega}$ and $\Eb{ \norm{ \mxi }^2 | \momega}$ in CFL framework, which is equivalent to the widely used $(G, B)$ gradient drift assumption\footnote{
	The $(G, B)$ gradient drift: $\Eb{ \norm{ \nabla f_i (\momega) }^2 } \le G^2 \Eb{ \norm{ \nabla f(\momega) }^2 } + B^2$, where $G^2 \ge 1$ and $B^2 \ge 0$.
}~\citep{gower2019sgd,karimireddy2020scaffold} on the bounded $\Eb{ \norm{ \nabla f_i (\momega) }^2 }$.
We prove this claim in Appendix ~\ref{sec:Proof of lemma bound of distribution drift}.
Notice that when $D > 0$, or $G > 0$, the clients' local optimum at round $t$, $\momega_{i, t}^{*}$ will be different from global optimum $\momega^{*}$. Besides, as illustrated before, $\Eb{ \mdelta | \momega } = 0$ and $\Eb{ \mxi | \momega } = 0$ are nature due to the assumption of global objective $f(\momega) = \Eb{f_{t,i}(\momega)}$ as in most FL works~\citep{karimireddy2020scaffold,karimireddy2020mime,mcmahan2017communication}.

\begin{assumption}[Bounded information loss] \label{Bounded information loss assumption}
	Let $R$ be a non-negative real number and $\momega$ be a vector. We assume that the information loss $\Delta_{t, i}(\momega)$ satisfies $| \Delta_{t, i}(\momega) | \le R$.

\end{assumption}

\begin{remark}

	In Assumption~\ref{Bounded information loss assumption}, the exact information loss may be difficult to calculate for certain approximation methods, such as generative replay. In Lemma~\ref{bounded approximation error} of Appendix~\ref{sec:proof of bounded approximation error}, we provide a detailed analysis of information loss for Taylor extension based regularization methods and provide a more general case in the main paper.
\end{remark}

\section{Theoretical results} \label{sec:theoretical_rates}
\begin{table*}[!t]

	\centering
	\resizebox{1.\textwidth}{!}{%
		\begin{tabular}{l l l l}
			\toprule

			\textit{Algorithm}             &                                                                                                                                & \textit{Strongly Convex}                                                                                                                                             & \textit{General Convex}                                                                                                                                                         \\
			\midrule
			\textit{SGD}                   &
			                               & $\frac{\sigma^2}{\mu NK \varepsilon} + \frac{1}{\mu}$                                                                          & $\frac{\sigma^2}{NK \varepsilon^2} + \frac{1}{\varepsilon}$                                                                                                                                                                                                                                                                                            \\ \midrule
			\textit{FedAvg}                &
			~~~\citet{li2020on}            & $\frac{\sigma^2}{\mu^2 NK \varepsilon} + \frac{(G^2 + D^2) K}{\mu^2 \varepsilon}$                                              & -                                                                                                                                                                                                                                                                                                                                                      \\
			                               & ~~~\citet{khaled2020tighter}
			                               & $\frac{\sigma^2 + G^2 + D^2}{\mu N K \varepsilon} + \frac{\sigma + G + D}{\mu \sqrt{\varepsilon}} + \frac{N (A^2 + B^2)}{\mu}$ & $\frac{\sigma^2 + G^2 + D^2}{N K \varepsilon^2} + \frac{\sigma + G + D}{\mu \varepsilon^{\frac{3}{2}}} + \frac{N (A^2 + B^2)}{\varepsilon}$                                                                                                                                                                                                            \\
			                               & ~~~\citet{karimireddy2020scaffold}                                                                                             & $\frac{\sigma^2}{\mu N K \varepsilon} + \frac{G + D}{\mu \sqrt{\varepsilon}} + \frac{c_{p_B}}{\mu}$                                                                  & $\frac{\sigma^2}{KN \varepsilon^2} + \frac{G + D}{\varepsilon^{\frac{3}{2}}} + \frac{c_{p_B} (D^2 + G^2)}{\varepsilon}$                                                         \\

			                               & ~~~\citet{woodworth2020minibatch}                                                                                              & $\frac{\sqrt{c_{p_B}})}{K \sqrt{\mu \epsilon}} + \frac{\sigma^2 c_{p_B}}{N K \epsilon^2} + \frac{G + D}{\mu \sqrt{\epsilon}} + \frac{\sigma}{\mu \sqrt{K \epsilon}}$ & $\frac{c_{p_B}}{K \epsilon} + \frac{\sigma^2 c_{p_B}}{NK \epsilon^2} + \frac{(G + D) c_{p_B}}{\epsilon^{\frac{3}{2}}} + \frac{\sigma c_{p_B}}{\sqrt{K} \epsilon^{\frac{3}{2}}}$ \\

			                               & ~~Ours                                                                                                                         & $\frac{\sigma^2}{\mu NK \varepsilon} + \frac{G^2 + D^2}{\mu \varepsilon}
				+ \frac{c_{p_B}}{\mu}$
			                               & $\frac{c_{p_B}}{\varepsilon} + \frac{\sigma^2}{ N K \varepsilon^2} + \frac{G^2 + D^2}{\varepsilon^2}
			$
			\\
			\midrule
			\textit{CL}                    &
			~~~\citet{yin2020optimization} & $\frac{\mu}{\varepsilon}$ (GD)                                                                                                 & -                                                                                                                                                                                                                                                                                                                                                      \\
			                               & ~~Ours                                                                                                                         &
			$\frac{1 + A^2}{\mu}
				+ \frac{\sigma^2}{\mu K \epsilon}
				+ \frac{c_A}{\mu \epsilon}
			$
			                               &
			$
				\frac{c_{p_B}}{\varepsilon}
				+ \frac{\sigma^2}{K \varepsilon^2} + \frac{c_A}{\varepsilon^2}
			$
			\\
			\midrule
			\textit{CFL}                   &
			~~Ours                         & $\frac{c_{p_B}}{\mu} + \frac{\sigma^2}{\mu N K \epsilon}
				+ \frac{c_A + G^2}{\mu \epsilon}
			$
			                               &
			$
				\frac{c_{p_B}}{\varepsilon}
				+ \frac{\sigma^2}{N K \varepsilon^2}
				+ \frac{c_A + G^2}{\varepsilon^2}
			$                                                                                                                                                                                                                                                                                                                                                                                                                                                                                                                        \\

			\bottomrule
		\end{tabular}%

	}
	\caption{
		{Number of communication rounds required to reach $\varepsilon + \varphi$ accuracy} for $\mu$ strongly convex and general convex functions under time-evolving FL scenarios (Assumption \ref{Smoothness and convexity assumption}--\ref{Bounded information loss assumption}).
		We can recover the rates in conventional FL setups by setting $D = 0$ and $A = 0$.
		Note that $\varphi = 0$ in \fedavg.
		Our convergence rate of \fedavg matches the results in previous works~\citep{karimireddy2020scaffold}.
		Our SGD rate of CL on the strongly-convex case is novel. Note that all elaborated results use SGD unless specifically mentioned.
		$c_{p_B} = 1 + B^2 + A^2$, $c_{A} = \frac{D^2 R^2}{R^2 + D^2}$.
		$N$ is the number of chosen clients in each round, and $K$ is the number of local iterations.
	}
	\label{Convergence rates of different algorithms}
\end{table*}

In this section, we analyze the theoretical performance of Algorithm \ref{CFL Algorithm Framework} under Assumption \ref{Smoothness and convexity assumption}--\ref{Bounded information loss assumption}.
We prove that CFL methods converges faster than \fedavg, despite that the term (refer to $\varphi$ in Theorem \ref{Convergence rate of CFL}, which is a bias term that can not eliminate by reducing learning rate) induced by information loss concurrently trades off the optimization.
Additionally, our results include the \fedavg rate (and match the prior work) and a novel convergence rate for SGD on CL.
In Table \ref{Convergence rates of different algorithms}, we summarize the convergence rate of different algorithms\footnote{
	To match the notations of prior works that include all clients in all training rounds, we use the abbreviation $N$ to denote the number of selected clients in each round in this section.
}.
The proof details refer to Appendix \ref{sec:Proof of Theorem Convergence rate of CFL-R} and Appendix \ref{sec:Non-convex}.

\subsection{Convergence Rate of CFL Methods}

Before analyzing Algorithm \ref{CFL Algorithm Framework}, we want to clarify that the introduced information loss $\Delta_{t, i}$, which is an approximation error, will introduce a constant that cannot be eliminated by reducing the learning rate. Therefore, instead of aiming for convergence to an arbitrary $\epsilon$, we aim to optimize until the expected error is smaller than $\epsilon + \varphi$, where $\varphi = \frac{\sum_{t=1}^{T} q_t \varphi_t}{\sum_{t=1}^{T} q_t}$ for some sequence $q_t$ and constant $\varphi_t$. The value of $\varphi_t$ is positively related to the bounded information loss $R$ defined in Assumption~\ref{Bounded information loss assumption}. Details of the proof can be found in Appendix \ref{sec:Proof of Theorem Convergence rate of CFL-R}.

\begin{theorem}
	[Convergence rate of CFL methods]
	\label{Convergence rate of CFL}
	Assume $\{ f_{t, i} (\momega) \}$ satisfy Assumption \ref{Smoothness and convexity assumption}--\ref{Bounded information loss assumption}, the output of Algorithm \ref{CFL Algorithm Framework} has expected error smaller than $\epsilon + \varphi$, for $\eta_g = 1$, $\eta_l \le \frac{\sqrt{3 + 4 (1 + B^2 + A^2)} - \sqrt{4 (1 + B^2 + A^2)}}{6KL\sqrt{1 + B^2 + A^2} }$,
	$p_{\tau, i} = \frac{D^2}{t D^2 + (t - 1) R^2}$ ($\tau < t$), and $p_{t, i} = \frac{(t-1)R^2 + D^2}{tD^2 + (t-1)R^2}$ on round $t$.

	When $\{ f_{t, i} (\momega) \}$ are $\mu$-strongly convex functions, we have
	\begin{align*}
		\textstyle
		T = \cO \left(
		\frac{L c_O}{\mu} + \frac{\sigma^2}{\mu N K \epsilon}
		+ \frac{1}{\mu \epsilon} \left( G^2 + \frac{D^2 R^2}{R^2 + D^2} \right)
		\right) \, ,
	\end{align*}
	and when $\{ f_{t, i} (\momega) \}$ are general convex functions ($\mu = 0$), we have
	\begin{align*}
		\textstyle
		T = \cO \left(
		\frac{c_O \mathcal{H}}{\epsilon}
		+ \frac{\sigma^2 \mathcal{H}}{N K \epsilon^2} + \frac{\mathcal{H}}{\epsilon^2} \left( G^2 + \frac{D^2 R^2}{R^2 + D^2} \right)
		\right) \,,
	\end{align*}
	and when $\{ f_{t, i} (\momega) \}$ are non-convex, setting $\eta = K\eta_g \eta_l = \frac{\sqrt{KN}}{\sqrt{T} L}$, when $\frac{1}{T} \sum_{t=1}^{T} \Eb{\norm{\nabla f(\momega_t)}^2}$ reach $\epsilon$ we have
	\begin{align*}
		\textstyle
		 & T =
		\cO \left( \frac{L^2(f_0 - f_*)^2}{NK c_m^2 \epsilon^2}
		+ \frac{1}{\epsilon^2} \left( \frac{\sqrt{KN} c_{R1}^2}{L}+ \frac{\sigma^2}{\sqrt{N K}} \right)^2 \right) \, ,
	\end{align*}
	where $c_O = 1 + A^2 + B^2$,
	$c_{R1} = \frac{R^2}{R^2 + D^2}$,
	$c_m$ is a constant related to $A$, $B$ and $R$,
	and $\mathcal{H} = \norm{\momega_0 - \momega^*}^2$.
	$N$ is the number of chosen clients in each round, and $K$ is the number of local iterations.
	We elaborate the choice of $p_{t, i}$ in Appendix~\ref{sec:weights of rounds}.
\end{theorem}

\begin{remark}
	When setting $N = 1$ in Theorem \ref{Convergence rate of CFL}, we recover the convergence rate of standalone CL methods.
	To the best of our knowledge, we are the first to provide such theoretical guarantees for SGD: the recent work~\citep{Yin20continual} only gives the convergence rate of GD for regularization based CL methods on general convex case, under the constraint of information loss $R = 0$.
\end{remark}

\begin{proposition}

	For stateless cross-device FL scenarios (clients only appear once during training), we can derive lower bounds than that of Theorem \ref{Convergence rate of CFL}.
	More precisely, by applying Lemma \ref{Bounded Gradient Noise of CFL} to the considered stateless scenario, the key variance term $c_{normal} = G^2 + \frac{D^2 R^2}{R^2 + D^2}$ in Theorem \ref{Convergence rate of CFL} becomes
	$c_{stateless} = \frac{(G^2 + D^2) R^2}{G^2 + D^2 + R^2}$.
	$c_{stateless}$ is the lower bound of $c_{normal}$, indicating that CFL methods provide higher performance improvements in stateless FL scenarios.

\end{proposition}

\subsection{Convergence Rate of \fedavg under Time-evolving Scenarios}

To show the faster convergence of CFL methods than \fedavg, we provide the convergence rate of \fedavg under time-evolving scenarios and Assumption  \ref{Smoothness and convexity assumption}--\ref{Bounded information loss assumption}.
Note that prior works only consider traditional FL scenarios (w/o time-evolving heterogeneous data) for \fedavg, and we offer new rates of \fedavg for such scenarios by setting $p_{t, i} = 1$.

\begin{theorem}[Convergence rate of \fedavg under time-evolving scenarios]
	\label{Convergence rate of FedAvg}
	Assume $\{ f_{t, i} (\momega) \}$ satisfy Assumption \ref{Smoothness and convexity assumption}--\ref{Bounded information loss assumption}, the output of \fedavg has expected error smaller than $\epsilon$, for $\eta_g = 1$ and $\eta_l \le \frac{\sqrt{3 + 4 (1 + B^2 + A^2)} - \sqrt{4 (1 + B^2 + A^2)}}{6KL\sqrt{1 + B^2 + A^2} }$.
	When $\{ f_{t, i} (\momega) \}$ are $\mu$-strongly convex functions, we have
	\begin{align*}
		\textstyle
		T = \cO \left(
		\frac{L c_O}{\mu} + \frac{\sigma^2}{\mu N K \epsilon}
		+ \frac{G^2 + D^2}{\mu \epsilon}
		\right) \,,
	\end{align*}
	and when $\{ f_{t, i} (\momega) \}$ are general convex functions ($\mu = 0$), we have
	\begin{align*}
		\textstyle
		T = \cO \left(
		\frac{c_O \mathcal{H} }{\epsilon}
		+ \frac{\sigma^2 \mathcal{H}}{N K \epsilon^2} + \frac{(G^2 + D^2) \mathcal{H}}{\epsilon^2}
		\right) \,,
	\end{align*}
	and when $\{ f_{t, i} (\momega) \}$ are non-convex, setting $\eta = K\eta_g \eta_l = \frac{\sqrt{KN}}{\sqrt{T} L}$, when $\frac{1}{T} \sum_{t=1}^{T} \Eb{\norm{\nabla f(\momega_t)}^2}$ reach $\epsilon$ we have
	\begin{align*}
		\textstyle
		 & T =
		\cO \left( \frac{L^2(f_0 - f_*)^2}{NK c_m^2 \epsilon^2}
		+ \frac{1}{\epsilon^2} \left( \frac{\sqrt{KN}}{L}+ \frac{\sigma^2}{\sqrt{N K}} \right)^2 \right) \, ,
	\end{align*}
	where $c_O = 1 + A^2 + B^2$,
	and $\mathcal{H} = \norm{\momega_0 - \momega^*}^2$.
	$N$ is the number of chosen clients in each round, and $K$ is the number of local iterations.
\end{theorem}

Note that in both Theorem \ref{Convergence rate of CFL} and Theorem \ref{Convergence rate of FedAvg}, we only elaborate dominate terms to clarify the difference. Details please refer to Theorem \ref{Convergence rate of CFL full} and \ref{Convergence rate of FedAvg full} in Appendix \ref{sec:Theoretical Results in appendix}.

\begin{remark}
	\label{rmk:R}
	\fedavg (under time-evolving scenarios) is a special case of CFL methods by setting $p_{t, i} = 1$, as stated in Theorem \ref{Convergence rate of CFL} and Theorem \ref{Convergence rate of FedAvg}.
	We show in Appendix \ref{sec:Proof of Theorem Convergence rate of CFL-R} that the rate of \fedavg (under time-evolving scenarios) is equivalent to setting the upper bound of information loss $R \to \infty$.
	Similar observations can be found by setting $R \to \infty$ in $p_{t, i} = \frac{(t-1)R^2 + D^2}{tD^2 + (t-1)R^2}$ in Theorem \ref{Convergence rate of CFL}.
\end{remark}

\begin{remark}
	\label{sec:Theoretical Analysis}
	CFL methods accelerate the convergence by reducing the variance term of \fedavg.
	For example, the gradient noise term $G^2 + D^2$ of \fedavg for convex functions in Theorem \ref{Convergence rate of FedAvg} can be decreased to $G^2 + \frac{D^2 R^2}{R^2 + D^2}$ in Theorem \ref{Convergence rate of CFL}.
	Similarly, the gradient noise term of non-convex functions can be reduced from $1$ to $\frac{R^2}{R^2 + D^2}$.
	The benefits of variance reduction in CFL depend on the approximation accuracy $R$, as stated in remark~\ref{rmk:R}.

\end{remark}

\subsection{Discussion: Convergence Rate of CFL for Time-dependent Drifts}
\label{sec:time_dependent_drifts}
In the previous sections, our convergence rates---though valuable---only assume a time $t$ independent time drift $\mxi_{t, i}$ (c.f.\ Assumption~\ref{Bounded gradient noise of CFL assumption}), an assumption may not always hold in practice.
In this section, we extend our convergence analysis of CFL to the correlated time drifts.

\begin{assumption}[Correlated time drifts]
	\label{Generalized time drift with correlation assumption}
	Let $\mxi_{t_1, i}$ and $\mxi_{t_2, i}$ be time drifts of client $i$  at times $t_1$ and $t_2$, respectively. We assume that the expected value of the correlation between $\mxi_{t_1, i}$ and $\mxi_{t_2, i}$ given $\momega$ satisfies $\Eb{ \langle \mxi_{t_1, i}, \mxi_{t_2, i} \rangle |\momega } \le F_{t_1, t_2}^2 + C_{t_1, t_2}^2 \Eb{ \norm{ \nabla f(\momega) }^2 }$.
\end{assumption}

Assumption~\ref{Generalized time drift with correlation assumption} assumes that time drifts $\xi_{t_1, i}$ and $\xi_{t_2, i}$ are correlated, which allows us to simulate more complex scenarios by setting different relationships between $F_{t_1, i}$ and $F_{t_2, i}$ (also $C_{t_1, i}$ and $C_{t_2, i}$). In the following subsections, we present two case studies to examine time drifts in real-world scenarios.

\subsubsection{Case 1}
\label{sec:case1}
In the first case, we consider the following scenario: Local data sets in different rounds have overlap. Then we assume the bounded correlation between time drifts in the following assumption.
\begin{assumption}[Correlated time drifts: Case 1]
	Let $\mxi_{t_1, i}$ and $\mxi_{t_2, i}$ be time drifts of client $i$  at times $t_1$ and $t_2$, respectively. We assume that the expected value of the correlation between $\mxi_{t_1, i}$ and $\mxi_{t_2, i}$ given $\momega$ satisfies $\Eb{ \langle \mxi_{t_1, i}, \mxi_{t_2, i} \rangle |\momega} \le F^2 + C^2 \Eb{ \norm{ \nabla f(\momega) }^2 }$.
	\label{Time drift with correlation assumption}
\end{assumption}

Assumption~\ref{Time drift with correlation assumption} states that the correlation between these time drifts is bounded by constants $F$ and $C$. This is a generalized assumption when the change in time drifts over time is unknown, and we use the same $F$ and $C$ for all $t_1, t_2$ in Assumption~\ref{Generalized time drift with correlation assumption}. For instance, in stateless FL, servers may not have information on the overlap of new clients with clients in previous rounds. Under this more realistic assumption, we provide the convergence rate of CFL.

\begin{theorem}[Convergence rate of CFL methods with correlated time drifts]
	\label{Convergence rate of CFL with correlated time drifts}
	Assume $\{ f_{t, i} (\momega) \}$ satisfy Assumption \ref{Smoothness and convexity assumption}--\ref{Time drift with correlation assumption}, the output of Algorithm \ref{CFL Algorithm Framework} has expected error smaller than $\epsilon + \varphi$, for $\eta_g = 1$, $\eta_l \le \frac{\sqrt{3 + 4 c_O} - \sqrt{4 c_O}}{6KL\sqrt{c_O} }$,
	$p_{\tau, i} = \frac{M^2}{t (M^2) + (t - 1) R^2}$ (for $\tau < t$), and $p_{t, i} = \frac{(t-1)R^2 + M^2}{t(M^2) + (t-1)R^2}$ on round $t$.

	When $\{ f_{t, i} (\momega) \}$ are $\mu$-strongly convex functions, we have
	\begin{align*}
		\textstyle
		T = \cO \left(
		\frac{L c_O}{\mu} + \frac{\sigma^2}{\mu N K \epsilon}
		+ \frac{1}{\mu \epsilon} \left( G^2 + \frac{M^2 R^2}{R^2 + M^2} \right)
		\right) \, ,
	\end{align*}
	and when $\{ f_{t, i} (\momega) \}$ are general convex functions ($\mu = 0$), we have
	\begin{align*}
		\textstyle
		T = \cO \left(
		\frac{c_O \mathcal{H}}{\epsilon}
		+ \frac{\sigma^2 \mathcal{H}}{N K \epsilon^2} + \frac{\mathcal{H}}{\epsilon^2} \left( G^2 + \frac{M^2 R^2}{R^2 + M^2} \right)
		\right) \,,
	\end{align*}
	and when $\{ f_{t, i} (\momega) \}$ are non-convex, by setting $\eta = K\eta_g \eta_l = \frac{\sqrt{KN}}{\sqrt{T} L}$, $\frac{1}{T} \sum_{t=1}^{T} \Eb{\norm{\nabla f(\momega_t)}^2}$ needs to take the following steps to reach $\epsilon$
	\begin{align*}
		\textstyle
		 & T =
		\cO \left( \frac{L^2(f_0 - f_*)^2}{NK c_m^2 \epsilon^2}
		+ \frac{1}{\epsilon^2} \left( \frac{\sqrt{KN} c_{R1}^2}{L}+ \frac{\sigma^2}{\sqrt{N K}} \right)^2 \right) \, ,
	\end{align*}
	where $M^2 = max(0, D^2 - F^2)$,
	$c_O = 1 + A^2 + max(B^2, C^2)$,
	$c_{R1} = \frac{R^2}{R^2 + M^2}$,
	$c_m$ is a constant related to $A$, $B$, $C$ and $R$,
	and $\mathcal{H} = \norm{\momega_0 - \momega^*}^2$.
	$N$ is the number of chosen clients in each round, and $K$ is the number of local iterations.
	We elaborate the choice of $p_{t, i}$ in Appendix~\ref{sec:weights of rounds}.
\end{theorem}

\begin{remark}
	Comparing Theorem~\ref{Convergence rate of CFL} and Theorem~\ref{Convergence rate of CFL with correlated time drifts}, we can see that the convergence rate of CFL with correlated time drifts is equivalent to setting the time drift $D^2$ to be $max(0, D^2 - F^2)$ in Theorem~\ref{Convergence rate of CFL}. This suggests that the correlation between time drifts can reduce the impact of time drifts on the convergence rate, and the weights of previous rounds' objective function $p_{\tau, i}$ decrease as $F^2$ increases. This aligns with our observations in real-world settings. For instance, when local data sets contain data from previous rounds, the model performance is significantly improved (cf. Table~\ref{Performance on overlap local data sets}). This observation also supports the effectiveness of core set based methods in Continual Learning.
\end{remark}



\subsubsection{Case 2}
In the second case, we assume that local data sets change gradually over time. Specifically, the overlap between the current data set $D_t$ and the previous data set $D_{\tau}$ decreases as the time difference $|t - \tau|$ increases. This leads to the extension of Assumption~\ref{Generalized time drift with correlation assumption}.
\begin{assumption}[Correlated time drifts: Case 2]
	\label{Time drift with correlation assumption: a special case}
	We assume that the correlation between time drifts $\mxi_{t_1, i}$ and $\mxi_{t_2, i}$ decreases as $|t_1 - t_2|$ increases to model a scenario where local data sets change gradually over time. This is formalized a
	\begin{align*}
		\Eb{ \langle \mxi_{t_1, i}, \mxi_{t_2, i} \rangle |\momega } \le \alpha^{|t_1 - t_2|}\left( D^2 + A^2 \Eb{ \norm{ \nabla f(\momega) }^2 } \right) \,,
	\end{align*}
	where $0 \le \alpha < 1$, and $D$ and $A$ are constants as defined in Assumption~\ref{Bounded gradient noise of CFL assumption}.
\end{assumption}

Under Assumption~\ref{Time drift with correlation assumption: a special case}, the correlation between $\xi_{\tau}$ and $\xi_{t}$ increases as the absolute value of $t - \tau$ decreases. The value of $\alpha$ controls the rate at which local data sets change over time. A large value of $\alpha$ indicates a slower rate of change.

Deriving the final convergence rate for this scenario is non-trivial. Therefore, we analyze the choice of weights $p_{\tau, i}$ and provide some insights.
On top of the proof stated in Section~\ref{sec:proof of optimal weights of special case},
we optimize $p_{t, i}$ to minimize the upper bound of the convergence rate. Under Assumption  \ref{Smoothness and convexity assumption}--\ref{Bounded information loss assumption} and Assumption~\ref{Time drift with correlation assumption: a special case}, the optimal weights of $p_{t, i}$ can be given by
\begin{subequations}
	\begin{align*}
		 & \min_{p} \frac{1}{N} \sum_{i=1}^{N} (1 - p_{t, i})^2 R^2 + G^2 + D^2 \sum_{\tau_1, \tau_2}^{t} \alpha^{\| \tau_1 - \tau_2 \|} p_{\tau_1} p_{\tau_2} \, , \\
		 & \text{s.t.} \; \sum_{\tau=1}^{t} p_{\tau, i} = 1, \forall i = 1, \dots, N \, .
	\end{align*}
\end{subequations}
The given optimization problem can be reformulated as a quadratic programming problem and can be solved by finding the solution to the corresponding linear system.

\begin{theorem}[Optimal choice of $p_{\tau, i}$ under correlated time drifts]
	Let $f_{t,i}(\omega)$ satisfy Assumptions \ref{Smoothness and convexity assumption}--\ref{Bounded information loss assumption} and Assumption \ref{Time drift with correlation assumption: a special case}. On round $t$, the optimal choice of $\hat{p} = [p_{1,i}, \dots, p_{t,i}]^\top$ is determined by solving the linear system
	\begin{align}
		\begin{bmatrix}
			Q              & \mathbf{1} \\
			\mathbf{1}^{T} & 0
		\end{bmatrix}
		\begin{bmatrix}
			\hat{p} \\
			\lambda
		\end{bmatrix}
		=
		\begin{bmatrix}
			\mathbf{0} \\
			1
		\end{bmatrix},
	\end{align}
	where each entry of $\mathbf{1} \in \mathbb{R}^{t}$ is $1$, and $\lambda$ is the Lagrange multiplier that comes out of the solution alongside $\hat{p}$. Matrix $Q$ is defined as
	\begin{align*}
		Q = \begin{bmatrix}
			    D^2 + R^2              & \alpha D^2 + R^2       & \cdots & \alpha^{t-2} D^2 + R^2 & \alpha^{t-1} D^2 \\
			    \alpha D^2 + R^2       & D^2 + R^2              & \cdots & \alpha^{t-3} D^2 + R^2 & \alpha^{t-2} D^2 \\
			    \vdots                 & \vdots                 &        & \vdots                 & \vdots           \\
			    \alpha^{t-2} D^2 + R^2 & \alpha^{t-3} D^2 + R^2 & \cdots & \alpha^2 D^2 + R^2     & \alpha D^2       \\
			    \alpha^{t-1} D^2       & \alpha^{t-2} D^2       & \cdots & \alpha D^2             & D^2
		    \end{bmatrix} \, .
	\end{align*}
	\label{Optimal choice of weights with correlated time drifts}
\end{theorem}
The optimal $p_{\tau, i}$ can be found by solving the linear system in Theorem~\ref{Optimal choice of weights with correlated time drifts}.
It is easy to show that $Q$ is positive definite (see Lemma~\ref{pd Q}), and we can then use LU factorization to solve the linear system in Theorem~~\ref{Optimal choice of weights with correlated time drifts}.

\begin{example}
	Solving the linear system in Theorem~\ref{Optimal choice of weights with correlated time drifts} requires the values of $\alpha$, $D$, and $R$.
	It is difficult to provide an explicit expression of $p_{t, i}$ with arbitrary $\alpha$, $D$, and $R$.
	Therefore, we use a toy example to provide some insights.
	We set different values of $\alpha$, $D$, and $R$ ($t = 4$ for all settings) and solve the linear system in Theorem~\ref{Optimal choice of weights with correlated time drifts} accordingly. The results in Table~\ref{Optimal weights of CFL with correlated time drifts} show that the optimal weights of the current round increase as $\alpha$ increases. This indicates that a larger correlation between time drifts results in lower importance of previous rounds' objectives.
	This aligns with our observation in Theorem~\ref{Convergence rate of CFL with correlated time drifts} that, as the time correlation bound $F^2$ increases, $p_{\tau, i}$ defined in Theorem~\ref{Convergence rate of CFL with correlated time drifts} decreases and the convergence rate in Theorem~\ref{Convergence rate of CFL with correlated time drifts} increases, indicating that the impact of time drifts on convergence is reduced.

	\begin{table}[!t]
		\centering
		\caption{
			{Optimal weights of CFL with correlated time drifts.}
			We solve the optimization problem defined in Theorem~\ref{Optimal choice of weights with correlated time drifts} with different $\alpha$, $D$, and $R$, and report the optimal $p_{\tau}$ for $\tau = 1, \dots, 4$.
		}
		\resizebox{.5\textwidth}{!}{%
			\begin{tabular}{c c c c c c c}
				\toprule
				$\alpha$ & $D$ & $R$ & $p_1$  & $p_2$  & $p_3$   & $p_4$  \\
				\midrule
				0        & 2   & 1   & 0.1818 & 0.1818 & 0.1818  & 0.4545 \\
				0.5      & 1   & 0.5 & 0.2857 & 0.1429 & 0.0000  & 0.5714 \\
				0.5      & 1   & 1   & 0.2632 & 0.1316 & -0.0789 & 0.6842 \\
				0.8      & 1   & 0.5 & 0.3870 & 0.0774 & -0.2077 & 0.7434 \\
				0.8      & 2   & 0.5 & 0.3960 & 0.0792 & -0.1188 & 0.6436 \\
				\bottomrule
			\end{tabular}}

		\label{Optimal weights of CFL with correlated time drifts}
	\end{table}
\end{example}

\section{Experiments}
\label{sec:Experiments}
We utilize the Noisy Quadratic Model---a simple convex model---to verify the theoretical results presented in Section~\ref{sec:theoretical_rates}. The results in Appendix~\ref{sec:Noisy quadratic model} demonstrate that the convergence rate of CFL methods is much faster than baselines like \fedavg and FedProx, with a smoother convergence curve. These findings support the variance reduction effect of CFL methods mentioned in Remark~\ref{sec:Theoretical Analysis}.

We further conduct extensive empirical evaluations comparing the performance of our CFL methods with various strong FL competitors on a variety of realistic data sets. Additionally, we explore various approximation techniques in our CFL methods and demonstrate how they can be used to achieve efficient federated learning with time-evolving data heterogeneity.

\subsection{Noisy Quadratic Model}
\label{sec:Noisy quadratic model}

\begin{figure*}[!t]
	\centering
	\subfigure[Small $L$, strongly convex, small round drift]{ \includegraphics[width=.23\textwidth,]{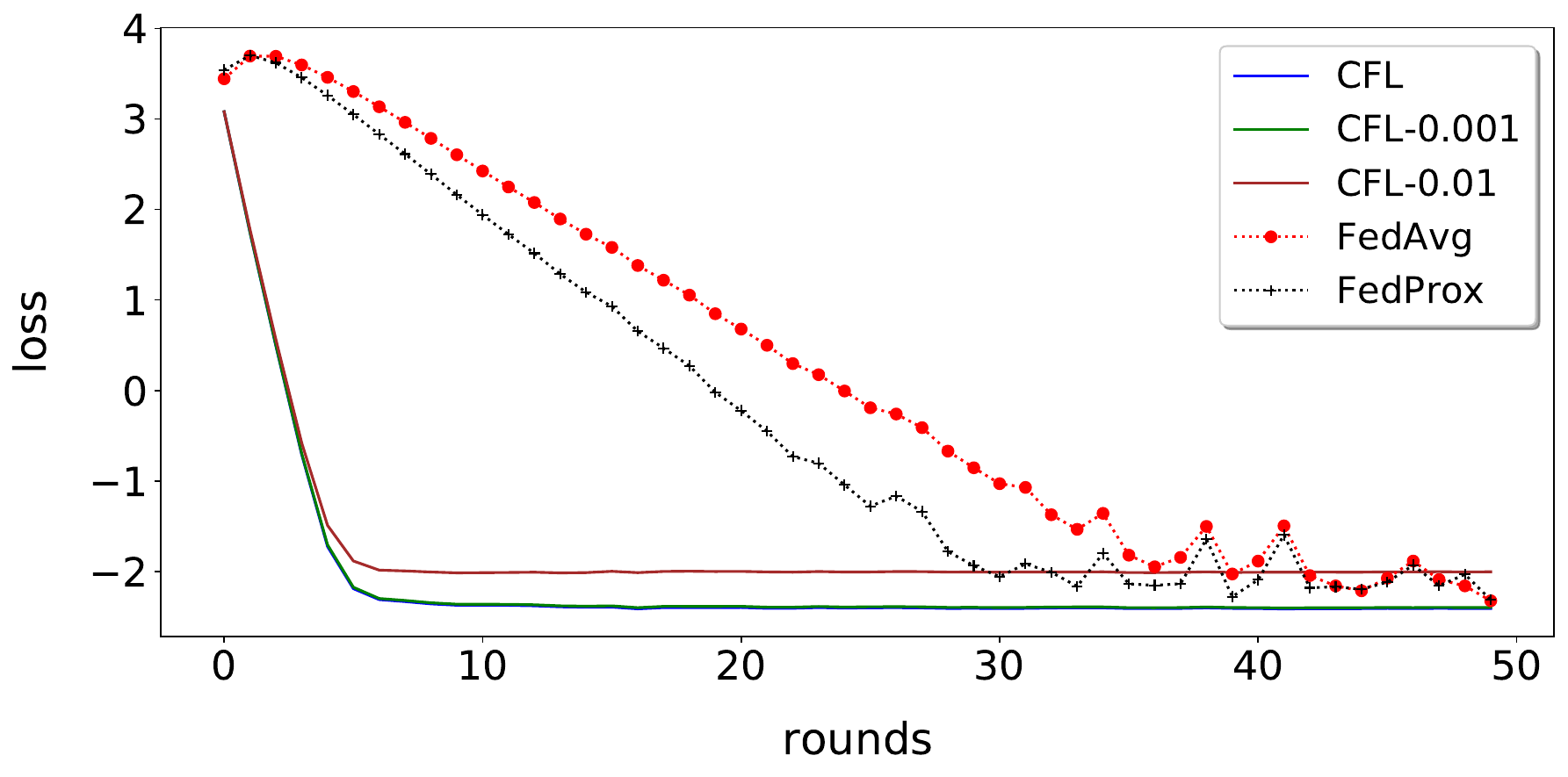}}
	\subfigure[Large $L$, strongly convex, small round drift]{ \includegraphics[width=.23\textwidth,]{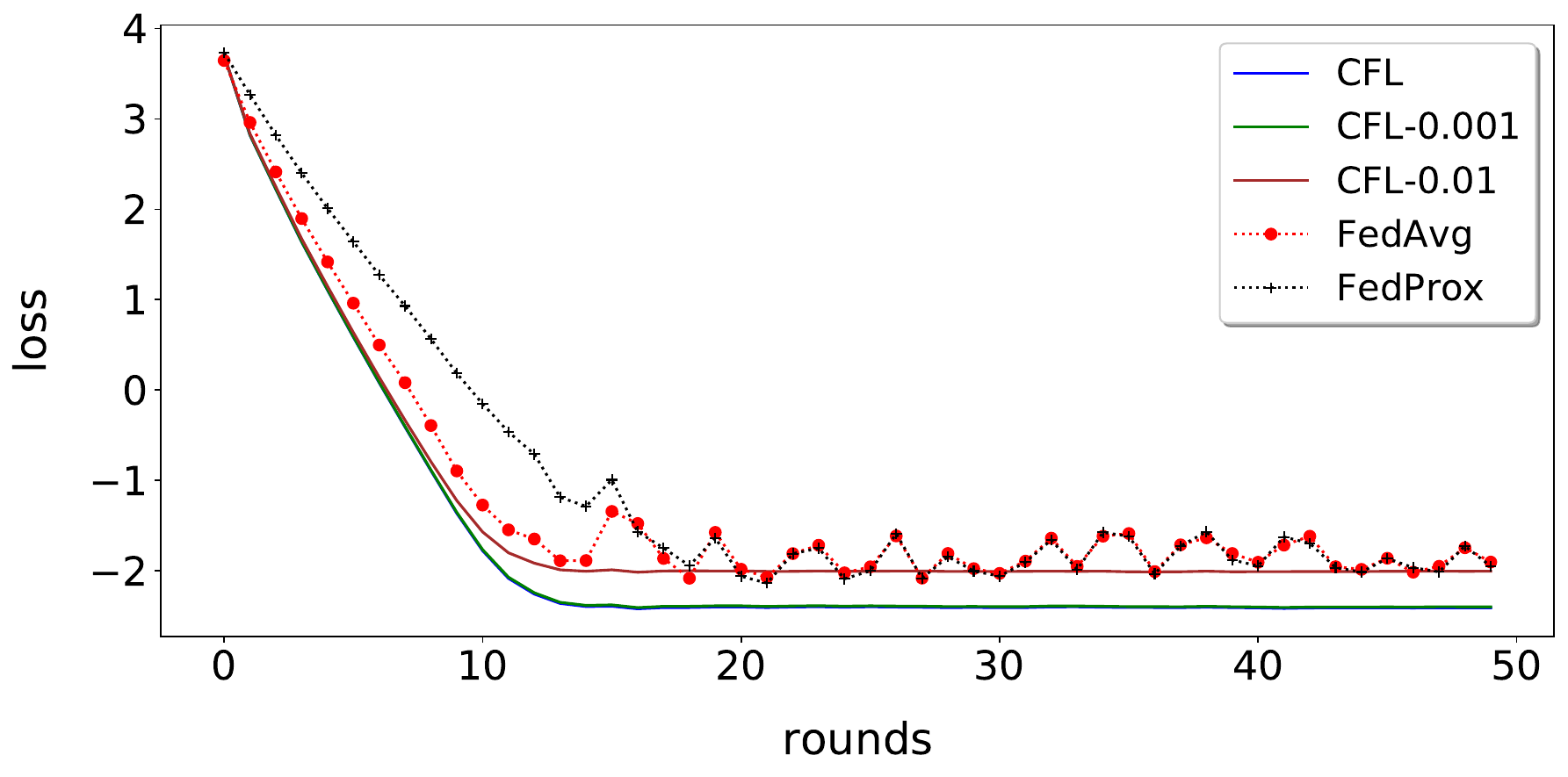}}
	\subfigure[Small $L$, general convex, small round drift]{ \includegraphics[width=.23\textwidth,]{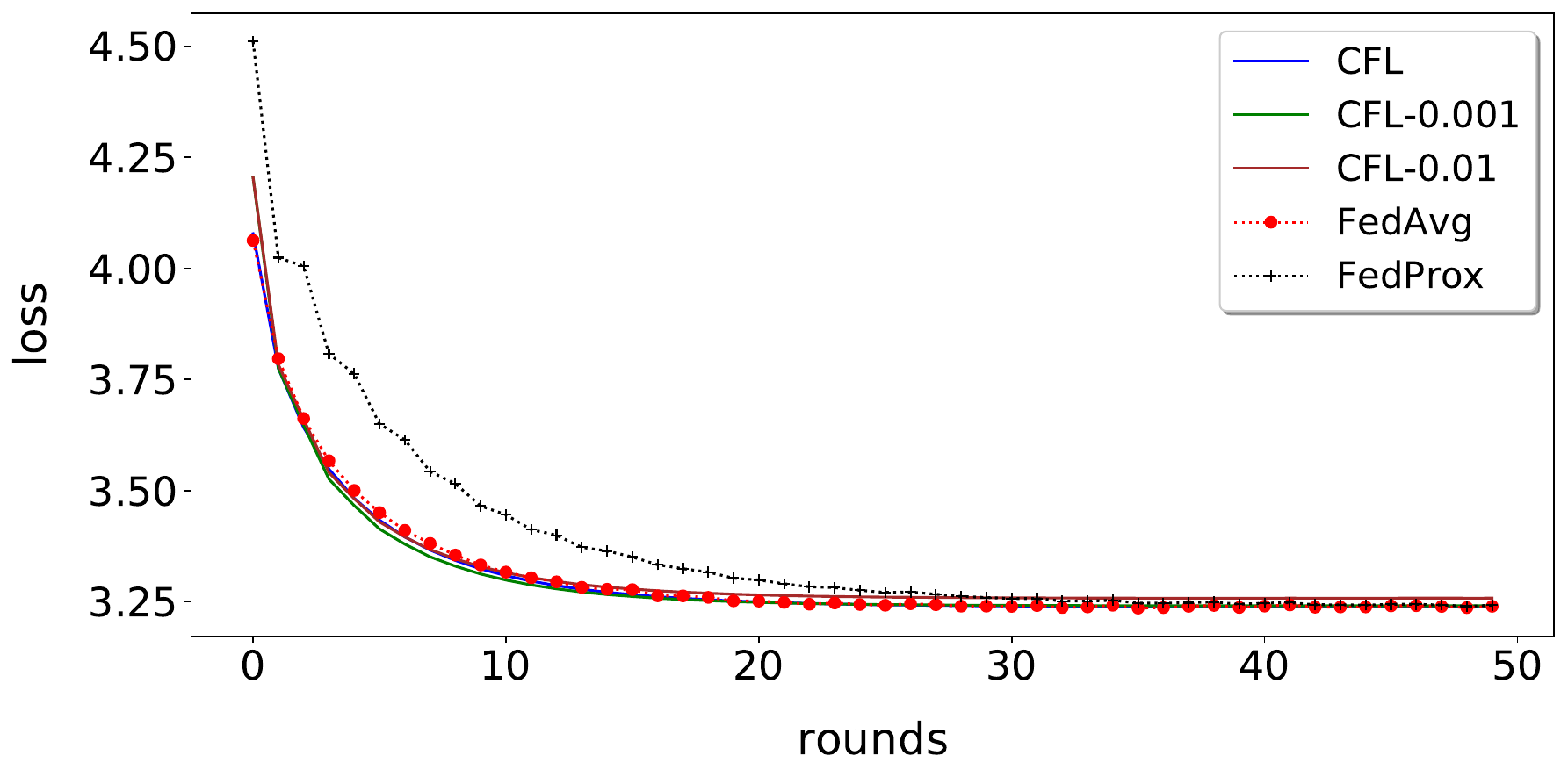}}
	\subfigure[Large $L$, general convex, small round drift]{ \includegraphics[width=.23\textwidth,]{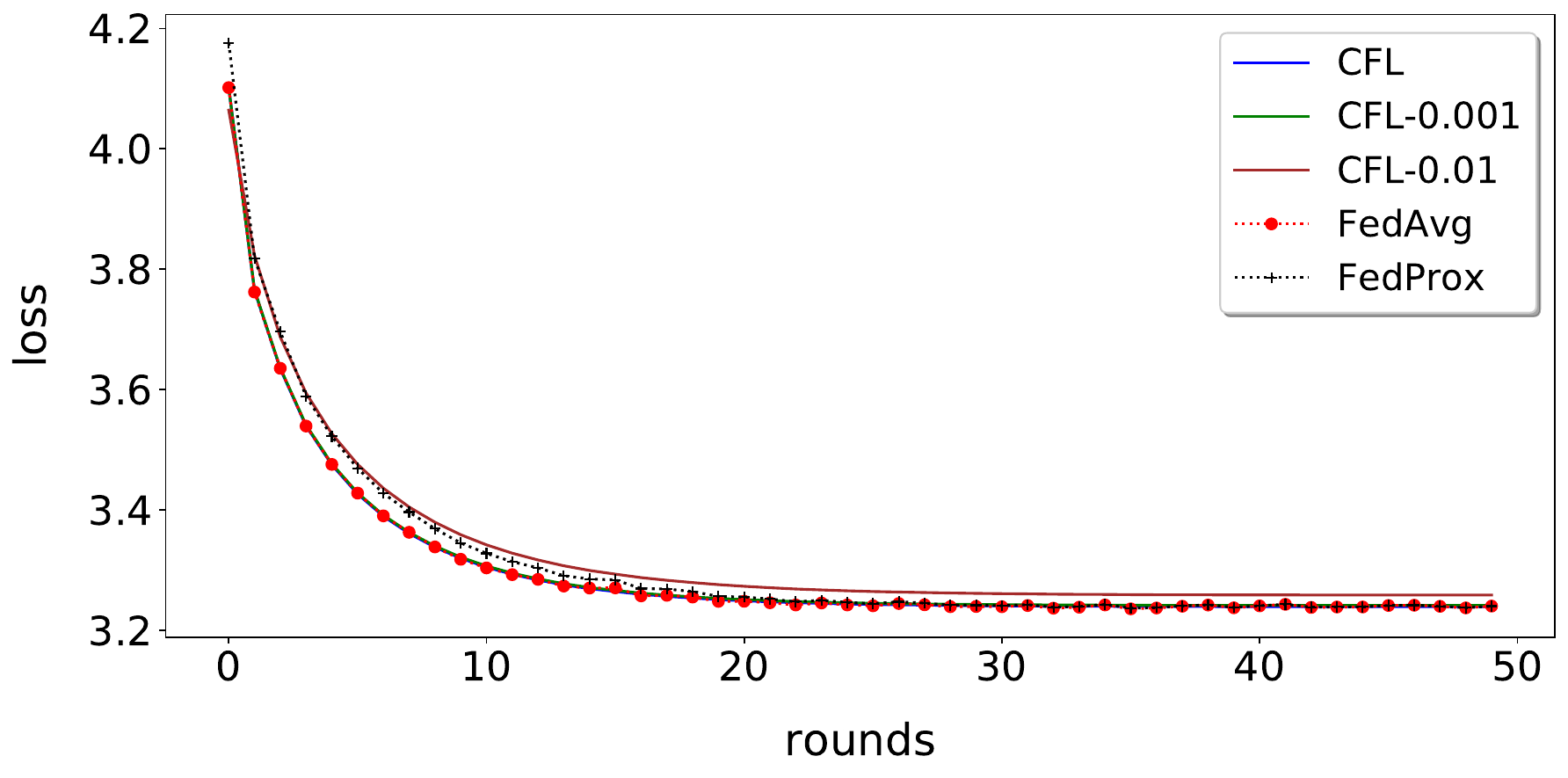}}
	\\
	\centering
	\subfigure[Small $L$, strongly convex, big round drift]{ \includegraphics[width=.23\textwidth,]{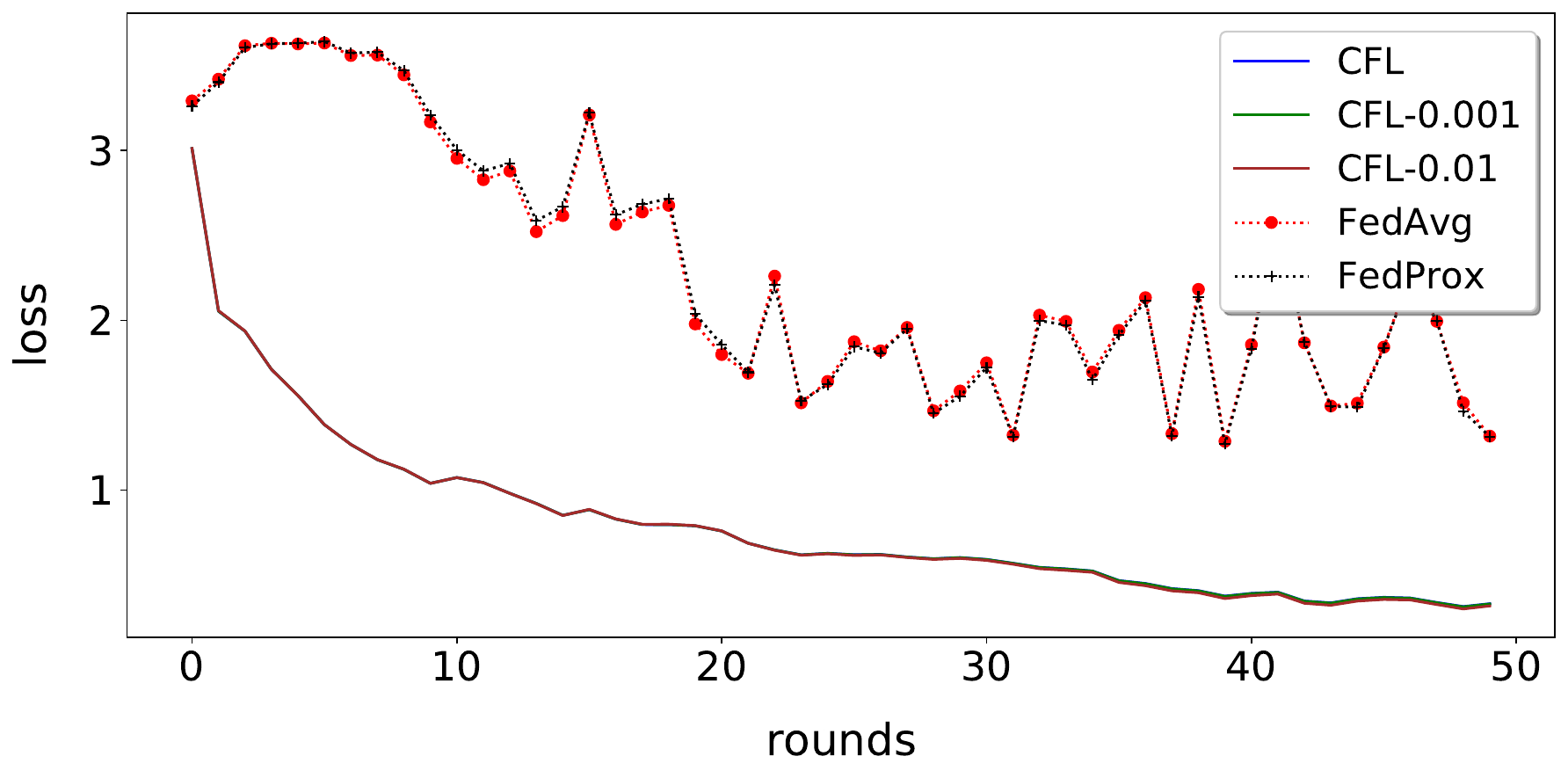}}
	\subfigure[Large $L$, strongly convex, big round drift]{ \includegraphics[width=.23\textwidth,]{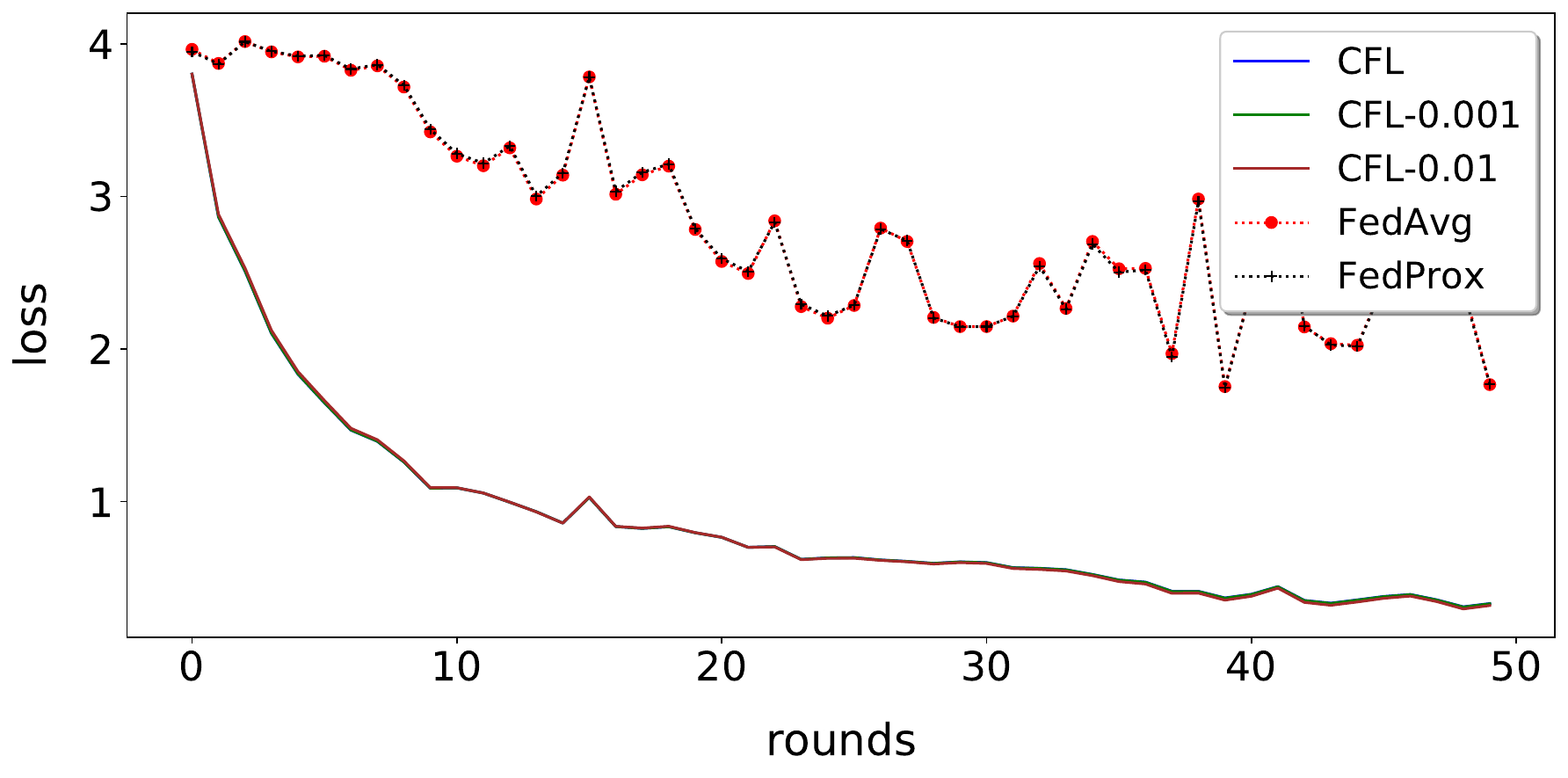}}
	\subfigure[Small $L$, general convex, big round drift]{ \includegraphics[width=.23\textwidth,]{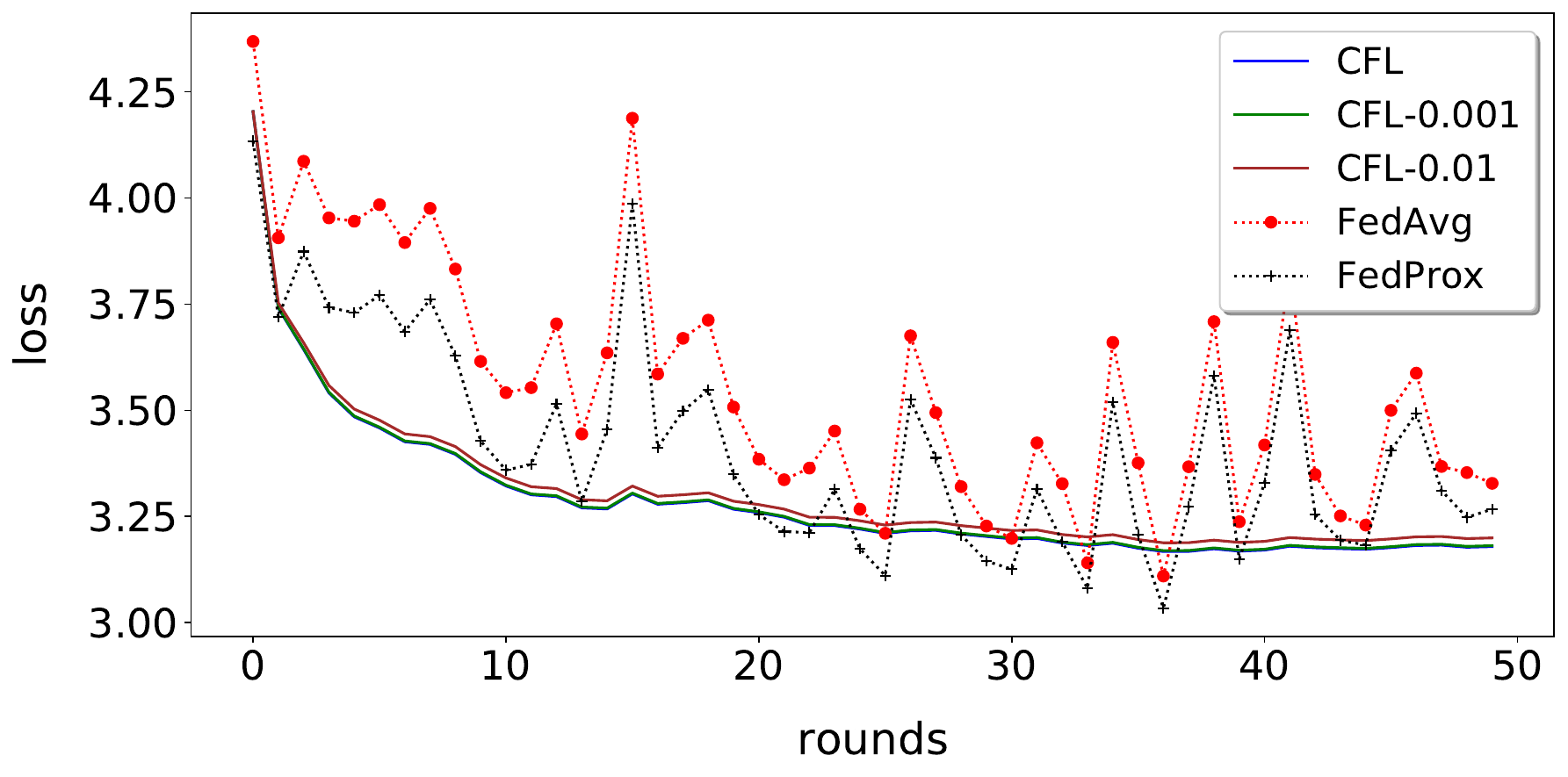}}
	\subfigure[Large $L$, general convex, big round drift]{ \includegraphics[width=.23\textwidth,]{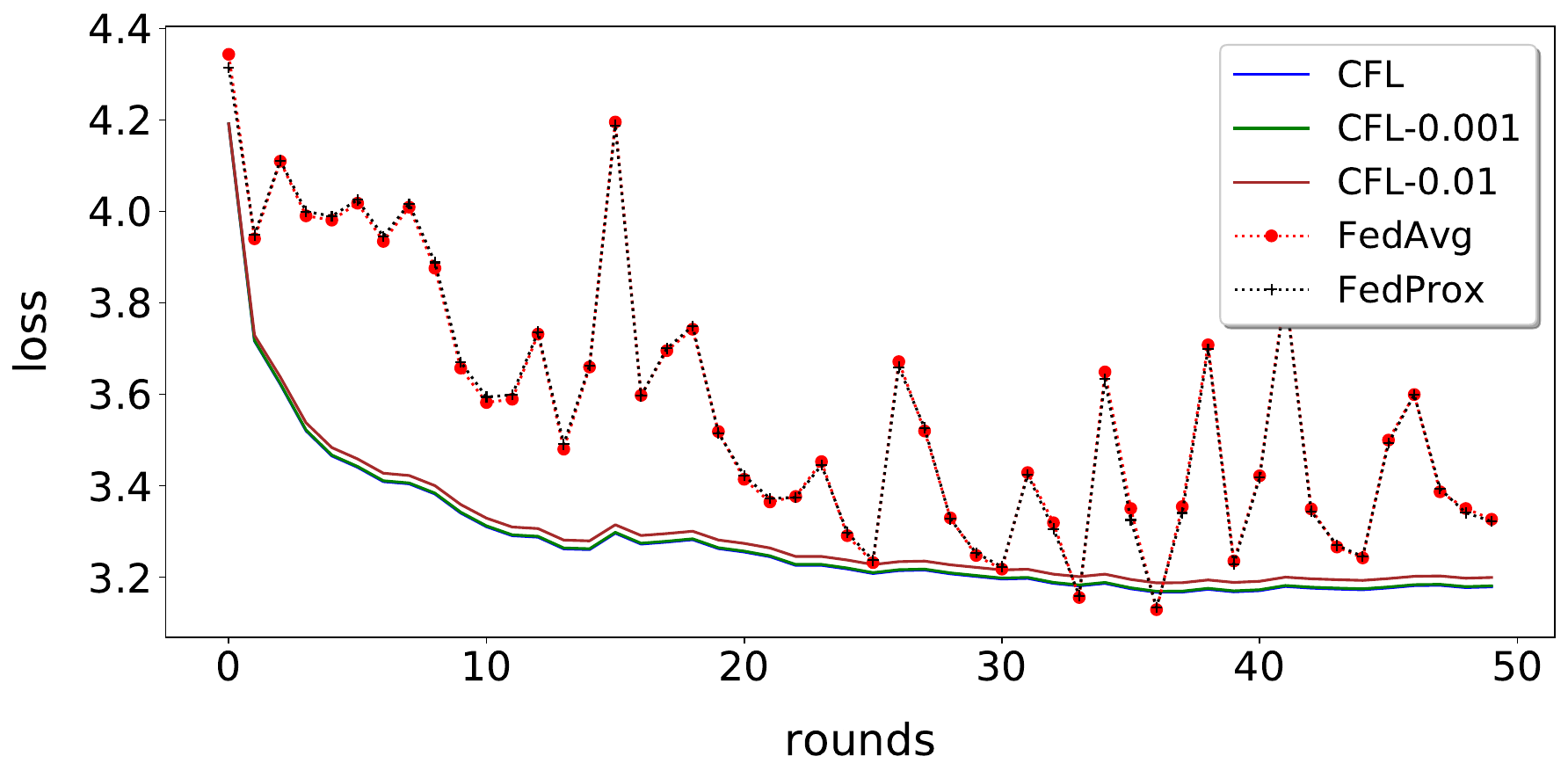}}
	\caption{{Performance of different algorithms on noisy quadratic model.} The curves all evaluate the loss on global test data sets. We set the information loss to 0 for CFL, and CFL-0.01 and CFL-0.001 represent different levels of information loss.}
	\label{Performance of different algorithms on noisy quadratic model Appendix}
\end{figure*}

In this section, we introduce the Noisy Quadratic Models and show the superior performance of CFL over \fedavg and FedProx~\citep{li2018federated}.

\subsubsection{Setup of Noisy Quadratic Models}

We utilize the noisy quadratic model from \cite{zhang2019algorithmic} to simulate the assumptions made in our proof. This model is defined by $f(\momega) = \momega^T \mA \momega + \mB^T \momega + \mC$, where $\momega \in \R^{n}$ is the parameter we aim to optimize and $\mA \in \R^{n \times n}$, $\mB \in \R^{n}$, and $\mC \in \R$ are matrices with $\mA \succeq 0$. To construct $\mA$, we first create a diagonal matrix $\Lambda$ and then let $\mA = \mU^{T} \Lambda \mU$, where $\mU$ is a unitary matrix. By controlling the eigenvalues of $\mA$, we can control the convexity of the model, ensuring that $\mu \le \lambda_{min}(\mA) \le \lambda_{max}(\mA) \le L$. Since the eigenvalues of $\Lambda$ and $\mA$ are the same, it is simple to control the eigenvalues of $\mA$ by controlling those of $\Lambda$.

To simulate Assumptions \ref{Bounded noise in stochastic gradient assumption} and \ref{Bounded gradient noise of CFL assumption}, we introduce the gradient noise model given by
\begin{align*}
	g_{t,i}(\momega) = \mA \momega + \mB + \mdelta_i + \mxi_{t, i} + \mnu_{t, i, k}, \qquad
	\tilde{g}_{t,i}(\momega) = g_{t,i}(\momega) + \cN (\mI, \mI / 1000),
\end{align*}
where $\mdelta_i$ represents the client drift of client $j$, $\mxi_{t, i}$ denotes the round drift of client $j$ on round $i$, and $\mnu_{t, i, k}$ denotes the noise of gradient of client $j$ on round $i$, iteration $k$. All of these drifts and noise are generated from a Normal distribution with zero mean and different variance. The term $\cN (I, I / 1000)$ is used to simulate information loss, and we use $\tilde{g}_{t,i}(\momega)$ instead of $g_{t,i}(\momega)$ in the CFL algorithm, and set different values of $I$ to study the effect of information loss.

In practice, to show if the numerical results match our theoretical results, we conducted eight experiments with varying settings. We examined large and small values of $L$, as well as different levels of convexity and round drift. For $L$, we set $L=20$ for large values and $L=5$ for small values. For convexity, we set $\mu=1$ for strongly convex and $\mu=0$ for general convexity. For round drift, we set $\E{\norm{\mxi_{t,i}}^2}=100$ for large drift and $\E{\norm{\mxi_{t,i}}^2}=0.01$ for small drift. Additionally, we fixed the variance of $\mdelta_i$ and $\mnu_{t,i,k}$ at $\E{\norm{\mdelta_i}^2}=0.01$ and $\E{\norm{\mnu_{t,i,k}}^2}=0.00001$.

To perform gradient descent, we set $\momega_{t,i,k+1} = \momega_{t,i,k} - \eta_l g_{t,i}(\momega_{t,i,k})$. We use the loss function $\norm{A \momega + B}$ because it reaches 0 at the global optimal point for convex functions, and has a large value when far from the global optimal point.

\subsubsection{Results for Noisy Quadratic Models}

The results of our experiments can be found in Table \ref{Best learning rate of different algorithms on noisy quadratic model} and Figure \ref{Performance of different algorithms on noisy quadratic model Appendix}. From these results, we can draw the following conclusions:
\begin{itemize}[nosep,leftmargin=12pt]
	\item CFL algorithms can tolerate higher learning rates. As seen in Table \ref{Best learning rate of different algorithms on noisy quadratic model}, the best learning rates for different FL algorithms on noisy quadratic model are higher for CFL algorithms.
	\item CFL algorithms demonstrate better convergence than other FL baselines. From Figure \ref{Performance of different algorithms on noisy quadratic model Appendix}, it is clear that CFL outperforms other FL baselines in all settings. Additionally, the loss curves of CFL are smoother, indicating a reduced variance of gradients. Furthermore, we find that the convergence of CFL with larger information loss (CFL-0.01) is worse than that of CFL with smaller information loss (CFL and CFL-0.001).
\end{itemize}

\begin{table*}[!t]
	\centering
	\caption{{Best learning rate of different algorithms on noisy quadratic model.} we denote small round drift as SRD, big round drift as BRD, small $L$ as SL, large $L$ as LL, strongly convex as SC, and general convex as GC.}
	\begin{tabular}{c c c c c c c c c}
		\toprule
		        & \multicolumn{4}{c}{SRD} & \multicolumn{4}{c}{BRD}                                                 \\
		\cline{2-5}                                     \cline{6-9}
		Method  & SL-SC                   & LL-SC                   & SL-GC & LL-GC & SL-SC & LL-SC & SL-GC & LL-GC \\
		\midrule
		FedAvg  & 0.03                    & 0.08                    & 0.33  & 0.09  & 0.02  & 0.01  & 0.09  & 0.02  \\
		FedProx & 0.04                    & 0.07                    & 0.35  & 0.09  & 0.02  & 0.01  & 0.26  & 0.02  \\
		CFL     & 0.2                     & 0.09                    & 0.35  & 0.09  & 0.25  & 0.08  & 0.3   & 0.08  \\
		\bottomrule
	\end{tabular}
	\label{Best learning rate of different algorithms on noisy quadratic model}
\end{table*}

\subsection{Experiments Setup for Real Datasets} \label{sec:setup}
\textit{Time-evolving heterogeneous local data sets.}
We study the use of federated learning on split-CIFAR10, split-CIFAR100, and split-Fashion-MNIST data sets for image classification tasks using ResNet18 and a two layer MLP.
The split is based on the Latent Dirichlet Allocation (LDA) method, as proposed in~\cite{yurochkin2019bayesian,hsu2019measuring,reddi2021adaptive}, to control distribution drift with the parameter $\alpha$ (see Algorithm \ref{Data spliting algorithm}).
Higher values of $\alpha$ result in smaller drifts.\\
Unless specifically mentioned otherwise our studies use the following protocol.
All data sets are partitioned to 210 subsets for 7 distinct clients: all clients are selected and trained for 500 communication rounds, and each client randomly samples one of the corresponding 30 subsets for the local training---this challenging time-evolving scenario mimics the realistic client data sampling scheme (from some underlying distributions).
Notably, the training strategy used in this section is applicable to all FL baselines and CFL methods, and we assess the model performance using a global test data set\footnote{
	The training loss/accuracy on \eqref{CFL Formulation} formulation are closely aligned with that of global test data, as shown in Figure~\ref{Difference between train and test loss} of Appendix~\ref{sec:Additional Experiments}.
	We here only present the global test results for the common interests in practice.
}.
We carefully tune the hyper-parameters in all algorithms (details refer to Appendix~\ref{sec:realistic setup appendix}), and report the optimal results (i.e.\ mean test accuracy across the past 5 best epochs) after repeated trials.

\textit{Approximation techniques in CFL methods.}
\label{Approximation techniques in CFL}
Below, we review three representative types of information approximation techniques in CL, and use them to investigate the effect of various types of information loss under the \eqref{CFL Formulation} formulation.
We refer to Appendix~\ref{sec:Approximation Methods Appendix} for a more comprehensive introduction and discussion.
\begin{itemize}[nosep,leftmargin=12pt]
	\item \textit{Regularization methods.}
	      Instead of keeping data sets from previous rounds, we keep track the gradients and Hessian matrices of local objective functions, and use
	      Taylor Extension to approximate the local objective functions in previous rounds.
	      Note that the trade-off between Hessian estimation and computational overhead constrains the practical feasible of such approach.

	\item \textit{Core set methods.}
	      Another simple yet effective treatment in CL lines in the category of Exemplar Replay~\citep{rebuffi2017icarl,castro2018end}.
	      This method selects and regularly saves previous core-set samples (a.k.a.\ exemplars), and replays them with the current local data sets.

	\item \textit{Generative methods.}
	      Maintaining a core set for each client may become impractical when learning scales to millions of clients.
	      To ensure a privacy-preserved federated learning, the generative models~\citep{goodfellow2014generative} could be used locally to capture the local data distribution: fresh data samples will be generated on the fly and combined with the current local data set.
	      For the sake of simplicity, we use Markov Chain Monte Carlo~\citep{nori2014r2} in our assessment.
\end{itemize}

\subsection{Experiments Results for Real Datasets} \label{sec:results}
In this section, we conduct experiments on real data sets and show the effectiveness of CFL-based methods.
\subsubsection{Comments on Different CFL Approximation Techniques}
In Table \ref{Ablation Study Table}, we examine the performance of several approximation techniques using the split-CIFAR10 data set.
We can conclude that
\begin{enumerate}[nosep,leftmargin=12pt]
	\item \emph{Core set methods outperform other methods by a significant margin.}
	      The simple choice of ``naive core set'', i.e.\ randomly and uniformly sample data from the local data set, surpasses the sampling technique described in iCaRL~\citep{rebuffi2017icarl}\footnote{
		      The algorithmic details of the sampling method in iCaRL refer to Algorithm \ref{iCaRL Core Set} in Appendix~\ref{sec:Approximation Methods Appendix}.
	      } for CL, despite their faster convergence in the initial training phase.
	\item \emph{The quality of Hessian estimation matters for regularization based methods.}
	      PyHessian~\citep{Yao2020pyhessian}, as a method to approximate diagonal Hessian matrix, is slightly preferable than Fisher Information Matrix, though the latter one involves less computation.
	\item \emph{The performance of generative methods is restricted}, and we hypothesize that the poor quality of generated samples contributes to the constraint.
\end{enumerate}

In the subsequent evaluation, we consider naive core set sampling for CFL-Core-Set method, and use PyHessian for CFL-Regularization.
We exclude the results of generative methods, due to the trivial performance gain and significant computational overhead.

\begin{table*}[!t]
	\centering
	\caption{
		{Top-1 accuracy for different choices of approximation techniques in CFL.}
		We train ResNet18 on split-CIFAR10 data set (w/ $\alpha = 0.2$) for 300 communication rounds, and the data set is partitioned to 300 subsets for 10 different clients.
		All examined algorithms use \fedavg as the backbone.
	}
	\resizebox{1.\textwidth}{!}{%
		\begin{tabular}{l c c c c c c c}
			\toprule
			\multirow{2}{*}{Algorithm} & \multicolumn{2}{c}{Regularization Methods} & \multicolumn{4}{c}{Core Set Methods} & Generative Methods                                                                                     \\
			\cmidrule(lr){2-3} \cmidrule(lr){4-7} \cmidrule(lr){8-8}

			                           & PyHessian                                  & Fisher Information Matrix            & Naive (Small Set)  & Naive (Large Set)     & iCaRL (Small Set) & iCaRL (Large Set) & MCMC              \\

			\midrule

			Accuracy                   & $74.15 \pm 0.66$                           & $73.70 \pm 0.39$                     & $77.84 \pm 0.06$   & \bm{$78.90 \pm 0.09$} & $76.97 \pm 0.16$  & $77.09 \pm 0.09$  & $74.18 \pm 0.08 $ \\

			\bottomrule
		\end{tabular}%
	}

	\label{Ablation Study Table}
\end{table*}

\begin{table*}[!t]
	\centering
	\caption{{Top-1 accuracy of various CFL methods on diverse data sets} for training ResNet18 with 500 communication rounds.
		In order to observe a noticeable performance difference on Fashion-MNIST,
		we
		use $\alpha = 0.1$ instead.
		All examined algorithms use \fedavg as the backbone.
		Both CFL-Regularization and CFL-Regularization-Full are regularization based method, and the difference lies on where the regularization is applied: the full version applies regularization to all layers while the other only considers the top layers.
		$R_{r}$ indicates the accuracy of CFL-Regularization, while $R_{f}$ denotes the accuracy of \fedavg.
	}
	\resizebox{1.\textwidth}{!}{%
		\begin{tabular}{l c cc}
			\toprule
			Algorithm                          & Accuracy on Fashion-MNIST ($\%$)    & Accuracy on CIFAR10 ($\%$)    & Accuracy on CIFAR100 ($\%$)    \\
			\midrule
			\; \fedavg                         & $86.75 \pm 0.14$                    & $70.51 \pm 0.19$              & $49.97 \pm 0.19$               \\
			\; CFL-Regularization              & $87.02 \pm 0.21$                    & $70.86 \pm 0.31$              & $50.69 \pm 0.06$               \\
			\; CFL-Core-Set                    & $\bm{88.32 \pm 0.12}$           & $\bm{81.48 \pm 0.24}$     & $\bm{53.17 \pm 0.08}$      \\
			\; CFL-Regularization-Full         & $77.37 \pm 0.63$                    & $33.09 \pm 1.45$              & $14.84 \pm 0.12$               \\
			\bottomrule
			\toprule
			Metric                             & Improvement on Fashion-MNIST ($\%$) & Improvement on CIFAR10 ($\%$) & Improvement on CIFAR100 ($\%$) \\
			\midrule
			\; Absolute ($R_{r} - R_{f}$)      & 0.27                                & 0.35                          & $0.72$                         \\
			\; Ratio ($(R_{r} - R_{f}) / R_f$) & 0.31                                & 0.50                          & 1.44                           \\

			\bottomrule
		\end{tabular}%
	}

	\label{Performance on Different data sets}
\end{table*}

\subsubsection{Understanding Various Approximated CFL Implementations on Different Datasets}
Table~\ref{Performance on Different data sets} experimentally studies the impacts of various information approximation techniques in CFL methods---as discussed in Section~\ref{sec:setup}---and compares them with the backbone algorithm of CFL methods, i.e.\ \fedavg.
We have the following consistent findings on different data sets:
\begin{enumerate}[nosep,leftmargin=12pt]
	\item \emph{The improvement of CFL methods over \fedavg becomes larger for more challenging tasks}, as shown in the bottom of the Table \ref{Performance on Different data sets} for regularization based methods.
	      This might reflect the fact that the precision of the information approximation is more crucial for complicated tasks.
	      Similarly, we demonstrate in Figure~\ref{Loss on Different data sets} (in Appendix~\ref{sec:Additional Experiments}) that \emph{CFL methods offer better resistance to time-evolving non-iid data} and the significance of such finding is depending on the task difficulty.
	\item \emph{For CFL-Regularization method, applying regularization terms on top layers works better than that on all layers.}
	      This observation matches the recent work on decoupling feature extractor and classifier~\citep{collins2021exploiting,chen2021bridging,luo2021no}: the bottom layers are more generic across tasks and can serve as a global feature extractor, while the top layers are subject to task-specific information.
	      %
\end{enumerate}

We further investigate the connection between the information loss and the learning performance in Figure \ref{Information Loss Figure}.
We examine CFL methods with naive core sets, where the degree of information loss\footnote{
	We estimate the information loss by using $\frac{1}{tS}\sum_{i = 1}^{S} \sum_{\tau=1}^{t} \norm{\Delta_{\tau, i} (\momega)}$, where $\Delta_{\tau, i} (\momega)$ is defined in Equation~\eqref{information loss definition}, $S$ is the number of clients chosen in each round, and $t$ is the number of communication rounds.
} can be changed by altering the core set size from $20$ to $150$, with the same random seed and optimizers.
Figure \ref{Information Loss Figure} depicts that the \emph{performance of models is highly linked to the value of information loss: the lesser the information loss, the higher the performance.}

\begin{figure}
	\centering
	\includegraphics[width=.45\textwidth,]{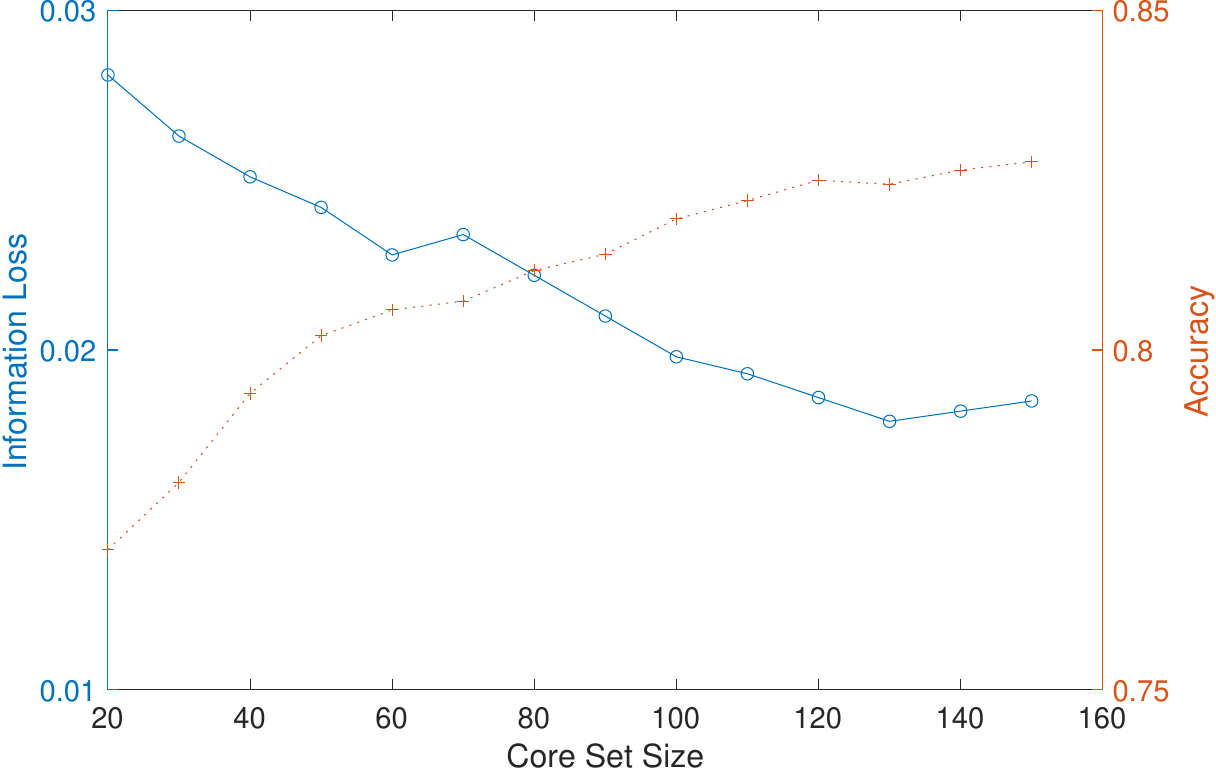}
	\caption{\small \textcolor{myblue}{Information loss} (left y-axis) and \textcolor{myred}{accuracy} (right y-axis) for training ResNet18 on split-CIFAR10 with different core set sizes (x-axis) and $\alpha=0.2$.}
	\label{Information Loss Figure}
\end{figure}

\subsubsection{Superior Performance of CFL Methods over Other Strong FL Baselines}
Alongside the comparison between CFL methods and \fedavg on various data sets (Table~\ref{Performance on Different data sets}), in Table~\ref{Performance of different algorithms} \emph{we verify the efficacy of CFL methods over other strong FL baselines}, on split-CIFAR10 data set.
We also examine MimeLite~\citep{karimireddy2020mime}, a method suggested for stateless FL scenarios.
Note that we exclude the results illustration for methods like SCAFFOLD~\citep{karimireddy2020scaffold} and Ditto~\citep{li2021ditto}, due to their infeasibility to be applied in our continual scenarios\footnote{
	We also naively adapted and examined these methods, but we cannot observe significant performance gains.
}.
We also examine overlapping local data sets in Table \ref{Performance on overlap local data sets} of Appendix \ref{sec:Additional Experiments}.
Results reveal that the advancement of CFL methods is consistent with Table~\ref{Performance of different algorithms}.

\begin{table*}[!t]

	\centering
	\caption{{Comparing the SOTA FL baselines with several CFL methods}, for training ResNet18 on split-CIFAR10 data set with different degrees of non-iid-ness $\alpha$ (and with total 500 communication rounds).}
	\resizebox{1.\textwidth}{!}{%
		\begin{tabular}{l c c c c c c}
			\toprule
			\multirow{2}{*}{Accuracy} & \multicolumn{3}{c}{FL baselines} & \multicolumn{3}{c}{CFL methods}                                                                                                \\
			\cmidrule(lr){2-4} \cmidrule(lr){5-7}

			                          & FedAvg                           & FedProx                         & MimeLite         & CFL-Core-Set          & CFL-Regularization & CFL-Regularization + FedProx \\

			\midrule

			$\alpha = 0.1$            & $70.51 \pm 0.19$                 & $71.57 \pm 0.17$                & $69.55 \pm 0.36$ & \bm{$81.48 \pm 0.24$} & $70.86 \pm 0.31$   & $71.50 \pm 0.07$             \\

			$\alpha = 0.2$            & $77.98 \pm 0.36$                 & $78.40 \pm 0.22$                & $78.65 \pm 0.26$ & \bm{$84.41 \pm 0.11$} & $78.76 \pm 0.26$   & $79.04 \pm 0.01$             \\

			\bottomrule
		\end{tabular}%
	}
	\label{Performance of different algorithms}
\end{table*}

%






\begin{table*}[!t]
	\centering
	\caption{
		{Learning with stateless clients}, regarding training ResNet18 on split-CIFAR10 with $\alpha = 0.2$.
		All examined algorithms use \fedavg as the backbone.
		The details w.r.t.\ CFL-Regularization refer to Appendix~\ref{sec:Approximation Methods Appendix}.
	}
	\resizebox{1.\textwidth}{!}{%
		\begin{tabular}{l c c c c c c c}
			\toprule
			\multirow{2}{*}{Algorithm} & \multicolumn{3}{c}{FL baselines} & \multicolumn{2}{c}{CFL methods}                                                                        \\
			\cmidrule(lr){2-4} \cmidrule(lr){5-6}

			                           & FedAvg                           & MimeLite                        & FedProx          & CFL-Regularization & CFL-Regularization + FedProx \\

			\midrule
			Accuracy                   & $78.07 \pm 0.06$                 & $78.47 \pm 0.19$                & $78.48 \pm 0.05$ & $79.07 \pm 0.25$   & \bm{$79.32 \pm 0.07$}        \\

			\bottomrule
		\end{tabular}%
	}

	\label{Performance on scenario S2}
\end{table*}


\subsubsection{Investigating Stateless FL Scenario}
In previous numerical investigations, we examine the \emph{stateful clients} and assume that these clients would learn and retain their client states throughout the training process.
The remainder of this section will discuss learning with \emph{stateless clients} (i.e.\ each client only appear once).
To do this, we change the data partition strategy such that fresh data partitions are sampled on the fly for each client and communication round.

Table~\ref{Performance on scenario S2} evaluates the above-mentioned scenario.
We find that \emph{CFL-Regularization methods consistently outperform alternative FL baselines}, similar to the observations in Table~\ref{Performance of different algorithms}.
However, due to the nature of stateless under the privacy concern, the advancements of core set method in CFL algorithm cannot be seen in this scenario.
Additionally, we discuss the applicability of different algorithms in a variety of time-evolving scenarios in Table~\ref{Applicable Scenarios for Different Algorithms} of Appendix~\ref{sec:Additional Experiments}.

\section{Conclusion and Future Work}

In this paper, we formally define Continual Federated Learning (CFL) and demonstrate the necessity of saving information from previous rounds in order to improve convergence through theoretical analysis. Additionally, we find that the quality of approximation methods has a significant impact on the convergence of CFL.

We propose that designing approximation methods with reduced information loss in Federated Learning (FL) would be a valuable research direction. Additionally, exploring the potential of more advanced Continual Learning (CL) methods to enhance FL is worth investigating.

\clearpage
\bibliography{paper}
\bibliographystyle{configuration/iclr2022_conference}
\clearpage
\appendix
\onecolumn
{
	\hypersetup{linkcolor=black}
	\parskip=0em
	\renewcommand{\contentsname}{Contents of Appendix}
	\tableofcontents
	\addtocontents{toc}{\protect\setcounter{tocdepth}{3}}
}

\section{Techniques}
Here we show some technical lemmas which are helpful in the theoretical proof.

\begin{lemma}[Linear convergence rate {\citep[Lemma 1]{karimireddy2020scaffold}}]
	\label{linear convergence rate}
	For every non-negative sequence $\{d_{r-1}\}_{r \ge 1}$, and any parameter $\mu > 0$, $T \ge \frac{1}{2 \eta_{max} \mu}$, there exists a constant step-size $\eta \le \eta_{max}$ and weights $\omega_t = (1 - \mu \eta)^{1-t}$ such that for $W_T = \sum_{t=1}^{T+1} \omega_t$
	\begin{small}
		\begin{talign*}
			\Phi_{T} = \frac{1}{W_T} \sum_{t=1}^{T+1} \left(\frac{\omega_{t-1}}{\eta}(1 - \mu \eta) d_{t-1} - \frac{\omega_t}{\eta} d_t \right)
			= \cO \left(  \mu d_0 exp(- \mu \eta_{max}T) \right) \,.
		\end{talign*}
	\end{small}
\end{lemma}

\begin{lemma}[Sub-linear convergence rate {\citep[Lemma 2]{karimireddy2020scaffold}}]
	\label{sub-linear convergence rate}
	For every non-negative sequence $\{d_{r-1}\}_{r \ge 1}$ and any parameters $\eta_{max} > 0$, $c_1 \ge 0$, $c_2 \ge 0$, $T \ge 0$, there exists a constant step-size $\eta \le \eta_{max}$ such that
	\begin{small}
		\begin{talign*}
			\Phi_{T} = \frac{1}{T + 1} \sum_{t=1}^{T+1} \left(  \frac{d_{t-1}}{\eta} - \frac{d_{t}}{\eta} + c_1 \eta + c_2 \eta^2 \right)
			\le \frac{d_0}{\eta_{max}(T+1)} + \frac{2 \sqrt{c_1 d_0}}{\sqrt{T + 1}} + 2 \left( \frac{d_0}{T + 1} \right)^{\frac{2}{3}} c_2^{\frac{1}{3}} \,.
		\end{talign*}
	\end{small}
\end{lemma}

\begin{lemma}[Relaxed triangle inequality {\citep[Lemma 3]{karimireddy2020scaffold}}]
	\label{relaxed triangle inequality}
	Let $\{\vv_1, ..., \vv_\tau \} $ be $\tau$ vectors in $\cR^{d}$. Then the following is true:
	\begin{small}
		\begin{talign*}
			\|\frac{1}{\tau} \sum_{i=1}^{\tau} \vv_i\|^{2} & \le \frac{1}{\tau} \sum_{i=1}^{\tau} \|\vv_i\|^{2} \,.
		\end{talign*}
	\end{small}
\end{lemma}

\begin{lemma}[Separating mean and variance {\citep[Lemma 14]{stich2020error}}]
	\label{separating mean and variance}
	Let ${\Xi_1, \Xi_2, ..., \Xi_\tau}$ be $\tau$ random variables in $\cR^{d}$ which are not necessarily independent. First suppose that their mean is $E[\Xi_i] = \xi_i$, and variance is bounded as $\E{\norm{\Xi_i - \xi_i}^2} \le M\|\xi_i\|^{2} + \sigma^{2}$, then the following holds
	\begin{small}
		\begin{talign*}
			\E{\norm{\sum_{i=1}^{\tau} \Xi_i}^2} \le (\tau + M)\sum_{i=1}^{\tau} \|\xi_i\|^{2} + \tau \sigma^{2} \,.
		\end{talign*}
	\end{small}
\end{lemma}

\begin{lemma}[Perturbed strongly convexity {\citep[Lemma 5]{karimireddy2020scaffold}}]
	\label{perturbed strong convexity}
	The following holds for any $L$-smooth and $\mu$-strongly convex function $h$, and any $\xx$, $\yy$, $\zz$ in the domain of $h$
	\begin{talign*}
		\langle \nabla h(\xx), \zz - \yy \rangle \ge h(\zz) - h(\yy) + \frac{\mu}{4} \|\yy - \zz\|^{2} - L \|\zz - \xx\|^{2} \, .
	\end{talign*}
\end{lemma}

\begin{lemma}
	\label{time sum}
	Define $S = \sum_{t=1}^{T} \frac{p_t}{t}$ where $\sum_{t=1}^{T} p_t = 1$, and $p_t \le p_{t+1}$
	then $S \le \frac{1}{T} \sum_{t=1}^{T} \frac{1}{t} = \cO \left( \frac{lnT}{T} \right)$.
\end{lemma}

\section{Convergence Rate of CFL}

In this section, we will provide the theoretical proof of the convergence rates of CFL and FedAvg in time-varying scenarios under Assumptions~\ref{Smoothness and convexity assumption}--\ref{Bounded information loss assumption}). We will prove Theorems~\ref{Convergence rate of CFL}, \ref{Convergence rate of FedAvg}, \ref{Convergence rate of CFL with correlated time drifts}, and \ref{Optimal choice of weights with correlated time drifts}. Before proceeding with the proof, we will introduce some useful lemmas.
\subsection{Bound of Gradient Noise}
\label{sec:Proof of lemma bound of distribution drift}

\begin{lemma}[Bound of Gradient Noise]
	Suppose $f_i(\momega)$ is the local objective function, and $f(\momega) = \Eb{f_i(\momega)}$ is the global objective function.
	Define $\nabla f(\momega) = \nabla f_i(\momega) + \mvarsigma_i$, $\E[\mvarsigma_i] = 0$,
	and assume $\Eb{ \norm{ \mvarsigma }^2 \vert \momega } \le A^2 \Eb{ \norm{ \nabla f(\momega) }^2 } + B^2$, we have
	\begin{small}
		\begin{talign*}
			\Eb{ \norm{ \nabla f_i (\momega) }^2 } \le (A^2 + 1) \Eb{ \norm{ \nabla f(\momega) }^2 } + B^2 \,.
		\end{talign*}
	\end{small}
\end{lemma}

\begin{proof}
	\begin{small}
		\begin{talign*}
			\Eb{ \norm{ \nabla f_i (\momega) }^2 } & = \Eb{ \norm{ \nabla f(\momega) + \mvarsigma_i }^2 }
			= \Eb{ \norm{ \nabla f(\momega) }^2 } + \Eb{ \norm{ \mvarsigma_i }^2 }     \, ,                       \\
			& \le (A + 1) \Eb{ \norm{ \nabla f(\momega) }^2 } + B \, .
		\end{talign*}
	\end{small}
\end{proof}

\subsection{Bounded Gradient Noise of CFL}

\label{sec:proof of lemma bounded drift}

\begin{lemma}[Bounded Gradient Noise of CFL]
	\label{Bounded Gradient Noise of CFL}
	For Formulation \eqref{CFL Formulation},
	consider \algopt and FedAvg, we can bound the gradient drift as\\
	(1) FedAvg only optimize current local objective functions, thus
	\begin{small}
		\begin{talign*}
			& \E{ \norm{ \nabla f_{t,i} (\momega) }^2 } \le (1 + A^2 + B^2) \E{ \norm{ \nabla f(\momega) }^2 } + G^2 + D^2 \,.
		\end{talign*}
	\end{small}%
	(2) For \algopt, in general (same clients can appear in different rounds), we have
	\begin{small}
		\begin{talign*}
			& \E{ \norm{ \sum_{\tau=1}^{t} p_{\tau, i} \nabla f_{\tau,i}(\momega) }^2 } \le \left( 1 + B^2 + A^2 \sum_{\tau=1}^{t} p_{\tau, i}^2 \right) \E{ \norm{ \nabla f(\momega) }^2 }
			+ G^2 + D^2 \sum_{\tau=1}^{t} p_{\tau, i}^2 \,.
		\end{talign*}
	\end{small}%
	(3) For \algopt, when clients only participate in training once($\mdelta_{t, i}$ are independent for all $t$ and $i$), we have
	\begin{small}
		\begin{talign*}
			\E{ \norm{ \sum_{\tau=1}^{t} p_{\tau, i} \nabla f_{\tau,i}(\momega) }^2 }
			\le \left( 1 +(B^2 + A^2) \sum_{\tau=1}^{t} p_{\tau, i}^2 \right) \E{ \norm{ \nabla f(\momega) }^2 }
			+ (G^2 + D^2) \sum_{\tau=1}^{t} p_{\tau, i}^2 \,.
		\end{talign*}
	\end{small}
\end{lemma}

\begin{proof}
	Using Assumption \ref{Bounded gradient noise of CFL assumption},
	together with the fact that $\nabla f_{t, i}(\momega) = \nabla f(\momega) + \mdelta_{t, i} + \mxi_{t, i}$ (Here to simplify the notation, we use $\{t, i\}$ pair to denote the client participate in training on round $t$ and number $i$), we have
	\begin{small}
		\begin{talign*}
			\E{ \norm{ \nabla f_{t, i}(\momega) }^2 } & = \E{ \norm{ \nabla f(\momega) + \mdelta_{t, i} + \mxi_{t, i} }^2 }
			= \E{ \norm{ \nabla f(\momega) }^2 } + \E{ \norm{ \mdelta_{t, i} }^2 } + \E{ \norm{ \mxi_{t, i} }^2 }              \\
			& \le (1 + A^2 + B^2) \E{ \norm{ \nabla f(\momega) }^2 } + G^2 + D^2 \,.
		\end{talign*}
	\end{small}%
	The second inequality is based on the independence of noise, and directly use Assumption \ref{Bounded gradient noise of CFL assumption} we get the last inequality.
	Similarly, we have
	\begin{small}
		\begin{talign*}
			\E{ \norm{ \sum_{\tau=1}^{t} p_{\tau, i} \nabla f_{\tau, i}(\momega) }^2 } & = \E{ \norm{ \sum_{\tau=1}^{t} p_{\tau, i} (\nabla f_{\tau, i}(\momega) + \mdelta_i + \mxi_{\tau, i}) }^2 }                                      \\
			& = \E{ \norm{ \nabla f(\momega) }^2 } + \E{ \norm{ \mdelta_i }^2 } + \E{ \norm{ \sum_{\tau=1}^{t} p_{\tau, i} \mxi_{\tau, i} }^2 }                \\
			& \le \left(1 + B^2 + A^2 \sum_{\tau=1}^{t} p_{\tau, i}^2\right) \E{ \norm{ \nabla f(\momega) }^2 } + G^2 + D^2\sum_{\tau=1}^{t} p_{\tau, i}^2 \,.
		\end{talign*}
	\end{small}%
	For the last inequality, when $\mdelta_{t_1, i}$ and $\mdelta_{t_2, i}$ are independent for all $t_1$, $t_2$
	\begin{small}
		\begin{talign*}
			\E{ \norm{ \sum_{\tau=1}^{t} p_{\tau, i} \nabla f_{\tau, i}(\momega) }^2 } & = \E{ \norm{ \sum_{\tau=1}^{t} p_{\tau, i} (\nabla f_{\tau, i}(\momega) + \mdelta_{t, i} + \mxi_{\tau, i}) }^2 }                                                     \\
			& = \E{ \norm{ \nabla f(\momega) }^2 } + \E{ \norm{ \sum_{\tau=1}^{t} p_{\tau, i} \mdelta_{t, i} }^2 } + \E{ \norm{ \sum_{\tau=1}^{t} p_{\tau, i} \mxi_{\tau, i} }^2 } \\
			& \le \left(1 + (A^2 + B^2) \sum_{\tau=1}^{t} p_{\tau, i}^2\right) \E{ \norm{ \nabla f(\momega) }^2 } + (G^2 + D^2)\sum_{\tau=1}^{t} p_{\tau, i}^2 \,.
		\end{talign*}
	\end{small}
\end{proof}

\subsection{Bounded Approximation Error}
\label{sec:proof of bounded approximation error}

\begin{lemma}[Bounded Approximation Error]
	\label{bounded approximation error}
	For $\Delta_{t, i}(\momega) = \nabla f_{t, i}(\momega) - \nabla \tilde{f}_{t, i}(\momega)$ (c.f.\ Definition~\ref{information loss definition}), if we assume that $\norm{\nabla^2 f_{t, i} (\omega_1) - \nabla^2 f_{t, i} (\omega_2)} \le \epsilon$, { when use Taylor Extension,
			\begin{small}
				\begin{talign*}
					\nabla \tilde{f}_{t, i}(\momega) = \nabla f_{t, i} (\momega_t) + \nabla^2 f_{t, i} (\momega_t) (\momega - \momega_t)  \,,
				\end{talign*}
			\end{small}
		}%
	we have
	\begin{small}
		\begin{talign*}
			\norm{\Delta_{t, i}(\momega)} \le \epsilon \norm{ \momega - \hat{\momega}_{t, i} } \,.
		\end{talign*}
	\end{small}
\end{lemma}

\begin{proof}
	\begin{small}
		\begin{talign*}
			\norm{\Delta_{t,i}(\momega)} & = \norm{\nabla f_{t,i}(\momega) - \nabla \tilde{f}_{t,i}(\momega)}
			= \norm{ \nabla f_{t, i} (\momega) - \nabla f_{t, i} (\momega_{t,i,K}) - \mH_{tjK} (\momega - \momega_{t,i,K})} \, ,
			\\
			& = \norm{ \nabla^2 f_{t, i} (\momega_{\mxi, t})(\momega - \momega_{t,i,K}) - \mH_{tjK} (\momega - \momega_{t,i,K})}
			= \norm{ (\nabla^2 f_{t, i} (\momega_{\mxi, t}) - \mH_{tjK})(\momega - \momega_{t,i,K})} \, ,
			\\
			& \le \epsilon \|\momega - \momega_{t,i,K}\| \, .
		\end{talign*}
	\end{small}%
	The first two equations come from the mean value theorem which says for continuous function $f$ in closed intervals $[a, b]$ and differentiable on open interval $(a, b)$, there exists a point $c \subseteq (a, b)$ such that $f'(c) = \frac{f(b)-f(a)}{b-a}$.
\end{proof}




\subsection{Proof of Theorem \ref{Convergence rate of CFL}}
\label{sec:Proof of Theorem Convergence rate of CFL-R}

In this section we will give the complete proof of convergence rate of Algorithm \ref{CFL Algorithm Framework Appendix} when the local objective functions are convex.

\begin{algorithm}[!t]
	\begin{algorithmic}[1]
		\Require{initial weights $\momega_0$, global learning rate $\eta_g$, local learning rate $\eta_l$}
		\For{round $t = 1, \ldots, T$}
		\myState{\textit{communicate} $\momega_t$ to the chosen clients.}
		\For{\textit{client} $i \in \cS_t$ \textit{in parallel}}
		\myState{initialize local model $\momega_{t,i,0} = \momega_t$.}
		\For{$k = 1, \ldots, K$}
		\myState{$\tilde{g}_{t,i,k}(\momega_{t,i,k-1}) = \left( p_{t, i} \nabla f_{t, i} (\momega_{t,i,k-1}) + \sum_{\tau=1}^{t-1} p_{\tau, i} \nabla \tilde{f}_{t, i} (\momega_{t,i,k-1}) \right) + \mnu_{t, i, k}$.}
		\myState{$\momega_{t,i,k} \gets \momega_{t,i,k-1} - \eta_l \tilde{g}_{t,i,k}(\momega_{t,i,k-1})$.}
		\EndFor
		\myState{\textit{communicate} $\Delta \momega_{t,i} \gets \momega_{t,i,K} - \momega_t$.}
		\EndFor
		\myState{$\Delta \momega_t \gets \frac{\eta_g}{S} \sum_{i \in \cS_t} \Delta \momega_{t,i}$.}
		\myState{$\momega_{t+1} \gets \momega_t + \Delta \momega_t$.}
		\EndFor
	\end{algorithmic}
	\mycaptionof{algorithm}{\small CFL Framework}
	\label{CFL Algorithm Framework Appendix}
\end{algorithm}

\subsubsection{Preliminaries}
The local objective function of round $t$ on client $i$ is
\begin{small}
	\begin{talign*}
		\bar{f}_{t,i}(\momega) = (p_{t, i} f_{t, i}(\momega) + \sum_{\tau=1}^{t-1} p_{\tau, i} \tilde{f}_{\tau, i}(\momega) \, .
	\end{talign*}
\end{small}
Without loss of generality, assume $\sum_{\tau=1}^{t} p_{\tau, i} = 1$.
Then follow the steps in Algorithm \ref{CFL Algorithm Framework Appendix}, we have
\begin{talign*}
	\nabla \bar{f}_{t,i}(\momega) =  p_{t, i} \nabla f_{t, i}(\momega) + \sum_{\tau=1}^{t-1} p_{\tau, i} \nabla \tilde{f}_{\tau, i}(\momega)  \, .
\end{talign*}
Then the noisy gradient in mini-batch SGD can be described as
\begin{talign*}
	\bar{g}_{t, i, k}(\momega_{t, i, k-1}) = \nabla \bar{f}_{t,i}(\momega_{t,i,k-1}) + \mnu_{t,i,k} \, .
\end{talign*}
Then follow the steps in Algorithm \ref{CFL Algorithm Framework Appendix}, we have
\begin{talign*}
	\Delta \momega_t    & = \frac{- \eta_l \eta_g}{N} \sum_{i=1}^{N} \sum_{k=1}^{K} \bar{g}_{t,i}(\momega_{t,i,k-1}) \, ,        \\
	E[\Delta \momega_t] & = \frac{- \eta_l \eta_g}{N} \sum_{i=1}^{N} \sum_{k=1}^{K} \nabla \bar{f}_{t,i}(\momega_{t,i,k-1}) \, .
\end{talign*}
Then let $\eta = K \eta_l \eta_g$, we have
\begin{talign}
	\Delta \momega_t    & = \frac{- \eta}{NK} \sum_{i=1}^{N} \sum_{k=1}^{K} \bar{g}_{t,i}(\momega_{t,i,k-1}) \, ,  \label{Delta omega}            \\
	E[\Delta \momega_t] & = \frac{- \eta}{NK} \sum_{i=1}^{N} \sum_{k=1}^{K} \nabla \bar{f}_{t,i}(\momega_{t,i,k-1}) \, . \label{Mean delta omega}
\end{talign}

\subsubsection{One Step Progress}

\label{sec: One Step Progress}

Optimizing Algorithm~\ref{CFL Algorithm Framework Appendix} for one step, we can find that
\begin{talign}
	\label{First part of one step}
	& \E{\norm{\momega_t + \Delta \momega_t - \momega^*}^2}
	= \|\momega_t - \momega^{*}\|^{2} - \underbrace{2E \langle \momega_t - \momega^{*}, \Delta \momega_t \rangle}_{A_1} + \underbrace{\E{\norm{\Delta \momega_t}^2}}_{A_2} \, .
\end{talign}
Firstly we consider $A_1$ part in Equation~\eqref{First part of one step}
\begin{talign}
	& - 2E \langle \momega_t - \momega^{*}, \Delta \momega_t \rangle
	= \frac{2 \eta}{NK} \sum_{i=1}^{N} \sum_{k=1}^{K} \langle \nabla \bar{f}_{t,i}(\momega_{t,i,k-1}), \momega^{*} - \momega_t \rangle \nonumber \, ,
	\\
	& = \frac{2 \eta}{NK} \sum_{i=1}^{N} \sum_{k=1}^{K} \langle p_{t, i} \nabla f_{t, i}(\momega_{t, i, k}) + \sum_{\tau = 1}^{t-1} p_{\tau, i} \nabla \tilde{f}_{\tau, i}(\momega_{t, i, k}), \momega^{*} - \momega_t \rangle \nonumber \, ,
	\\
	& = \frac{2 \eta}{NK} \sum_{i=1}^{N} \sum_{k=1}^{K} \sum_{\tau=1}^{t} \langle p_{\tau, i} \nabla f_{\tau, i}(\momega_{t, i, k}),  \momega^{*} - \momega_t \rangle - \frac{2 \eta}{NK} \sum_{i=1}^{N} \sum_{k=1}^{K} \sum_{\tau=1}^{t-1} p_{\tau, i} \langle \Delta_{\tau, i}(\momega_{t, i, k}) , \momega^{*} - \momega_t \rangle \nonumber \, ,
	\\
	& \le \frac{2 \eta}{NK} \sum_{i=1}^{N} \sum_{k=1}^{K} \sum_{\tau=1}^{t} p_{\tau, i} \langle \nabla f_{\tau, i}(\momega_{t, i, k}),  \momega^{*} - \momega_t \rangle
	+ \frac{2 \eta}{N} \sum_{i=1}^{N} \sum_{\tau=1}^{t-1} p_{\tau, i} R \norm{\momega^{*} - \momega_t}
	\, . \label{first step of A1}
\end{talign}

We use Inequality~\eqref{Mean delta omega} to obtain the first inequality, and Assumption~\ref{Bounded information loss assumption} and Cauchy–Schwarz Inequality for the last inequality. 
Then for the first term of Inequality~\eqref{first step of A1}, by applying Lemma~\ref{perturbed strong convexity}, we can conclude that
\begin{talign}
	\langle \nabla f_{\tau, i}(\momega_{t,i,k-1}), \momega^{*} - \momega_t \rangle & \le f_{\tau, i}(\momega^*) - f_{\tau, i}(\momega_t) - \frac{\mu}{4} \|\momega_t - \momega^{*}\|^2 + L \|\momega_{t,i,k-1} - \momega_t\|^2 \, . \label{using perturbed strong convexity}
\end{talign}
Combine Inequality~\eqref{first step of A1} and Inequality~\eqref{using perturbed strong convexity}, we have
\begin{talign}
	& - 2E \langle \momega_t - \momega^{*}, \Delta \momega_t \rangle \nonumber                        \\
	& \le -2 \eta (f_t(\momega_t) - f_t(\momega^*) + \frac{\mu}{4} \|\momega_t - \momega^*\|^2)
	+ \frac{2 \eta L}{NK} \sum_{i=1}^{N} \sum_{k=1}^{K} \|\momega_{t,i,k-1} - \momega_t\|^2 
 \nonumber  \\
	&\qquad + 2 \eta \sum_{\tau=1}^{t-1} p_{\tau, i} R \norm{\momega^{*} - \momega_t} \, . \label{A1-final}
\end{talign}
Secondly, we are willing to deal with $A_2$ in Equation~\eqref{First part of one step}. Define $c_{p_B, i} = 1 + B^2 + A^2(\sum_{\tau=1}^{t} p_{\tau, i}^2)$, $c_{p_G, i} = G^2 + D^2  (\sum_{\tau=1}^{t} p_{\tau, i}^2)$,
$c_{p_B} = \frac{1}{N} \sum_{i=1}^{N} c_{p_B, i}$,
and $c_{p_G} = \frac{1}{N} \sum_{i=1}^{N} c_{p_G, i}$
we have
\begin{small}
	\begin{talign}
		& \E{\norm{\Delta \momega_t }^2}
		= \frac{\eta^2}{N^2 K^2} \E{ \norm{
				\sum_{i=1}^{N} \sum_{k=1}^{K}  \left(
				p_{t, i} \nabla f_{t, i} (\momega_{t, i, k-1})
				+ \sum_{\tau=1}^{t-1} p_{\tau, i} \nabla \tilde{f}_{\tau, i} (\momega_{t, i, k-1})
				+ \mnu_{t, i, k} \right)
			}^2 } \, , \nonumber
		\\
		& = \frac{\eta^2}{N^2 K^2} \E{ \norm{
				\sum_{i=1}^{N} \sum_{k=1}^{K} \left(
				\sum_{\tau=1}^{t} p_{\tau, i} \nabla f_{\tau, i} (\momega_{t, i, k-1})
				- \sum_{\tau=1}^{t-1} p_{\tau, i} \Delta_{\tau, i} (\momega_{t, i, k-1})
				+ \mnu_{t, i, k}
				\right)
			}^2} \, , \nonumber
		\\
		& \overset{Lem \; \ref{separating mean and variance}}{\le} \frac{\eta^2}{N K} \sum_{i=1}^{N} \sum_{k=1}^{K} \E{ \norm{
				\sum_{\tau=1}^{t} p_{\tau, i} \nabla f_{\tau, i} (\momega_{t, i, k-1})
				- \sum_{\tau=1}^{t-1} p_{\tau, i} \Delta_{\tau, i} (\momega_{t, i, k-1})
			}^2 }
		+ \frac{\eta^2 \sigma^2}{N K} \, , \nonumber
		\\
		& \overset{Lem \; \ref{relaxed triangle inequality}}{\le} \frac{2\eta^2}{N K} \sum_{i=1}^{N} \sum_{k=1}^{K} \left(
		\E{ \norm{
				\sum_{\tau=1}^{t} p_{\tau, i} \nabla f_{\tau, i} (\momega_{t, i, k-1})
			}^2 }
		+ \E{ \norm{
				\sum_{\tau=1}^{t-1} p_{\tau, i} \Delta_{\tau, i} (\momega_{t, i, k-1})
			}^2 }
		\right)
		+ \frac{\eta^2 \sigma^2}{NK} \, , \label{A2 middle}
		\\
		& \overset{Lem \; \ref{Bounded Gradient Noise of CFL}}{\le} \frac{2 \eta^2}{N K} \sum_{i=1}^{N} \sum_{k=1}^{K} \left(
		\left(1 + B^2 + A^2(\sum_{\tau=1}^{t} p_{\tau, i}^2) \right) \E{ \norm{
				\nabla f(\momega_{t, i, k-1})
			}^2}
		+  \left( G^2 + D^2  (\sum_{\tau=1}^{t} p_{\tau, i}^2) \right)
		\right)  \nonumber                                                                                                       \\
		& + \frac{2 \eta^2}{N} \sum_{i=1}^{N} (\sum_{\tau=1}^{t-1} p_{\tau, i})^2 R^2
		+ \frac{\eta^2 \sigma^2}{NK} \, , \nonumber
		\\
		& \overset{Lem \; \ref{relaxed triangle inequality}}{\le} \frac{4 \eta^2}{N K} \sum_{i=1}^{N} \sum_{k=1}^{K}
		c_{p_B, i} \left( \E{ \norm{
				\nabla f(\momega_{t, i, k-1}) - \nabla f(\momega_{t})
			}^2}
		+ \E{ \norm{
				\nabla f(\momega_{t})
			}^2 }
		\right)
		+ 2 \eta^2 c_{p_G} \nonumber                                                                                            \\
		& + \frac{2 \eta^2}{N} \sum_{i=1}^{N} (\sum_{\tau=1}^{t-1} p_{\tau, i})^2 R^2
		+ \frac{\eta^2 \sigma^2}{NK} \, , \nonumber
		\\
		& \overset{Ass \; \ref{Smoothness and convexity assumption}}{\le} 16 \eta^2 L c_{p_B} (f(\momega_t) - f(\momega^*))
		+ \frac{8 \eta^2 L^2}{N K} \sum_{i=1}^{N} \sum_{k=1}^{K}
		c_{p_B, i} \E{ \norm{
				\momega_{t, i, k-1} - \momega_{t}
			}^2}
		+ 2 \eta^2 c_{p_G} \nonumber                                                                                            \\
		& + \frac{2 \eta^2}{N} \sum_{i=1}^{N} (\sum_{\tau=1}^{t-1} p_{\tau, i})^2 R^2
		+ \frac{\eta^2 \sigma^2}{NK} \, . \label{A2-final}
	\end{talign}
\end{small}

We now turn our attention to solving $\E{\norm{\momega_{t,i,k-1} - \momega_t}^2}$. By examining Inequalities~\eqref{A1-final} and~\eqref{A2-final}, we can see that this is the remaining task. We can write this expectation as follows
\begin{talign}
	\E{\norm{\momega_{t,i,k} - \momega_t}^2} & = \eta_l^2 \E{\norm{\sum_{\tau=1}^{k} \bar{g}_{t,i}(\momega_{t,i,\tau - 1})}^2} \, , \label{on step drift-1}                                                                                               \\
	& \overset{Lem \; \ref{separating mean and variance}}{\le} k \eta_l^2 \sum_{\tau = 1}^{k} \E{\norm{\nabla \bar{f}_{t,i} (\momega_{t,i,\tau - 1})}^2} + k \eta_l^2 \sigma^2 \,. \label{one step drift-2}
\end{talign}
We utilize Lemma~\ref{separating mean and variance} to derive Inequality~\eqref{one step drift-2}. The remaining challenge in this inequality is to bound $\E{\norm{\nabla \bar{f}_{t,i} (\momega)}^2}$. To address this, we consider that
\begin{talign}
	& \E{\norm{\nabla \bar{f}_{t,i} (\momega)}^2}
	= \E{ \norm{
			p_{t, i} \nabla f_{t, i} (\momega)
			+ \sum_{\tau=1}^{t-1} p_{\tau, i} \nabla \tilde{f}_{\tau, i} (\momega)
		}^2 } \, ,
	= \E{ \norm{
			\sum_{\tau=1}^{t} p_{\tau, i} \nabla f_{\tau, i} (\momega)
			- \sum_{\tau=1}^{t-1} p_{\tau, i} \Delta_{\tau, i} (\momega)
		}^2 } \nonumber \, ,
	\\
	& \overset{Lem \; \ref{relaxed triangle inequality}}{\le} 2 \E{ \norm{
			\sum_{\tau=1}^{t} p_{\tau, i} \nabla f_{\tau, i} (\momega)
		}^2 }
	+ 2 \E{ \norm{
			\sum_{\tau=1}^{t-1} p_{\tau, i} \Delta_{\tau, i} (\momega)
		}^2 } \nonumber \, ,
	\\
	& \overset{Lem \; \ref{Bounded Gradient Noise of CFL}}{\le} 2 c_{p_B, i} \E{ \norm{ \nabla f(\momega) }^2 }
	+ 2 c_{p_G, i}
	+ 2 (\sum_{\tau=1}^{t-1} p_{\tau, i})^2 R^2 \nonumber \, ,
	\\
	& \overset{Ass \; \ref{Smoothness and convexity assumption}}{\le} 4 L^2 c_{p_B, i} \E{ \norm{ \momega - \momega_t }^2 }
	+ 8 L c_{p_B, i}  \left( f(\momega_t) - f(\momega^*) \right)
	+ 2 c_{p_G, i}
	+ 2 (\sum_{\tau=1}^{t-1} p_{\tau, i})^2 R^2 \, . \label{end of gradient drift}
\end{talign}
Combine Inequalities~\eqref{one step drift-2} and~\eqref{end of gradient drift} we get,
\begin{talign}
	& \E{\norm{\momega_{t,i,k} - \momega_t}^2}
	\le 4 k^2 \eta_l^2 L^2 c_{p_B, i} \E{\norm{\momega_{t,i,k-1} - \momega_t}^2}
	+ 8 k^2 \eta_l^2 L c_{p_B, i} (f(\momega_t) - f(\momega^*)) + 2 k^2 \eta_l^2 c_{p_G, i} \nonumber    \\
	& + 2 k^2 \eta_{l}^{2} (\sum_{\tau=1}^{t-1} p_{\tau, i})^2 R^2 + k \eta_l^2 \sigma^2 \, . 
        \label{eq 18}
\end{talign}
By combining Inequalities~\eqref{one step drift-2} and ~\eqref{eq 18} iteratively, we obtain the following result
\begin{talign}
	& \E{\norm{\momega_{t,i,k} - \momega_t}^2} \nonumber
	\le \sum_{\tau = 0}^{k-1} (8 K^2 \eta_l^2 L c_{p_B, i} (f(\momega_t) - f(\momega^*))
	\nonumber                                             \\
	& + 2 K^2 \eta_l^2 c_{p_G, i}
	+ 2 K^2 \eta_{l}^{2} (\sum_{\tau=1}^{t-1} p_{\tau, i})^2 R^2
	+ K \eta_l^2 \sigma^2)(4 K^2 \eta_l^2 L^2 c_{p_B, i})^{\tau} \nonumber \, ,
	\\
	& \le \frac{
		8 K^2 \eta_l^2 L c_{p_B, i} (f(\momega_t) - f(\momega^*))
		+ 2 K^2 \eta_l^2 c_{p_G, i}
		+ 2 K^2 \eta_{l}^{2} (\sum_{\tau=1}^{t-1} p_{\tau, i})^2 R^2
		+ K \eta_l^2 \sigma^2}{1 - 4 K^2 \eta_l^2 L^2 c_{p_B, i}} \, . \label{one step drift final}
\end{talign}
Combine the Inequalities~\eqref{First part of one step},~\eqref{A1-final},~\eqref{A2-final}, and~\eqref{one step drift final}, we get
\begin{talign}
	& \E{\norm{\momega_t + \Delta \momega_t - \momega^*}^2} \nonumber     \, ,      \\
	& \le (1 - \frac{\mu \eta}{2}) \E{\norm{\momega_t - \momega^*}^2} \nonumber \\
	& + (	-2 \eta
	+ 16 \eta^2 L c_{p_B}
	+ \frac{1}{N} \sum_{i=1}^{N} \frac{16 K^2 \eta \eta_{l}^2 L c_{p_B, i} (1 + 4 \eta L c_{p_B, i})}{1 - 4 K^2 \eta_{l}^2 L^2 c_{p_B, i}}
	) (f(\momega_t) - f^*) \nonumber
	+ 2 \eta^2 c_{p_G}                                                           \\
	& + \frac{\eta^2 \sigma^2}{N K} \nonumber
	+ \frac{1}{N} \sum_{i=1}^{N} \left(
	2 \eta^2 (\sum_{\tau=1}^{t-1} p_{\tau, i})^2
	+ \frac{4 K^2 \eta \eta_l^2 L (\sum_{\tau=1}^{t-1} p_{\tau, i})^2 (1 + 4 \eta L c_{p_B, i})}{1 - 4 K^2 \eta_{l}^2 L^2 c_{p_B, i}}
	\right) R^2                                                                  \\
	& + \frac{1}{N} \sum_{i=1}^{N}
	\frac{2 K \eta \eta_l^2 L (2 K c_{p_G, i} + \sigma^2)(1 + 4 \eta L c_{p_B, i})}{1 - 4 K^2 \eta_l^2 L^2 c_{p_B, i}}
	+ \frac{2 \eta}{N} \sum_{i=1}^{N} \sum_{\tau=1}^{t-1} p_{\tau, i} R \norm{\momega_t - \momega^{*}} \label{on step first inequality} \, .
\end{talign}
To promise convergence, we must have
\begin{talign}
	-2 \eta + 16 \eta^2 L c_{p_B} + \frac{1}{N} \sum_{i=1}^{N} \frac{16 K^2 \eta \eta_l^2 L^2 c_{p_B,i} (1 + 4 \eta L c_{p_B, i})}{1 - 4 K^2 \eta_l^2 L^2 c_{p_B, i}} \le 0 \, .
    \label{eq 21}
\end{talign}
Using the fact that $\eta = K \eta_l \eta_g$, we can solve Inequality~\eqref{eq 21}and determine that if $\eta$ satisfies the condition
\begin{talign*}
	\eta & \le \frac{\sqrt{3 \eta_g^2 + 4 c_{p_B, i} \eta_g^4} - \sqrt{4 c_{p_B, i} \eta_g^4}}{6L \sqrt{c_{p_B, i}}} \, ,
\end{talign*}
for any $c_{p_B, i}$, convergence is guaranteed. Defining $\hat{c}{p_B}$ as the largest $c{p_B, i}$, we can further simplify this condition to
\begin{talign*}
\eta \le \frac{\sqrt{3 \eta_g^2 + 4 \hat{c}{p_B} \eta_g^4} - \sqrt{4 \hat{c}_{p_B} \eta_g^4}}{12L \sqrt{\hat{c}_{p_B}}} \, .
\end{talign*}
This leads to the inequality
\begin{talign*}
1 - 4 K^2 \eta_l^2 L^2 c_{p_B, i} \ge \frac{33 - 8 \eta_g^2 \hat{c}_{p_B} + 4 \sqrt{3 \eta_g^2 \hat{c}_{p_B} + 4 \eta_g^4 \hat{c}_{p_B}^2}}{36} \, ,
\end{talign*}
and therefore
\begin{talign*}
\frac{1 + 4 \eta L c{p_B, i}}{1 - 4 K^2 \eta_l^2 L^2 c_{p_B, i}} \le \frac{1}{12} - \frac{7}{33 + \sqrt{3 \eta_g^2 \hat{c}_{p_B} + 4 \eta_g^4 \hat{c}_{p_B}^2} - 2 \eta_g^2 \hat{c}_{p_B}} \le \frac{1}{12} \, .
\end{talign*}

Let $c_{R,i} = (\sum_{\tau=1}^{t-1} p_{\tau,i})^2 R^2$ and $c_{R} = \frac{1}{N} \sum_{i=1}^{N} c_{R,i}$. We assume $c_1 = 2c_{p_G} + \frac{\sigma^2}{NK} + 2c_R$, $c_2 = \frac{Lc_{p_G}}{3\eta_g^2} + \frac{L\sigma^2}{6\eta_g^2K} + \frac{Lc_R}{3\eta_g^2}$, and $c_3 = \frac{2}{N} \sum_{i=1}^{N} \sum_{\tau=1}^{t-1} p_{\tau,i}R$. Thus, Inequality \eqref{on step first inequality} can be rewritten as
\begin{talign*}
	& \E{\norm{\momega_t + \Delta \momega_t - \momega^*}^2} \le (1 - \frac{\mu \eta}{2}) \E{\norm{\momega_t - \momega^*}^2}
	-\eta (f(\momega_t) - f^*) + c_1 \eta^2 + c_2 \eta^3 + c_3 \eta \norm{\momega_t - \momega^*}\, .
\end{talign*}
Move $f(\momega_t) - f^*$ to the left side, we get
\begin{talign}
	f(\momega_t) - f(\momega^*) \le
	\frac{1}{\eta} (1 - \frac{\mu \eta}{2}) \E{\norm{\momega_t - \momega^*}^2}
	- \frac{1}{\eta} \E{\norm{\momega_{t+1} - \momega^*}^2}
	+ c_1 \eta + c_2 \eta^2
	+ \varphi_t \, .
	\label{one step final}
\end{talign}
Here $\varphi_t$ is a constant bounded by $c_3 \norm{\momega_0 - \momega^*}$.
Because of $\varphi_t$, the model cannot converge to an arbitrary $\epsilon$; instead, it can only converge to the neighborhood of $\epsilon + \varphi$, where $\varphi$ is the weighted sum of the sequence $\varphi_t$. The following parts give the convergence rates when models converge to $\epsilon + \varphi$ for convenience.

\subsubsection{Weights of rounds.}
\label{sec:weights of rounds}
We can carefully tune the $p_{\tau, i}$ to minimize the noise while speeding up the convergence. Here we give the lemma of the best choice of $p_{\tau, i}$.

\begin{lemma}
	\label{best p in appendix}
	On round $t$, when setting $p_{\tau, i} = \frac{D^2}{t D^2 + (t-1) R^2}$ for any $\tau < t$, and $p_{t, i} = \frac{(t-1) R^2 + D^2}{t D^2 + (t-1) R^2}$, we can minimize $c_1$ and $c_2$ to get best convergence, where $c_1 = 2c_{p_G} + \frac{\sigma^2}{NK} + 2c_R$, and $c_2 = \frac{Lc_{p_G}}{3\eta_g^2} + \frac{L\sigma^2}{6\eta_g^2K} + \frac{Lc_R}{3\eta_g^2}$.
\end{lemma}

\begin{proof}
	Minimizing $c_1$ and $c_2$ is equal to minimize $c_{p_G} + c_{R}$, that is
	\begin{talign*}
		& \min_{p} \frac{1}{N} \sum_{i=1}^{N} (1 - p_{t, i})^2 R^2 + G^2 + D^2 (\sum_{\tau=1}^{t} p_{\tau, i}^{2}) \,,
		s.t. \; \sum_{\tau=1}^{t} p_{\tau, i} = 1, \forall i = 1, \dots, N \, .
	\end{talign*}
	Solving above optimization problem, we get the results.
\end{proof}

Using lemma \ref{best p in appendix}, we have $\min \{ c_{p_G} + c_{R} \} = G^2 + D^2 \left( \frac{D^2 + (t-1)R^2}{tD^2 + (t-1) R^2} \right)$.
Notice that the value of $c_1$, $c_2$, and $c_3$ are changing over $t$.
We use $c_{1}^{t}$, $c_{2}^{t}$, and $c_{3}^{t}$ to represent $c_1$, $c_2$, and $c_3$ on round $t$.

\subsubsection{Convergence results}

\label{sec:convex convergence}

\textit{Strongly convex functions.}
By Equation \eqref{one step final}, when $f_{t, i} (\momega)$ are strongly convex functions, using Lemma \ref{linear convergence rate}, and define $q_t = (1 - \frac{\mu \eta}{2})^{1-t}$, we have
\begin{talign*}
	f(\momega_F) - f(\momega^*) \le \cO \left(
	\mu \norm{\momega_0 - \momega^*}^2 \exp (- \frac{\mu \eta T}{2})
	\right)
	+ \frac{1}{\sum_{t=1}^{T} q_t} q_t \left(
	c_1 \eta + c_2 \eta^2 + \varphi_t
	\right) \, .
\end{talign*}

By Lemma \ref{time sum}, we have
\begin{talign*}
	\frac{1}{\sum_{t=1}^{T} q_t} \sum_{t=1}^{T} q_t ( c_R + c_{p_G}) & =
	\frac{1}{\sum_{t=1}^{T} q_t} \sum_{t=1}^{T} q_t \left(
	G^2 + D^2 \left( \frac{D^2 + (t-1)R^2}{tD^2 + (t-1) R^2} \right)
	\right) \nonumber
	\\
	& = \frac{1}{\sum_{t=1}^{T} q_t} \sum_{t=1}^{T} q_t \left(
	G^2 + D^2 \left(
		\frac{R^2}{R^2 + D^2} + \frac{D^4}{(R^2 + D^2)^2} \frac{1}{t - \frac{R^2}{R^2 + D^2}}
		\right)
	\right) \nonumber
	\\
	& = \cO \left(
	G^2 + D^2 \left(
		\frac{R^2}{R^2 + D^2} + \frac{D^4}{(R^2 + D^2)^2} \frac{\ln T}{T}
		\right)
	\right) \, .
\end{talign*}

For $\varphi_t$, by Lemma \ref{best p in appendix}, we have
\begin{talign*}
	\varphi = \frac{1}{\sum_{t=1}^{T} q_t} \sum_{t=1}^{T} q_t \varphi_t & \le \cO \left( R
	\left(
		\frac{D^2}{D^2 + R^2} - \frac{D^4}{(D^2 + R^2)^2} \frac{\ln T}{T}
		\right) \norm{\momega_0 - \momega^*}
	\right) \, .
\end{talign*}

Then define $c_{1}^{o} = \frac{1}{\sum_{t=1}^{T} q_t} \sum_{t=1}^{T} q_t c_1^t$,
$c_{2}^{o} = \frac{1}{\sum_{t=1}^{T} q_t} \sum_{t=1}^{T} q_t c_2^t$,
we have
\begin{talign*}
	f(\momega_F) - f(\momega^*) - \varphi & \le \cO \left(
	\mu \norm{\momega_0 - \momega^*}^2 \exp (- \frac{\mu \eta T}{2})
	+ c_1^o \eta + c_2^o \eta^2
	\right)
	\\
	& = \cO \left(
	\mu \norm{\momega_0 - \momega^*}^2 \exp (- \frac{\mu T}{L c_O})
	+ \frac{\sigma^2}{\mu N K T}
	+ \frac{c_{A}}{\mu T}
	+ \frac{c_{B}}{\mu T^2}
	+ \frac{c_C}{\mu^2 T^2}
	+ \frac{c_D}{\mu^2 T^3}
	\right)
	\, ,
\end{talign*}
where $c_O = 1 + A^2 + B^2$,
$c_{A} = G^2 + \frac{D^2 R^2}{R^2 + D^2}$,
$c_{B} = \frac{D^6}{(R^2 + D^2)^2}$,
$c_{C} = \frac{L \sigma^2}{\eta_g^2 K} + \frac{L c_A}{\eta_g^2}$,
$c_D = \frac{L c_B}{\eta_g^2}$.

\textit{General convex functions.}
When $f_{t, i} (\momega)$ are general convex functions ($\mu = 0$), directly use Lemma \ref{sub-linear convergence rate} we get
\begin{talign*}
	f(\momega_F) - f(\momega^*) - \varphi & \le \cO \left(
	\frac{c_O F}{T}
	+ \sqrt{ \frac{\sigma^2 F}{N K T} }
	+ \sqrt{ \frac{F c_A}{T} }
	+ \frac{\sqrt{c_B F}}{T}
	+ \sqrt[3]{ \frac{c_C F^2}{T^2} }
	+  \frac{\sqrt[3]{c_D F^2} }{T}
	\right) \, ,
\end{talign*}
where $c_O = 1 + A^2 + B^2$,
$c_{A} = G^2 + \frac{D^2 R^2}{R^2 + D^2}$,
$c_{B} = \frac{D^6}{(R^2 + D^2)^2}$,
$c_{C} = \frac{L \sigma^2}{\eta_g^2 K} + \frac{L c_A}{\eta_g^2}$,
$c_D = \frac{L c_B}{\eta_g^2}$,
and $F = \norm{\momega_0 - \momega^*}^2$.

\textit{The drift of optimal point.}
Because of information loss, we can't get the true optimal point when considering previous rounds' information.
Instead, the drift can be described as
\begin{talign*}
	\varphi = c_3^o \norm{\momega_0 - \momega^*} = \cO \left( R
	\left(
		\frac{D^2}{D^2 + R^2} - \frac{D^4}{(D^2 + R^2)^2} \frac{\ln T}{T}
		\right) \norm{\momega_0 - \momega^*}
	\right) \, .
\end{talign*}

\textit{Convergence of \fedavg.}
Notice that \fedavg is a special case of CFL by setting $p_{t, i} = 1$ on round $t$.
Having this, we derive the one round convergence of \fedavg under our formulation as
\begin{talign}
	f(\momega_t) - f(\momega^*) \le
	\frac{1}{\eta} (1 - \frac{\mu \eta}{2}) \E{\norm{\momega_t - \momega^*}^2}
	- \frac{1}{\eta} \E{\norm{\momega_{t+1} - \momega^*}^2}
	+ c_1 \eta + c_2 \eta^2  \,,
	\label{one step final}
\end{talign}
where $c_1 = 2 (G^2 + D^2) + \frac{\sigma^2}{NK}$,
$c_2 = \frac{L (G^2 + D^2)}{3 \eta_g^2} + \frac{L \sigma^2}{6 \eta_g^2 K}$.
Then using Lemma \ref{linear convergence rate}, we get the convergence rate when $f_{t, i}(\momega)$ are $\mu$ strongly convex
\begin{talign*}
	f(\momega_t) - f(\momega^*) \le \cO \left(
	\mu \norm{\momega_0 - \momega^*}^2 \exp (\frac{\mu T}{L c_O})
	+ \frac{\hat{c}_{A}}{\mu T}
	+ \frac{\hat{c}_C}{\mu^2 T^2}
	\right) \, ,
\end{talign*}
where $c_O = 1 + A^2 + B^2$,
$c_{A} = G^2 + D^2$,
$c_{C} = \frac{L \sigma^2}{\eta_g^2 K} + \frac{L c_A}{\eta_g^2}$.
Besides, when $f_{t, i} (\momega)$ are general convex, we have
\begin{small}
	\begin{talign*}
		f(\momega_F) - f(\momega^*) & \le \cO \left(
		\frac{c_O F}{T}
		+ \sqrt{ \frac{\sigma^2 F}{N K T} }
		+ \sqrt{ \frac{F c_A}{T} }
		+ \sqrt[3]{ \frac{c_C F^2}{T^2} }
		\right) \, ,
	\end{talign*}
\end{small}%
where $c_O = 1 + A^2 + B^2$,
$c_{A} = G^2 + D^2$,
$c_{C} = \frac{L \sigma^2}{\eta_g^2 K} + \frac{L c_A}{\eta_g^2}$,
and $F = \norm{\momega_0 - \momega^*}^2$.


\subsection{Proof of Convergence Rate for Non-convex Setting}
\label{sec:Non-convex}

In this section, we present the convergence results for the CFL and \fedavg algorithms under Assumptions~\ref{Smoothness and convexity assumption}--\ref{Bounded information loss assumption} when the local objective functions $f_{t,i}(\momega)$ are not convex. Firstly, based on the definition of an L-smooth function, we have that

\begin{talign}
	\Eb{f (\momega_{t+1})} \le \Eb{f (\momega_t)} + \underbrace{\Eb{\langle \nabla f(\momega_t), \Delta \momega_t \rangle}}_{A_1} +   \frac{L}{2} \underbrace{\Eb{\norm{ \Delta \momega_t }^2} }_{A_2} \, .
 \label{non convex first}
\end{talign}
For $A_1$ part of Inequality~\eqref{non convex first}, we have
\begin{talign}
	& \Eb{\langle \nabla f(\momega_t), \Delta \momega_t \rangle} \nonumber                                                                                                                                                                                                            \\
	& = \frac{- \eta}{N K} \sum_{i=1}^{N} \sum_{k=1}^{K} \Eb{\langle \nabla f(\momega_t), \nabla \bar{f}_{t, i}(\momega_{t, i, k})\rangle} \nonumber      \, ,                                                                                                                            \\
	& = \frac{- \eta}{N K} \sum_{i=1}^{N} \sum_{k=1}^{K} \Eb{\langle \nabla f(\momega_t), \sum_{\tau=1}^{t} p_{\tau, i} \nabla f_{\tau, i}(\momega_{t,i,k}) - \sum_{\tau=1}^{t-1} p_{\tau,i} \Delta_{\tau,i}(\momega_{t,i,k}) \rangle} \nonumber    \, ,                                  \\
	& \le  \frac{- \eta}{N K} \sum_{i=1}^{N} \sum_{k=1}^{K} \Eb{\langle \nabla f(\momega_t), \sum_{\tau=1}^{t} p_{\tau, i} \nabla f_{\tau, i}(\momega_{t,i,k}) \rangle} + \frac{\eta}{2N}\sum_{i=1}^{N}(1 - p_{t, i})\left( \Eb{\norm{\nabla f(\momega_t)}^2} + R^2 \right) \nonumber \, , \\
	& = \frac{\eta}{2NK} \sum_{i=1}^{N} \sum_{k=1}^{K}\left(\Eb{\norm{\nabla f(\momega_t) - \sum_{\tau=1}^{t} p_{\tau, i} \nabla f_{\tau, i}(\momega_{t,i,k})}^2} - \Eb{\norm{\sum_{\tau=1}^{t} p_{\tau,i} \nabla f_{\tau,i}(\momega_{t,i,k})}^2}\right) \nonumber      \, ,              \\
	& +\frac{\eta}{2N}\sum_{i=1}^{N}(1 - p_{t, i})\left( \Eb{\norm{\nabla f(\momega_t)}^2} + R^2 \right) - \frac{\eta}{2} \Eb{\norm{\nabla f(\momega_t)}^2} \, . \label{A1 nonconvex 1}
\end{talign}
From Inequality~\eqref{A2 middle}, we have
\begin{talign}
	\Eb{\norm{\Delta \momega_t}^2} \le \frac{2\eta^2}{N K} \sum_{i=1}^{N} \sum_{k=1}^{K} \left(
	\E{ \norm{
			\sum_{\tau=1}^{t} p_{\tau, i} \nabla f_{\tau, i} (\momega_{t, i, k})
		}^2 }
	+ \E{ \norm{
			\sum_{\tau=1}^{t-1} p_{\tau, i} \Delta_{\tau, i} (\momega_{t, i, k})
		}^2 }
	\right)
	+ \frac{\eta^2 \sigma^2}{NK} \, . \label{eq 27}
\end{talign}

Combine Inequalities~\eqref{A1 nonconvex 1} and ~\eqref{eq 27}, we can observe that when $\eta \le \frac{1}{2 L}$, the $\E{ \norm{\sum_{\tau=1}^{t} p_{\tau, i} \nabla f_{\tau, i} (\momega_{t, i, k})}^2 }$ term can be ignored, since the coefficient number will less than 0. This leaves us with the remaining term in $A_2$ as
$\frac{2\eta^2}{N} \sum_{i=1}^{N} (1 - p_{t,i})^2 R^2 + \frac{\eta^2 \sigma^2}{NK}$.
To address the remaining challenge posed by Inequality~\eqref{A1 nonconvex 1}, we consider the quantity $\norm{\nabla f(\momega_t) - \sum_{\tau=1}^{t} p_{\tau, i} \nabla f_{\tau, i}(\momega_{t,i,k})}^2$
\begin{talign}
	& \Eb{\norm{\nabla f(\momega_t) - \sum_{\tau=1}^{t} p_{\tau, i} \nabla f_{\tau, i}(\momega_{t,i,k})}^2} \nonumber                                                                                                               \\
	& = \Eb{\norm{\nabla f(\momega_t) - \nabla f(\momega_{t,i,k}) +  \nabla f(\momega_{t,i,k}) - \sum_{\tau=1}^{t} p_{\tau, i} \nabla f_{\tau, i}(\momega_{t,i,k})}^2} \nonumber    \, ,                                                \\
	& \le 2 \Eb{\norm{\nabla f(\momega_t) - \nabla f(\momega_{t,i,k})}^2} + 2 \Eb{\norm{\nabla f(\momega_{t,i,k}) - \sum_{\tau=1}^{t} p_{\tau, i} \nabla f_{\tau, i}(\momega_{t,i,k})}^2} \nonumber  \, ,                               \\
	& \le 2 L^2 \Eb{\norm{\momega_t - \momega_{t,i,k}}^2} +  \left( B^2 + A^2 \sum_{\tau=1}^{t} p_{\tau, i}^2 \right) \Eb{ \norm{\nabla f(\momega_{t,i,k})}^2} + \left( G^2 + D^2 \sum_{\tau=1}^{t} p_{\tau, i}^2 \right) \nonumber \, , \\
	& \le 2 L^2 \Eb{\norm{\momega_t - \momega_{t,i,k}}^2} +
	2 L^2 \left( B^2 + A^2 \sum_{\tau=1}^{t} p_{\tau, i}^2 \right) \Eb{ \norm{\momega_{t,i,k} - \momega_t}^2} \nonumber                                                                                                              \\
	& + 2 \left( B^2 + A^2 \sum_{\tau=1}^{t} p_{\tau, i}^2 \right) \Eb{\norm{\nabla f(\momega_t)}^2}
	+ \left( G^2 + D^2 \sum_{\tau=1}^{t} p_{\tau, i}^2 \right) \, . \label{A1 nonconvex 2}
\end{talign}
Combine Inequalities~\eqref{A1 nonconvex 1} and~\eqref{A1 nonconvex 2}, we have
\begin{talign}
	& \Eb{\langle \nabla f(\momega_t), \Delta \momega_t \rangle} \nonumber                                       \\
	& \le \frac{\eta L^2}{NK} \sum_{i=1}^{N} \sum_{k=1}^{K} c_{p_B, i} \Eb{\norm{\momega_{t,i,k} - \momega_t}^2}
	+ \eta (c_{p_B} - 1 - \frac{\sum_{i=1}^{N} p_{t, i}}{2N}) \Eb{ \norm{\nabla f(\momega_t)}^2} \nonumber        \\
	& + \frac{\eta}{2} \left( c_{p_G} + \frac{1}{N} \sum_{i=1}^{N}(1 - p_{t, i}) R^2 \right) \, . \label{eq 29}
\end{talign}
Combine Inequalities~\eqref{non convex first} and~\eqref{eq 29}, we have
\begin{talign}
	\Eb{f (\momega_{t+1})} & \le \Eb{f (\momega_t)} + \frac{\eta L^2 }{NK} \sum_{i=1}^{N} \sum_{k=1}^{K} c_{p_B, i} \Eb{\norm{\momega_{t,i,k} - \momega_t}^2} \nonumber \\
	& + \eta (c_{p_B} - 1 - \frac{\sum_{i=1}^{N} p_{t, i}}{2N}) \Eb{ \norm{\nabla f(\momega_t)}^2}
	+ \frac{\eta}{2} \left( c_{p_G} + \frac{1}{N} \sum_{i=1}^{N}(1 - p_{t, i}) R^2 \right)
	\nonumber                                                                                                                                                           \\
	& + \frac{L\eta^2}{N} \sum_{i=1}^{N} (1 - p_{t,i})^2 R^2 + \frac{L \eta^2 \sigma^2}{2NK} \, . \nonumber
\end{talign}
By Inequalities~\eqref{one step drift-2} and~\eqref{end of gradient drift}, we have
\begin{talign}
	& \E{\norm{\momega_{t,i,k} - \momega_t}^2} \nonumber                                                 \\
	& \le 4 k^2 \eta_l^2 L^2 c_{p_B, i} \E{\norm{\momega_{t,i,k-1} - \momega_t}^2}
	+ 4 k^2 \eta_l^2 c_{p_B, i} (\Eb{\norm{\nabla f(\momega_t)}^2}) + 2 k^2 \eta_l^2 c_{p_G, i} \nonumber \\
	& + 2 k^2 \eta_{l}^{2} (\sum_{\tau=1}^{t-1} p_{\tau, i})^2 R^2 + k \eta_l^2 \sigma^2 \, .
    \label{bouned local drift equation non convex}
\end{talign}
By iteratively applying Inequality~\eqref{bouned local drift equation non convex}, we obtain
\begin{talign*}
	& \E{\norm{\momega_{t,i,k} - \momega_t}^2} \nonumber                                                                                     \\
	& \le \sum_{r=0}^{k-1} \left( 4 k^2 \eta_l^2 c_{p_B, i} (\Eb{\norm{\nabla f(\momega_t)}^2}) + 2 k^2 \eta_l^2 c_{p_G, i}
	+ 2 k^2 \eta_{l}^{2} (\sum_{\tau=1}^{t-1} p_{\tau, i})^2 R^2 + k \eta_l^2 \sigma^2 \right) \left( 4 k^2 \eta_l^2 L^2 c_{p_B, i} \right)^r \, , \\
	& \le \frac{\left( 4 k^2 \eta_l^2 c_{p_B, i} (\Eb{\norm{\nabla f(\momega_t)}^2}) + 2 k^2 \eta_l^2 c_{p_G, i}
		+ 2 k^2 \eta_{l}^{2} (\sum_{\tau=1}^{t-1} p_{\tau, i})^2 R^2 + k \eta_l^2 \sigma^2 \right)}{1 - 4 k^2 \eta_l^2 L^2 c_{p_B, i}} \, .
\end{talign*}
Let $8 K^2 \eta_l^2 c_{p_B, i} \le 1$, we have
\begin{talign}
	& \frac{1}{NK} \sum_{i=1}^{N} \sum_{k=1}^{K} c_{p_B, i} \E{\norm{\momega_{t,i,k-1} - \momega_t}^2} \nonumber                                                                        \\
	& \le \frac{1}{N} \sum_{i=1}^{N} c_{p_B, i}\frac{\left( 4 K^2 \eta_l^2 c_{p_B, i} (\Eb{\norm{\nabla f(\momega_t)}^2}) + 2 K^2 \eta_l^2 c_{p_G, i}
		+ 2 K^2 \eta_{l}^{2} (\sum_{\tau=1}^{t-1} p_{\tau, i})^2 R^2 + K \eta_l^2 \sigma^2 \right)}{1 - 4 K^2 \eta_l^2 L^2 c_{p_B, i}} \nonumber       \, ,                                               \\
	& \le \frac{1}{N} \sum_{i=1}^{N} c_{p_B, i} \Eb{\norm{\nabla f(\momega_t)}^2} + \frac{1}{2} c_{p_G, i} + \frac{1}{2} (\sum_{\tau=1}^{t-1} p_{\tau, i})^2 R^2 + \frac{\sigma^2}{4 K} \, , \nonumber \\
	& = c_{p_B} \Eb{\norm{\nabla f(\momega_t)}^2} + \frac{1}{2} c_{p_G} + \frac{1}{2N} \sum_{i=1}^{N} (\sum_{\tau=1}^{t-1} p_{\tau, i})^2 R^2 + \frac{\sigma^2}{4 K} \label{bouned local drift equation non convex final} \, .
\end{talign}
Then combine Inequalities~\eqref{non convex first},~\eqref{eq 27},~\eqref{eq 29}, and~\eqref{bouned local drift equation non convex final}, we have
\begin{talign}
	\Eb{f (\momega_{t+1})} & \le \Eb{f (\momega_t)} + \eta \left( (L^2 + 1) c_{p_B} - 1 - \frac{\sum_{i=1}^{N} p_{t, i}}{2N} \right) \Eb{\norm{\nabla f(\momega_t)}^2} \nonumber \\
	& + \frac{\eta}{2} \left( c_{p_G} + \frac{1}{N} \sum_{i=1}^{N}(1 - p_{t, i}) R^2 \right)
	+ \frac{L\eta^2}{N} \sum_{i=1}^{N} (1 - p_{t,i})^2 R^2 + \frac{L \eta^2 \sigma^2}{2NK} \nonumber                                                                             \\
	& + \eta L^2 \left( \frac{1}{2} c_{p_G} + \frac{1}{2N} \sum_{i=1}^{N} (\sum_{\tau=1}^{t-1} p_{\tau, i})^2 R^2 + \frac{\sigma^2}{4 K} \right) \, .
\end{talign}

Then we must to choose the value of $p_{t, i}$ for better convergence. Follow the idea that we want to train a better model when convergent, we choose $p_{t, i}$ that can minimize the constant term, which is by solving
\begin{talign*}
	\min_{p} \; & \;  G^2 + D^2 \sum_{\tau=1}^{t} p_{\tau, i}^2 + (\sum_{\tau=1}^{t-1} p_{\tau, 1})^2 R^2 \, ,
	s.t. \;  \; \sum_{\tau=1}^{t} p_{\tau, 1} = 1 \, .
\end{talign*}

We get same results as when the objective function is convex, that is, $p_{\tau, i} = \frac{D^2}{t D^2 + (t-1) R^2}$ for any $\tau < t$, and $p_{t, i} = \frac{(t-1) R^2 + D^2}{t D^2 + (t-1) R^2}$, and $\min \{ G^2 + D^2 \sum_{\tau=1}^{t} p_{\tau, i}^2 + (\sum_{\tau=1}^{t-1} p_{\tau, 1})^2 R^2 \} = G^2 + D^2 \left( \frac{D^2 + (t-1)R^2}{tD^2 + (t-1) R^2} \right)$. Then we simplify the equation by using $c_0(t)$, $c_1(t)$, and $c_2(t)$ to denote the constant terms, we have
\begin{talign*}
	\Eb{f (\momega_{t+1})} & \le \Eb{f (\momega_t)} - \eta c_0(t) \Eb{\norm{\nabla f(\momega_t)}^2} + \eta c_1(t) + \eta^2 c_2(t)
\end{talign*}
Moving $\Eb{\norm{\nabla f(\momega_t)}^2}$ term to the left, we have
\begin{talign*}
	\frac{1}{T} \sum_{t=1}^{T} c_0(t) \Eb{\norm{\nabla f(\momega_t)}^2} \le \frac{\Eb{f (\momega_{0})} - \Eb{f (\momega^*)}}{\eta } + \eta \frac{1}{T} \sum_{t=1}^{T} c_2(t) + \varphi
\end{talign*}

Consider $c_0$, because the value of $B^2$ and $A$ can't be constrained, the algorithm can't converge for very large $A$ and $B$. Here we first consider when $c_0 < 0$, and use $c_{m}$ to denote the $\min c_0$. Then let $\eta = \frac{\sqrt{KN}}{\sqrt{T} L}$, we derive the convergence rate as
\begin{talign*}
	\frac{1}{T} \sum_{t=1}^{T} \Eb{\norm{\nabla f(\momega_t)}^2} & = O(\frac{L(f_0 - f_*)}{\sqrt{TNK} c_m}
	+ \frac{1}{\sqrt{T}} \left( \frac{\sqrt{KN}}{L}\left( \frac{R^2}{R^2 + D^2} \right)^2+ \frac{\sigma^2}{\sqrt{N K}}\right) \nonumber \\
	& + \frac{\sqrt{KN}}{TL} \frac{D^4}{c_m (R^2 + D^2)^2}
	+ \frac{\sqrt{KN}}{T \sqrt{T}L} \left( \frac{D^4}{(R^2 + D^2)^2} \right)^2 )
\end{talign*}

Then we want to analyse when the algorithm won't converge. When the algorithm won't converge, that means $c_0 \ge 0$. Because $c_0 = (L^2 + 1) c_{p_B} - 1 - \frac{\sum_{i=1}^{N} p_{t, i}}{2N} = \frac{1}{N} \sum_{i=1}^{N} (L^2 + 1) (1 + A^2 + B^2 \sum_{\tau=1}^{t} p_{\tau, i}^2) - 1 - \frac{p_{t, i}}{2}$. Then the value of $A^2, B^2, L, p$ decide if $c_0 \ge 0$.

Notice that in FedAvg, $p_{t, i} = 1$, and $c_{0, avg} = (L^2 + 1) (1 + A^2 + B^2) - \frac{3}{2}$, and in CFL, we can setting different $p$ to get better convergence. For example, when $B$ is large, setting $p_{\tau, i} = \frac{1}{t}$, we have $c_{0, } = (L^2 + 1) (1 + A^2 + \frac{B^2}{t}) - 1 - \frac{1}{2t}$. Thus, we draw the conclusion as

Then we have following observations:
\begin{itemize}
	\item \emph{CFL converge faster than FedAvg by reducing the variance terms.} For $c_0 = \frac{1}{N} \sum_{i=1}^{N} (L^2 + 1) (1 + A^2 + B^2 \sum_{\tau=1}^{t} p_{\tau, i}^2) - 1 - \frac{p_{t, i}}{2}$, in FedAvg, we have $p_{t, i} = 1$, then $c_{0, avg} = (L^2 + 1) (1 + A^2 + B^2) - \frac{3}{2}$. However, in CFL, we can adjust $p$, such as when setting $p_{\tau, i} = \frac{1}{t}$, we have $c_{0} = (L^2 + 1) (1 + A^2 + \frac{B^2}{t}) - 1 - \frac{1}{2t}$, then the variance term reduce from $B^2$ to $\frac{B^2}{t}$.
	\item \emph{The convergence rate of CFL become better for larger $t$, since $c_{p_B}$ in $c_0$ become smaller for larger $t$.}
	\item Improve $N, K$ will speed up the convergence by reducing the terms about $f_0 - f_*$ and $\sigma^2$. However, larger $N$ and $K$ will cause larger information loss and round drift (terms with $R$ and $D$), thus, it should be a trade off in practice.
\end{itemize}

\subsection{Theoretical Results}
\label{sec:Theoretical Results in appendix}

We give the full version of Convergence results as follows.
\begin{theorem}[Convergence rate of CFL methods]
	\label{Convergence rate of CFL full}
	Assume $\{ f_{t, i} (\momega) \}$ satisfy Assumption \ref{Smoothness and convexity assumption}--\ref{Bounded information loss assumption}, the output of Algorithm \ref{CFL Algorithm Framework} has expected error smaller than $\epsilon + \varphi$, for $\eta_g = 1$, $\eta_l \le \frac{\sqrt{3 + 4 (1 + B^2 + A^2)} - \sqrt{4 (1 + B^2 + A^2)}}{6KL\sqrt{1 + B^2 + A^2} }$,
	$p_{\tau, i} = \frac{D^2}{t D^2 + (t - 1) R^2}$ ($\tau < t$), and $p_{t, i} = \frac{(t-1)R^2 + D^2}{tD^2 + (t-1)R^2}$ on round $t$.

	When $\{ f_{t, i} (\momega) \}$ are $\mu$-strongly convex functions, we have
	\begin{small}
		\begin{talign*}
			T = \cO \left(
			\frac{L c_O}{\mu} + \frac{\sigma^2}{\mu N K \epsilon}
			+ \frac{c_A}{\mu \epsilon} + \sqrt{ \frac{c_B}{\mu \epsilon} }
			+ \sqrt{ \frac{c_C}{\mu^2 \epsilon} }
			+ \sqrt[3]{ \frac{c_D}{\mu^2 \epsilon} }
			\right) \, ,
		\end{talign*}
	\end{small}
	and when $\{ f_{t, i} (\momega) \}$ are general convex functions ($\mu = 0$), we have
	\begin{small}
		\begin{talign*}
			T = \cO \left(
			\frac{c_O \mathcal{H} + \sqrt{c_B \mathcal{H}} + \sqrt[3]{c_D \mathcal{H}^2}}{\epsilon}
			+ \frac{\sigma^2 \mathcal{H}}{N K \epsilon^2} + \frac{c_A \mathcal{H}}{\epsilon^2}
			+ \sqrt{ \frac{c_C \mathcal{H}^2}{\epsilon^3} }
			\right) \,,
		\end{talign*}
	\end{small}%
	and when $\{ f_{t, i} (\momega) \}$ are non-convex, setting $\eta = K\eta_g \eta_l = \frac{\sqrt{KN}}{\sqrt{T} L}$, when $\frac{1}{T} \sum_{t=1}^{T} \Eb{\norm{\nabla f(\momega_t)}^2}$ reach $\epsilon$ we have
	\begin{small}
		\begin{talign*}
			& T =
			\cO \left( \frac{L^2(f_0 - f_*)^2}{NK c_m^2 \epsilon^2}
			+ \frac{1}{\epsilon^2} ( \frac{\sqrt{KN} c_{R1}^2}{L}+ \frac{\sigma^2}{\sqrt{N K}} )^2 + \frac{\sqrt{KN} c_{R2}}{c_m \epsilon L} \right) \, ,
		\end{talign*}
	\end{small}%
	where $c_O = 1 + A^2 + B^2$,
	$c_{A} = G^2 + \frac{D^2 R^2}{R^2 + D^2}$,
	$c_{B} = \frac{D^6}{(R^2 + D^2)^2}$,
	$c_{C} = \frac{L \sigma^2}{\eta_g^2 K} + \frac{L c_A}{\eta_g^2}$,
	$c_D = \frac{L c_B}{\eta_g^2}$,
	$c_{R1} = \frac{R^2}{R^2 + D^2}$,
	$c_{R2} = \frac{D^4}{(R^2 + D^2)^2}$,
	$c_m$ is a constant related to $A$, $B$ and $R$,
	and $\mathcal{H} = \norm{\momega_0 - \momega^*}^2$.
	$N$ is the number of chosen clients in each round, and $K$ is the number of local iterations.
\end{theorem}

\begin{theorem}[Convergence rate of \fedavg under time-varying scenarios]
	\label{Convergence rate of FedAvg full}
	Assume $\{ f_{t, i} (\momega) \}$ satisfy Assumption \ref{Smoothness and convexity assumption}--\ref{Bounded information loss assumption}, the output of \fedavg has expected error smaller than $\epsilon$, for $\eta_g = 1$ and $\eta_l \le \frac{\sqrt{3 + 4 (1 + B^2 + A^2)} - \sqrt{4 (1 + B^2 + A^2)}}{6KL\sqrt{1 + B^2 + A^2} }$.
	When $\{ f_{t, i} (\momega) \}$ are $\mu$-strongly convex functions, we have
	\begin{small}
		\begin{talign*}
			T = \cO \left(
			\frac{L c_O}{\mu} + \frac{\sigma^2}{\mu N K \epsilon}
			+ \frac{c_A}{\mu \epsilon}
			+ \sqrt{ \frac{c_C}{\mu^2 \epsilon} }
			\right) \,,
		\end{talign*}
	\end{small}%
	and when $\{ f_{t, i} (\momega) \}$ are general convex functions ($\mu = 0$), we have
	\begin{small}
		\begin{talign*}
			T = \cO \left(
			\frac{c_O \mathcal{H} }{\epsilon}
			+ \frac{\sigma^2 \mathcal{H}}{N K \epsilon^2} + \frac{c_A \mathcal{H}}{\epsilon^2}
			+ \sqrt{ \frac{c_C \mathcal{H}^2}{\epsilon^3} }
			\right) \,,
		\end{talign*}
	\end{small}%
	and when $\{ f_{t, i} (\momega) \}$ are non-convex, setting $\eta = K\eta_g \eta_l = \frac{\sqrt{KN}}{\sqrt{T} L}$, when $\frac{1}{T} \sum_{t=1}^{T} \Eb{\norm{\nabla f(\momega_t)}^2}$ reach $\epsilon$ we have
	\begin{small}
		\begin{talign*}
			& T =
			\cO \left( \frac{L^2(f_0 - f_*)^2}{NK c_m^2 \epsilon^2}
			+ \frac{1}{\epsilon^2} ( \frac{\sqrt{KN}}{L}+ \frac{\sigma^2}{\sqrt{N K}} )^2 \right) \, ,
		\end{talign*}
	\end{small}%
	where $c_O = 1 + A^2 + B^2$,
	$c_{A} = G^2 + D^2$,
	$c_{C} = \frac{L \sigma^2}{\eta_g^2 K} + \frac{L c_A}{\eta_g^2}$,
	and $\mathcal{H} = \norm{\momega_0 - \momega^*}^2$.
	$N$ is the number of chosen clients in each round, and $K$ is the number of local iterations.
\end{theorem}

\subsection{Convergence Rate of CFL under Correlated Time Drifts}
\label{sec:Convergence rate of CFL under correlated time drifts}

In this section, we study the scenario where the time drifts $\mxi_{t_1, i}$ and $\mxi_{t_2, i}$ are correlated. We first provide some useful lemmas.

\begin{lemma}[Bounded Gradient Noise of CFL under correlated time drifts]
	\label{Bounded Gradient Noise of CFL under correlated time drifts}
	For Formulation \eqref{CFL Formulation},
	consider \algopt and FedAvg, when time drift $\mxi_{t, i}$ satisfy Assumption~\ref{Time drift with correlation assumption}, we can bound the gradient drift as
	\begin{talign*}
		\E{ \norm{ \sum_{\tau=1}^{t} p_{\tau, i} \nabla f_{\tau, i}(\momega) }^2 } & \le \left( 1 + B^2 + A^2 \sum_{\tau=1}^{t} p_{\tau, i}^{2} + C^2 \sum_{\tau_1 \not = \tau_2}^{t} p_{\tau_1, i} p_{\tau_2, i} \right) \E{ \norm{\nabla f(\momega)}^2 }
		\nonumber                                                                                                                                                                                                                                          \\                          & \quad + G^2 + D^2 \sum_{\tau = 1}^{t} p_{\tau_i}^{2} + min(F^2, D^2) \sum_{\tau_1 \not = \tau_2}^{t} p_{\tau_1, i} p_{\tau_2, i} \, .
	\end{talign*}
\end{lemma}

\begin{proof}
	Using Assumption \ref{Bounded gradient noise of CFL assumption},
	together with the fact that $\nabla f_{t, i}(\momega) = \nabla f(\momega) + \mdelta_{t, i} + \mxi_{t, i}$ (Here to simplify the notation, we use $\{t, i\}$ pair to denote the client participate in training on round $t$ and number $i$), we have
	\begin{small}
		\begin{talign*}
			\E{ \norm{ \sum_{\tau=1}^{t} p_{\tau, i} \nabla f_{\tau, i}(\momega) }^2 } & = \E{ \norm{ \sum_{\tau=1}^{t} p_{\tau, i} (\nabla f_{\tau, i}(\momega) + \mdelta_i + \mxi_{\tau, i}) }^2 } \, ,                                                           \\
			& = \E{ \norm{ \nabla f(\momega) }^2 } + \E{ \norm{ \mdelta_i }^2 } + \E{ \norm{ \sum_{\tau=1}^{t} p_{\tau, i} \mxi_{\tau, i} }^2 }     \, ,                                \\
			& \le \left( 1 + B^2 + A^2 \sum_{\tau=1}^{t} p_{\tau, i}^{2} + C^2 \sum_{\tau_1 \not = \tau_2}^{t} p_{\tau_1, i} p_{\tau_2, i} \right) \E{ \norm{\nabla f(\momega)}^2 }
			\nonumber                                                                                                                                                                                                                                          \\                          & + G^2 + D^2 \sum_{\tau = 1}^{t} p_{\tau_i}^{2} + min(F^2, D^2) \sum_{\tau_1 \not = \tau_2}^{t} p_{\tau_1, i} p_{\tau_2, i} \, .
		\end{talign*}
	\end{small}
\end{proof}

Based on the results of Inequalities~\eqref{A2-final} and~\eqref{one step drift final}, we define $c_{p_B, i} = 1 + B^2 + A^2 \sum_{\tau=1}^{t} p_{\tau, i}^{2} + C^2 \sum_{\tau_1 \not = \tau_2}^{t} p_{\tau_1, i} p_{\tau_2, i}$, and $c_{p_G, i} = G^2 + D^2 \sum_{\tau = 1}^{t} p_{\tau_i}^{2} + F^2 \sum_{\tau_1 \not = \tau_2}^{t} p_{\tau_1, i} p_{\tau_2, i}$. This leads to the same results as Inequality~\eqref{one step final}. The difference between using Lemma~\ref{Bounded gradient noise of CFL assumption} and Lemma~\ref{Bounded Gradient Noise of CFL under correlated time drifts} is only in the definitions of $c_{p_B, i}$ and $c_{p_G, i}$. As in Section~\ref{sec:weights of rounds}, the optimal weights can be found by solving the following optimization problem
\begin{small}
	\begin{talign}
		& \min_{p} \frac{1}{N} \sum_{i=1}^{N} (1 - p_{t, i})^2 R^2 + G^2 + D^2 (\sum_{\tau=1}^{t} p_{\tau, i}^{2}) + min(F^2, D^2) \sum_{\tau_1 \not = \tau_2}^{t} p_{\tau_1} p_{\tau_2} \, ,\nonumber \\
		& s.t. \; \sum_{\tau=1}^{t} p_{\tau, i} = 1, \forall i = 1, \dots, N \, .
		\label{optimal weight of related time drift}
	\end{talign}
\end{small}%
Then we can get,
\begin{talign*}
	p_{\tau, i} = \frac{max(D^2 - F^2, 0)}{t\;max(D^2 - F^2, 0) + (t-1) R^2}, \forall \tau < t \,,
	p_{t, i}    = \frac{(t-1) R^2 + max(D^2 - F^2, 0)}{t\;max(D^2 - F^2, 0) + (t-1) R^2} \,.
\end{talign*}
Compare with the results in Section~\ref{sec:weights of rounds}, we can found that the weights when using Lemma~\ref{Bounded Gradient Noise of CFL under correlated time drifts} is the same as the weights of using Lemma~\ref{Bounded gradient noise of CFL assumption}, and set $D^2$ to $max(D^2 - F^2, 0)$. Then substitute the $p_{\tau, i}$ and $p_{t, i}$ in Section~\ref{sec:convex convergence} and Section~\ref{sec:Non-convex}, we derive the convergence rate under Assumption~\ref{Time drift with correlation assumption}.
\begin{theorem}[Convergence rate of CFL methods under correlated time drifts] \label{Convergence rate of CFL with correlated time drifts appendix}
	Assume $\{ f_{t, i} (\momega) \}$ satisfy Assumption \ref{Smoothness and convexity assumption}--\ref{Time drift with correlation assumption}, the output of Algorithm \ref{CFL Algorithm Framework} has expected error smaller than $\epsilon + \varphi$, for $\eta_g = 1$, $\eta_l \le \frac{\sqrt{3 + 4 c_O} - \sqrt{4 c_O}}{6KL\sqrt{c_O} }$,
	$p_{\tau, i} = \frac{M^2}{t (M^2) + (t - 1) R^2}$ (for $\tau < t$), and $p_{t, i} = \frac{(t-1)R^2 + M^2}{t(M^2) + (t-1)R^2}$ on round $t$.

	When $\{ f_{t, i} (\momega) \}$ are $\mu$-strongly convex functions, we have,
	\begin{small}
		\begin{talign*}
			T = \cO \left(
			\frac{L c_O}{\mu} + \frac{\sigma^2}{\mu N K \epsilon}
			+ \frac{1}{\mu \epsilon} \left( G^2 + \frac{M^2 R^2}{R^2 + M^2} \right)
			\right) \, ,
		\end{talign*}
	\end{small}
	and when $\{ f_{t, i} (\momega) \}$ are general convex functions ($\mu = 0$), we have
	\begin{small}
		\begin{talign*}
			T = \cO \left(
			\frac{c_O \mathcal{H} }{\epsilon}
			+ \frac{\sigma^2 \mathcal{H} }{N K \epsilon^2} + \frac{\mathcal{H} }{\epsilon^2} \left( G^2 + \frac{M^2 R^2}{R^2 + M^2} \right)
			\right) \,,
		\end{talign*}
	\end{small}%
	and when $\{ f_{t, i} (\momega) \}$ are non-convex, by setting $\eta = K\eta_g \eta_l = \frac{\sqrt{KN}}{\sqrt{T} L}$, $\frac{1}{T} \sum_{t=1}^{T} \Eb{\norm{\nabla f(\momega_t)}^2}$ needs to take the following steps to reach $\epsilon$
	\begin{small}
		\begin{talign*}
			& T =
			\cO \left( \frac{L^2(f_0 - f_*)^2}{NK c_m^2 \epsilon^2}
			+ \frac{1}{\epsilon^2} \left( \frac{\sqrt{KN} c_{R1}^2}{L}+ \frac{\sigma^2}{\sqrt{N K}} \right)^2 \right) \, ,
		\end{talign*}
	\end{small}%
	where $M^2 = max(0, D^2 - F^2)$
	$c_O = 1 + A^2 + max(B^2, C^2)$,
	$c_{R1} = \frac{R^2}{R^2 + M^2}$,
	$c_m$ is a constant related to $A$, $B$, $C$ and $R$,
	and $\mathcal{H} = \norm{\momega_0 - \momega^*}^2$.
	$N$ is the number of chosen clients in each round, and $K$ is the number of local iterations.
\end{theorem}

\subsection{Proof of Theorem~\ref{Time drift with correlation assumption: a special case}}
\label{sec:proof of optimal weights of special case}

\begin{lemma} Given
	\begin{talign}
		Q = \begin{bmatrix}
			D^2 + R^2              & \alpha D^2 + R^2       & \cdots & \alpha^{t-2} D^2 + R^2 & \alpha^{t-1} D^2 \\
			\alpha D^2 + R^2       & D^2 + R^2              & \cdots & \alpha^{t-3} D^2 + R^2 & \alpha^{t-2} D^2 \\
			\vdots                 & \vdots                 &        & \vdots                 & \vdots           \\
			\alpha^{t-2} D^2 + R^2 & \alpha^{t-3} D^2 + R^2 & \cdots & \alpha^2 D^2 + R^2     & \alpha D^2       \\
			\alpha^{t-1} D^2       & \alpha^{t-2} D^2       & \cdots & \alpha D^2             & D^2
		\end{bmatrix} \, ,
	\end{talign}
	we have $Q$ is positive definite.
	\label{pd Q}
\end{lemma}

\begin{proof}
	We can decompose $Q$ into two matrix, $M_{R}$ and $M_{F}$, and $Q = M_R + M_F$.
	Firstly, we would like to show that $M_R$ is positive and semi-definite, and $M_R$ is defined by,
	\begin{small}
		\begin{talign*}
			M_{R} = \begin{bmatrix}
				R^2    & R^2    & \cdots & R^2    & 0      \\
				R^2    & R^2    & \cdots & R^2    & 0      \\
				\vdots & \vdots &        & \vdots & \vdots \\
				R^2    & R^2    & \cdots & R^2    & 0      \\
				0      & 0      & \cdots & 0      & 0
			\end{bmatrix} \, .
		\end{talign*}
	\end{small}%
	First, let us consider the matrix $M_R$. It has a rank of 1 and $t-1$ zero eigenvalues, with a single non-zero eigenvalue of $(t-1) R^2$. This means that $M_R$ is positive and semi-definite. Next, we will prove that $M_F$ is positive definite. This matrix is defined as follows
	\begin{small}
		\begin{talign*}
			M_{F} = \begin{bmatrix}
				D^2              & \alpha D^2       & \cdots & \alpha^{t-2} D^2 & \alpha^{t-1} D^2 \\
				\alpha D^2       & D^2              & \cdots & \alpha^{t-3} D^2 & \alpha^{t-2} D^2 \\
				\vdots           & \vdots           &        & \vdots           & \vdots           \\
				\alpha^{t-2} D^2 & \alpha^{t-3} D^2 & \cdots & \alpha^2 D^2     & \alpha D^2       \\
				\alpha^{t-1} D^2 & \alpha^{t-2} D^2 & \cdots & \alpha D^2       & D^2
			\end{bmatrix} \, .
		\end{talign*}
	\end{small}
By row operations, we can transform $M_F$ into
	\begin{small}
		\begin{talign*}
			\begin{bmatrix}
				D^2    & \alpha D^2         & \cdots & \alpha^{t-2} D^2                    & \alpha^{t-1} D^2              \\
				0      & D^2 (1 - \alpha^2) & \cdots & D^2 (\alpha^{t-3} - \alpha^{t - 1}) & D^2 (\alpha^{t-2} - \alpha^t) \\
				\vdots & \vdots             &        & \vdots                              & \vdots                        \\
				0      & 0                  & \cdots & D^2 (1 - \alpha^2)                  & D^2 (\alpha - \alpha^3)       \\
				0      & 0                  & \cdots & 0                                   & D^2 (1 - \alpha^2)
			\end{bmatrix} \, .
		\end{talign*}
	\end{small}
Because $0 \le \alpha < 1$, we have $M_F$ is positive definite. Then $Q = M_R + M_F$ is also positive definite.

\end{proof}

\begin{lemma}[Bounded Gradient Noise of CFL under correlated time drifts]
	\label{Bounded Gradient Noise of CFL under correlated time drifts general}
	For Formulation \eqref{CFL Formulation},
	consider \algopt and FedAvg, when time drift $\xi_{t, i}$ satisfy Assumption~\ref{Generalized time drift with correlation assumption}, we can bound the gradient drift as\\
	\begin{small}
		\begin{talign*}
			\E{ \norm{ \sum_{\tau=1}^{t} p_{\tau, i} \nabla f_{\tau, i}(\momega) }^2 } & \le \left( 1 + B^2 + A^2 \sum_{\tau=1}^{t} p_{\tau, i}^{2} + \sum_{\tau_1 \not = \tau_2}^{t} C_{\tau_1, \tau_2}^2 p_{\tau_1, i} p_{\tau_2, i} \right) \E{ \norm{\nabla f(\momega)}^2 }
			\nonumber                                                                                                                                                                                                                                                           \\                          & + G^2 + D^2 \sum_{\tau = 1}^{t} p_{\tau_i}^{2} + \sum_{\tau_1 \not = \tau_2}^{t} min(F_{\tau_1, \tau_2}^2, D^2) p_{\tau_1, i} p_{\tau_2, i} \, .
		\end{talign*}
	\end{small}
\end{lemma}

\begin{proof}
	Using Assumption \ref{Generalized time drift with correlation assumption},
	together with the fact that $\nabla f_{t, i}(\momega) = \nabla f(\momega) + \mdelta_{t, i} + \mxi_{t, i}$ (Here to simplify the notation, we use $\{t, i\}$ pair to denote the client participate in training on round $t$ and number $i$), we have
	\begin{small}
		\begin{talign*}
			\E{ \norm{ \sum_{\tau=1}^{t} p_{\tau, i} \nabla f_{\tau, i}(\momega) }^2 } & = \E{ \norm{ \sum_{\tau=1}^{t} p_{\tau, i} (\nabla f_{\tau, i}(\momega) + \mdelta_i + \mxi_{\tau, i}) }^2 } \, ,                                                                           \\
			& = \E{ \norm{ \nabla f(\momega) }^2 } + \E{ \norm{ \mdelta_i }^2 } + \E{ \norm{ \sum_{\tau=1}^{t} p_{\tau, i} \mxi_{\tau, i} }^2 }                 \, ,                                     \\
			& \le \left( 1 + B^2 + A^2 \sum_{\tau=1}^{t} p_{\tau, i}^{2} + \sum_{\tau_1 \not = \tau_2}^{t} C_{\tau_1, \tau_2}^2 p_{\tau_1, i} p_{\tau_2, i} \right) \E{ \norm{\nabla f(\momega)}^2 } 
			\nonumber                                                                                                                                                                                                                                                           \\                          & + G^2 + D^2 \sum_{\tau = 1}^{t} p_{\tau_i}^{2} + \sum_{\tau_1 \not = \tau_2}^{t} min(F_{\tau_1, \tau_2}^2, D^2) p_{\tau_1, i} p_{\tau_2, i} \, .
		\end{talign*}
	\end{small}
\end{proof}

Based on Lemma~\ref{Bounded Gradient Noise of CFL under correlated time drifts general}, using the same process in Section~\ref{sec:Convergence rate of CFL under correlated time drifts}, we can prove that the optimal $p_{\tau, i}$ is given by solving the following optimization problem.
\begin{talign*}
	& \min_{p} \frac{1}{N} \sum_{i=1}^{N} (1 - p_{t, i})^2 R^2 + G^2 + D^2 (\sum_{\tau=1}^{t} p_{\tau, i}^{2}) + \sum_{\tau_1 \not = \tau_2}^{t} min(F_{\tau_1, \tau_2}^2, D^2) p_{\tau_1} p_{\tau_2} \, , \\
	& s.t. \; \sum_{\tau=1}^{t} p_{\tau, i} = 1, \forall i = 1, \dots, N \, .
\end{talign*}
When under Assumption~\ref{Time drift with correlation assumption: a special case}, we can change this optimization problem to,
\begin{talign*}
	& \min_{p} \frac{1}{N} \sum_{i=1}^{N} (1 - p_{t, i})^2 R^2 + G^2 + D^2 \sum_{\tau_1, \tau_2}^{t} \alpha^{\| \tau_1 - \tau_2 \|} p_{\tau_1} p_{\tau_2} \, ,\\
	& s.t. \; \sum_{\tau=1}^{t} p_{\tau, i} = 1, \forall i = 1, \dots, N \, .
\end{talign*}
This optimization problem can be further transformed to,
\begin{talign*}
	& \min_{p} \frac{1}{N} \sum_{i=1}^{N} \sum_{\tau_1, \tau_2}^{t} \alpha^{| \tau_1 - \tau_2 |} D^2 p_{\tau_1} p_{\tau_2} + \sum_{\tau_1, \tau_2}^{t - 1} R^2 p_{\tau_1} p_{\tau_2} \, , \\
	& s.t. \; \sum_{\tau=1}^{t} p_{\tau, i} = 1, \forall i = 1, \dots, N \, .
\end{talign*}
Then we can construct a matrix $Q \in \mathbb{R}^{t \times t}$, where $Q_{ij} = \alpha^{|i - j| D^2 + R^2}$ for $i, j < t$, and $Q_{ij} = \alpha^{|i - j| D^2}$ for others, then the optimization problem become,
\begin{talign*}
	& \min_{\hat{p}} \hat{p}^{T} Q \hat{p} \, ,
	s.t. \; \sum_{\tau=1}^{t} p_{\tau, i} = 1, \forall i = 1, \dots, N \, ,
\end{talign*}
where $\hat{p} \in \mathbb{R}^{t} = [p_1, \cdots, p_t]^{T}$. Based on Lemma~\ref{pd Q}, we have $Q$ is positive definite, and based on the quadratic programming theory, we can transform the optimization problem to the following linear system,
\begin{talign*}
	\begin{bmatrix}
		Q              & \mathbf{1} \\
		\mathbf{1}^{T} & 0
	\end{bmatrix}
	\begin{bmatrix}
		\hat{p} \\
		\lambda
	\end{bmatrix}
	=
	\begin{bmatrix}
		\mathbf{0} \\
		1
	\end{bmatrix} \, ,
\end{talign*}
where $\lambda$ is a set of Lagrange multipliers which come out of the solution alongside $\hat{p}$.

\section{Experiment details}

\subsection{Realistic Data Sets}

\label{sec:Realistic data sets}

\subsubsection{Setup}

\label{sec:realistic setup appendix}

We consider federated learning an image classifier on split-CIFAR10 and split-CIFAR100 data sets with ResNet18, and split-Fashion-MNIST data set with a two-layer MLP.
The ``split'' follows the idea introduced in previous works~\citep{yurochkin2019bayesian,hsu2019measuring,reddi2021adaptive}, where we leverage the Latent Dirichlet Allocation (LDA) to control the distribution drift with parameter $\alpha$ (See Algorithm \ref{Data spliting algorithm}).
Larger $\alpha$ denotes smaller drifts here.\\
In our experiments, unless specifically mentioned otherwise all data sets are partitioned to 210 subsets for 7 different clients: all clients are selected and trained for 500 communication rounds, and each client samples one of the corresponding 30 subsets randomly for the local training.
Note that unless mentioned otherwise the training strategy here applies to all FL baselines and CFL methods, and we evaluate the performance of models on global test data sets.
We carefully tuned the hyper-parameters in all algorithms, and we report the results under the optimal settings after many trials.
For CFL-Regularization, we set the weight of regularization term for $\beta = 1$ on fc layer, and $\beta = 0.1$ for last block (layer). For FedProx, we set the weight of proximal term $\mu = 0.1$, and for MimeLite, we set the global momentum weight for $\gamma = 0.01$. Besides, for all experiments, we set fixed learning rate $lr = 0.01$.
We also naively examined warm-up tuning strategies but can not observe significant improvement on final results.

\textit{Construct local data sets with overlap.}
\begin{figure}[!t]
	\centering
	\includegraphics[width=.5\textwidth,]{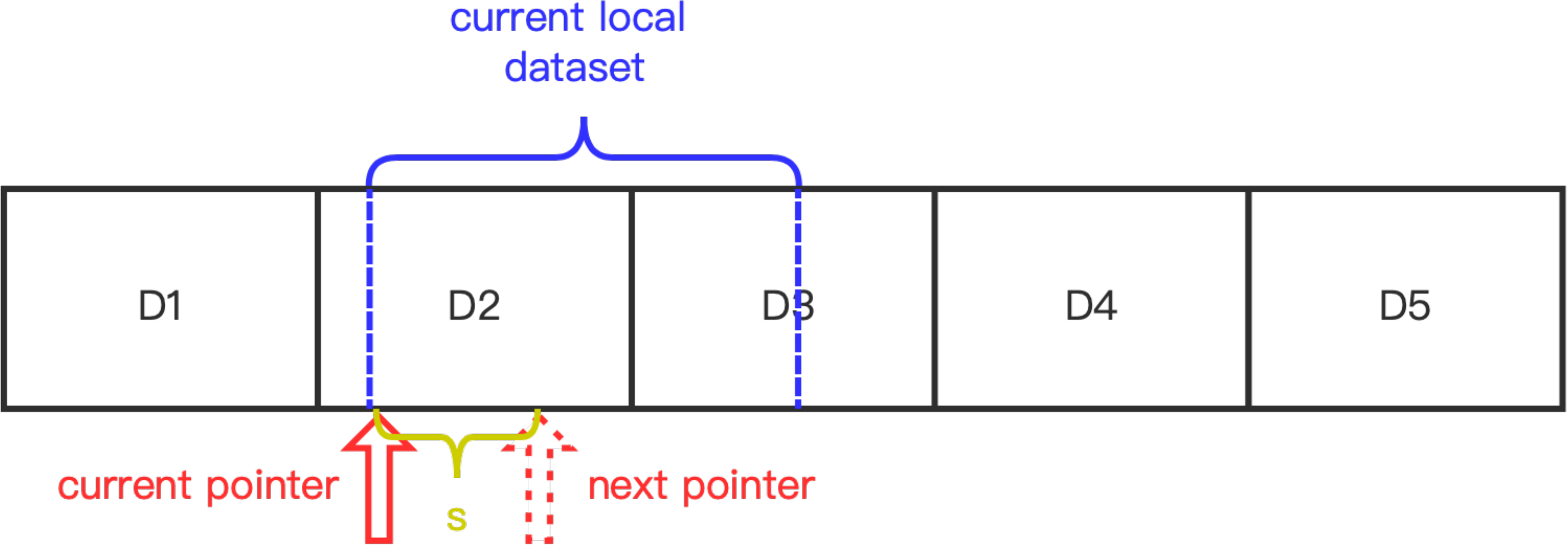}
	\caption{Process of constructing local data sets with overlap. Blue lines bound the current local data sets. Red pointer points to the start of current local data set.}
	\label{Process of constructing local data sets with overlap}
\end{figure}
Follow the partition methods of disjoint local data sets; we first split the whole data set to $M \times S$ subsets for $S$ different clients: each client has $M$ disjoint local subsets, and we put these $M$ subsets in sequence. For each client, we use a pointer to point to the start position of the current local data set. Figure ~\ref{Process of constructing local data sets with overlap} show the case when $M = 5$: $D_1 - D_5$ are $5$ disjoint local subsets of client $i$.
The pointer point to the start point of a local data set of the current round and the blue lines bound current subsets. At the beginning of training, pointers belong to each client will point to the beginning of the sequence. At the end of each round, the pointer moves $s$ steps, and if it moves to the end of all the sequence, it will come back to the start point.

In practice, the Cifar10 data set is partitioned to 210 subsets for 7 clients, and set the size of local data set $S_l = 285$. Then when we set $s = 285$, the local data sets are disjoint, which means the overlap is $0 \%$. Otherwise, when set $s = 213$, the overlap is $25 \%$, and when set $s = 142$, the overlap is $50\%$.

\subsubsection{Approximation Methods} \label{sec:Approximation Methods Appendix}

\textit{Regularization Methods.} We use Taylor Extension to approximate the objective functions of previous rounds, that is
\begin{talign*}
	\tilde{f}(\momega) = f(\hat{\momega}) + \nabla f(\hat{\momega})^T (\momega - \hat{\momega}) + \frac{1}{2} (\momega - \hat{\momega})^{T} \nabla^2 f(\hat{\momega}) (\momega - \hat{\momega}) \, ,
\end{talign*}
where $\hat{\momega}$ are the parameters of previous rounds. Then the gradient of $\tilde{f}(\momega)$ is
\begin{talign*}
	\nabla \tilde{f} (\momega) = \nabla f(\hat{\momega}) + \nabla^2 f(\hat{\momega}) (\momega - \hat{\momega}) \,.
\end{talign*}
Having this, we can approximate the gradients of objective functions of previous rounds.
The problem is how to approximate $\nabla^2 f(\hat{\momega})$.
Here we use two kinds of methods to calculate the Hessian matrices. The first is to use the Fisher Information Matrix as introduced in EWC~\citep{kirkpatrick2017overcoming}. The Fisher Information Matrix is equal to the diagonal Hessian matrix when using cross entropy loss functions. The Fisher Information Matrix can be calculated by
\begin{talign*}
	F_{ij} = [\nabla f(\hat{\momega})]_{ij}^{2} \,,
\end{talign*}
where $F_{ij}$ are the entries of Fisher Information Matrix.
Another method is to calculate the diagonal Hessian matrix. We use the PyHessian package~\citep{Yao2020pyhessian} in this experiment.

In each round, the server will collect Hessian matrices from clients, and store them in a buffer. Then at the beginning of next round, server combine send latest 40 Hessian matrices, gradients, and parameters, and send them to chosen clients. Each clients use these to calculate regularization terms.

In practice, we only add regularization terms to the top layers (last block and fc layer of ResNet18). This is based on the assumption that top layers contain more personal information while bottom layers contain more general information.
We verified that this strategy perform better than adding regularization terms to all layers in experiments.

\textit{Generation Methods.} We use MCMC to generate samples in each round. In practice, we initialize $\xx$ from uniform distribution, and update $\xx$ by
\begin{talign*}
	\xx^k = \xx^{k-1} - \eta \nabla_\xx E(\xx^{k-1}) + \momega \, ,
\end{talign*}
where $E(\xx^k-1)$ is the function that can measure the distance between $\xx^{k-1}$ and the real local distribution. $\momega \sim \cN (0, \sigma)$.

In practice, we generate 50-100 samples of each local data set and add them to the current local data sets for training.
However, because of the low quality of the generated data, the improvement is limited.
Besides, the generated data will pollute the batch normalization(BN) layer, and we should use real data to refresh BN layer at the end of each round.

\textit{Core Set Methods.}
Another simple yet effective treatment in CL lines in the category of Exemplar Replay~\citep{rebuffi2017icarl,castro2018end}.
This approach stores past core-set samples (a.k.a.\ exemplars) selectively and periodically, and replays them together with the current local data sets.
We tried two sample methods. First is so-called Naive method, in which the samples are uniformly chosen from local data sets.
Except of naive select core sets, we also tried another core set sampling method introduced in iCaRL~\citep{rebuffi2017icarl}. See Algorithm \ref{iCaRL Core Set} for details.

\begin{algorithm}[!t]
	\begin{algorithmic}[1]
		\Require{Image set $\mX = \{ \xx_1, \xx_2, \xx_3, ..., \xx_n \}$ of class $y$, $m$: target number of samples, $\phi: \cX \to \cR^d$: current feature function}
		\myState{$\mu \gets \frac{1}{n} \sum_{\xx \in \mX} \phi(\xx)$}
		\For{$k = 1, ..., m$}
		\myState{$p_k \gets \argmin_{\xx \in \mX} \norm{\mu - \frac{1}{k} \left( \phi(\xx) + \sum_{j=1}^{k-1} \phi(p_i) \right)}$}
		\EndFor
		\myState{$P \gets (p_1, p_2, ..., p_m)$}
	\end{algorithmic}
	\mycaptionof{algorithm}{\small iCaRL Construct Core Set}
	\label{iCaRL Core Set}
\end{algorithm}

In practice, we save 100 figures of each local data set and combine them with the current local data sets. Saving core sets perform best compared with these three approximation methods; however, only valid when the number of clients is limited.

\subsubsection{Additional Experiments} \label{sec:Additional Experiments}

\textit{Applicability of different algorithms.}
We list the applicability of different algorithms under various time-varying scenarios in Table~\ref{Applicable Scenarios for Different Algorithms}.

\begin{table*}[!t]
	\small
	\centering
	\resizebox{1.\textwidth}{!}{%
		\begin{tabular}{l c c c c c c}
			\toprule
			\multirow{2}{*}{Scenarios} & \multicolumn{4}{c}{FL baselines} & \multicolumn{2}{c}{CFL methods}                                                           \\
			\cmidrule(lr){2-5} \cmidrule(lr){6-7}

			                           & FedAvg                           & FedProx                         & SCAFFOLD & MimeLite & CFL-Regularization & CFL-Core-Set \\

			\midrule

			Stateful clients           & $\surd$                          & $\surd$                         & $\times$ & $\surd$  & $\surd$            & $\surd$      \\

			Stateless clients          & $\surd$                          & $\surd$                         & $\times$ & $\surd$  & $\surd$            & $\times$     \\

			\bottomrule
		\end{tabular}%
	}
	\caption{\small
		\textit{The applicability of various algorithms under different time-varying scenarios.}
	}
	\label{Applicable Scenarios for Different Algorithms}
\end{table*}

\textit{Overlapping local data sets.}
We further relax the difficulty of federated continual learning, from challenging non-overlapping time-evolving heterogeneous data (e.g.\ in Table~\ref{Ablation Study Table} and Table~\ref{Performance on Different data sets}) to a moderate time-evolving case (i.e.\ the local data evolves with the overlapping, while the size of local data sets stay unchanged).
Table \ref{Performance on overlap local data sets} illustrates the performance of FL baselines and CFL methods, under different degrees of overlapping (the construction details refers to Appendix~\ref{sec:realistic setup appendix}): \emph{the improvement of CFL methods is consistent to our previous results, while the overlap parameter has no obvious connection with the final global test performance.}
We believe that both the overlap degree and the new arriving data influence final performance, and we leave future work on realistic time-evolving FL data sets to gain a better understanding.\looseness=-1

\begin{table*}[!t]
	\small
	\centering
	\resizebox{.6\textwidth}{3.3em}{%
		\begin{tabular}{l c c c c c c}
			\toprule
			\multirow{2}{*}{Overlap} & \multicolumn{2}{c}{FL baselines} & \multicolumn{2}{c}{CFL methods}                                              \\
			\cmidrule(lr){2-3} \cmidrule(lr){4-5}

			                         & FedAvg                           & MimeLite                        & CFL-Core-Set          & CFL-Regularization \\

			\midrule
			$0\%$                    & $80.66 \pm 0.40$                 & $80.78 \pm 0.07$                & \bm{$84.87 \pm 0.11$} & $81.27 \pm 0.38$   \\

			$25\%$                   & $80.16 \pm 0.01$                 & $80.17 \pm 0.11$                & \bm{$84.66 \pm 0.11$} & $80.67 \pm 0.27$   \\

			$50\%$                   & $79.97 \pm 0.23$                 & $80.35 \pm 0.39$                & \bm{$84.59 \pm 0.07$} & $80.91 \pm 0.25$   \\

			\bottomrule
		\end{tabular}%
	}
	\caption{\small
		{Benchmarking FL baselines and CFL methods on different degrees of local data set overlapping}, for training ResNet18 on split-CIFAR10 data set.
		The overlap reduces the degree of non-iid-ness.
		In order to observe a noticeable performance difference, we use a larger data distribution gap between rounds (i.e.\ $\alpha = 0.1$).\looseness=-1
	}
	\label{Performance on overlap local data sets}
\end{table*}

\textit{Convergence curves for different settings.} Figure \ref{Convergence Curves on Different data sets} show convergence curves of CFL and \fedavg on different data sets. The settings are the same as results in Table \ref{Performance on Different data sets}.
Models are trained on partitioned data sets with $\alpha = 0.1$, and all data sets are partitioned to 210 subsets for 7 clients.
To show the difference between different algorithms more clearly, all curves are smoothed by a 1D-Mean-Filter. Results show CFL can converge to a better optimum compare with \fedavg.

\begin{figure*}[!t]
	\centering
	\subfigure[Performance on Fashion-MNIST]{ \includegraphics[width=.3\textwidth,]{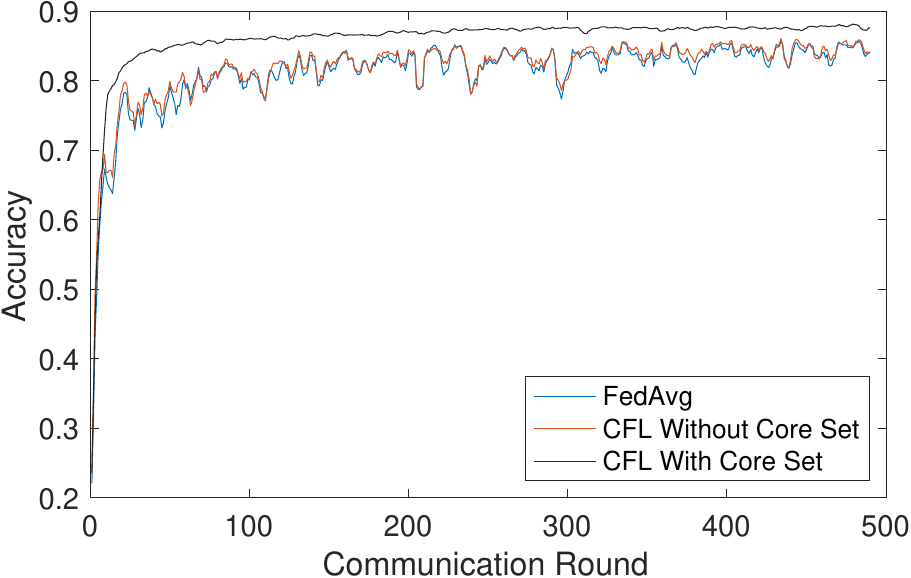}}
	\subfigure[Performance on Cifar10]{ \includegraphics[width=.3\textwidth,]{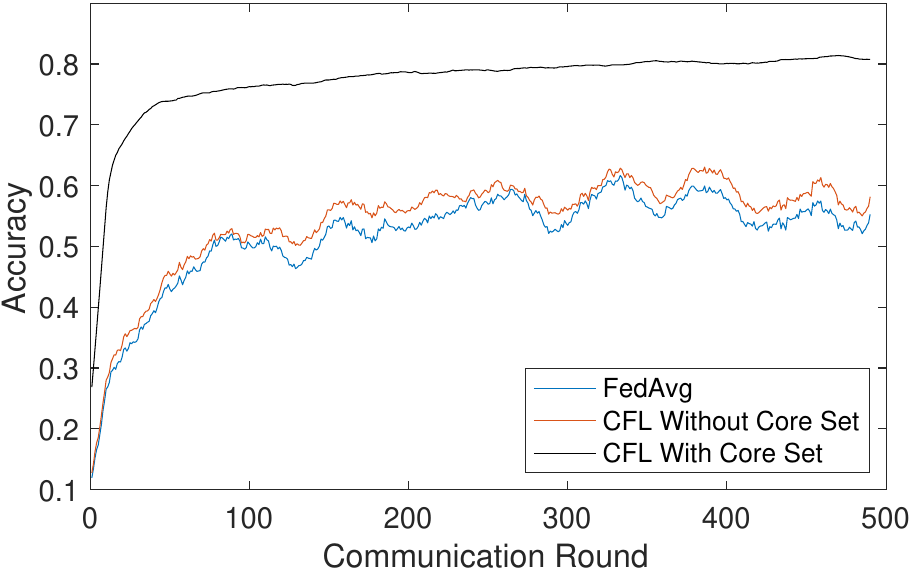}}
	\subfigure[Performance on Cifar100]{ \includegraphics[width=.3\textwidth,]{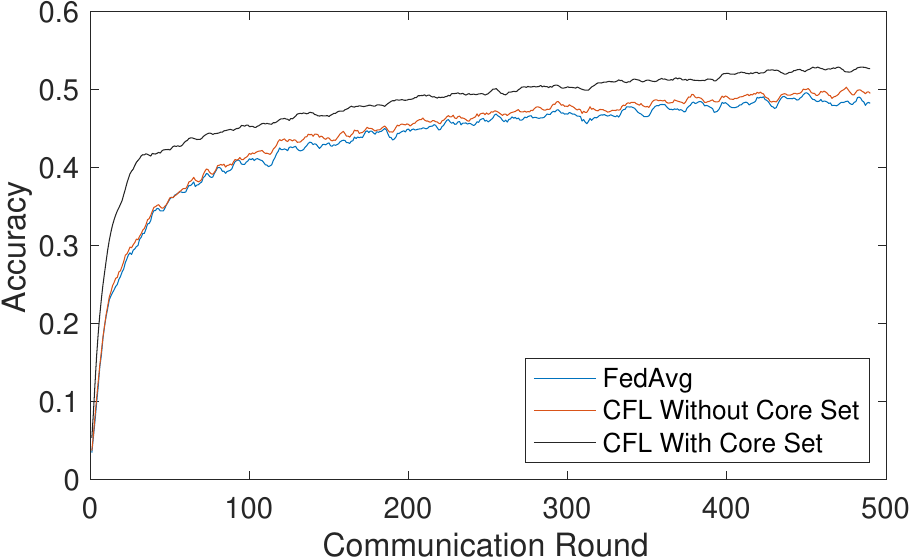}}
	\caption{Models are trained on various data sets with $\alpha = 0.1$. CFL Without Core Set method use regularization methods, and CFL With Core Set methods use core set methods. All these two CFL algorithms use \fedavg as backbone.}
	\label{Convergence Curves on Different data sets}
\end{figure*}

\textit{Investigating resistance of CFL to time-evolving scenarios.}
Figure \ref{Loss on Different data sets} shows the loss curve of CFL-Regularization and \fedavg on different data sets. Models are trained on partitioned data sets with $\alpha = 0.1$, and all data sets are partitioned to 210 subsets for 7 clients.
The first column shows the loss curve of the first 300 rounds, and the second column shows the loss of some chosen details. Because of the non-iidness of local data sets, the loss will suddenly rise. Notice that \emph{CFL-Regularization has an apparent mitigation effect on this situation}.

\textit{Difference between training and test loss.} Figure \ref{Difference between train and test loss} show the loss on global test data and past appeared training data. We evaluate the model with stateful clients, set $\alpha = 0.2$, and use \fedavg algorithm. We show that \emph{there is no significant difference between loss value on global test data and past appeared training data.}

\begin{figure*}[!t]
	\centering
	\subfigure[Loss on global test data and past appeared train data]{ \includegraphics[width=.3\textwidth,]{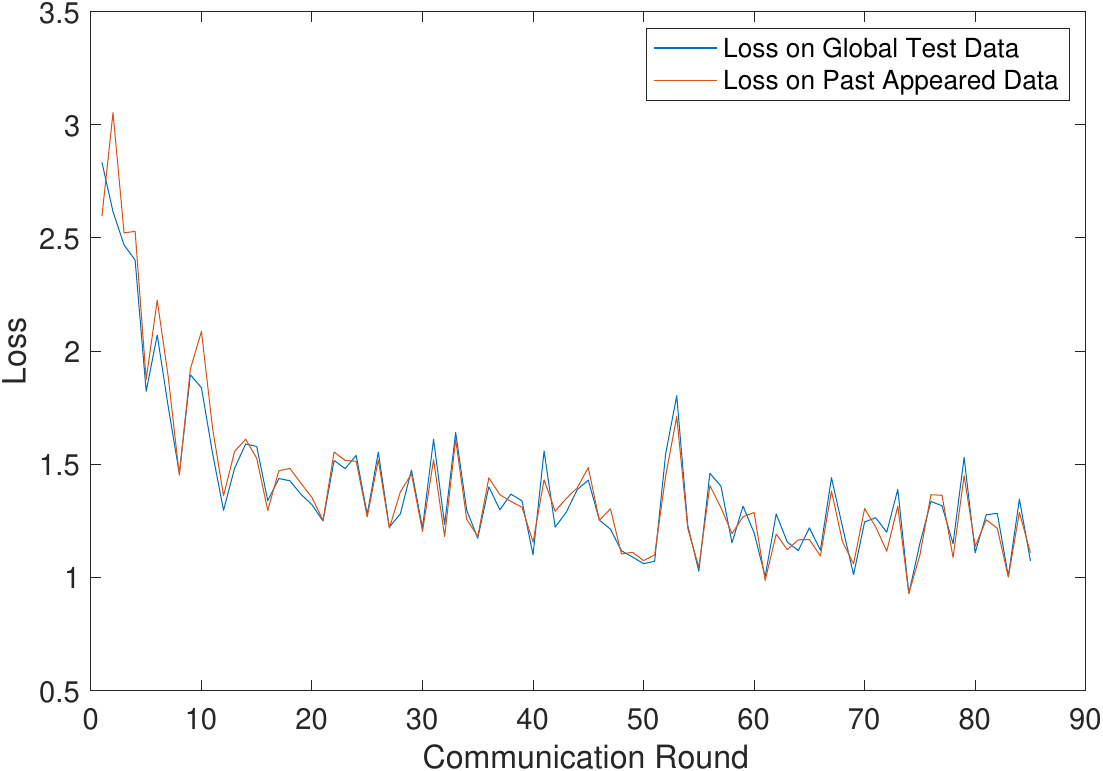}}
	\subfigure[Accuracy on global test data and past appeared train data]{ \includegraphics[width=.3\textwidth,]{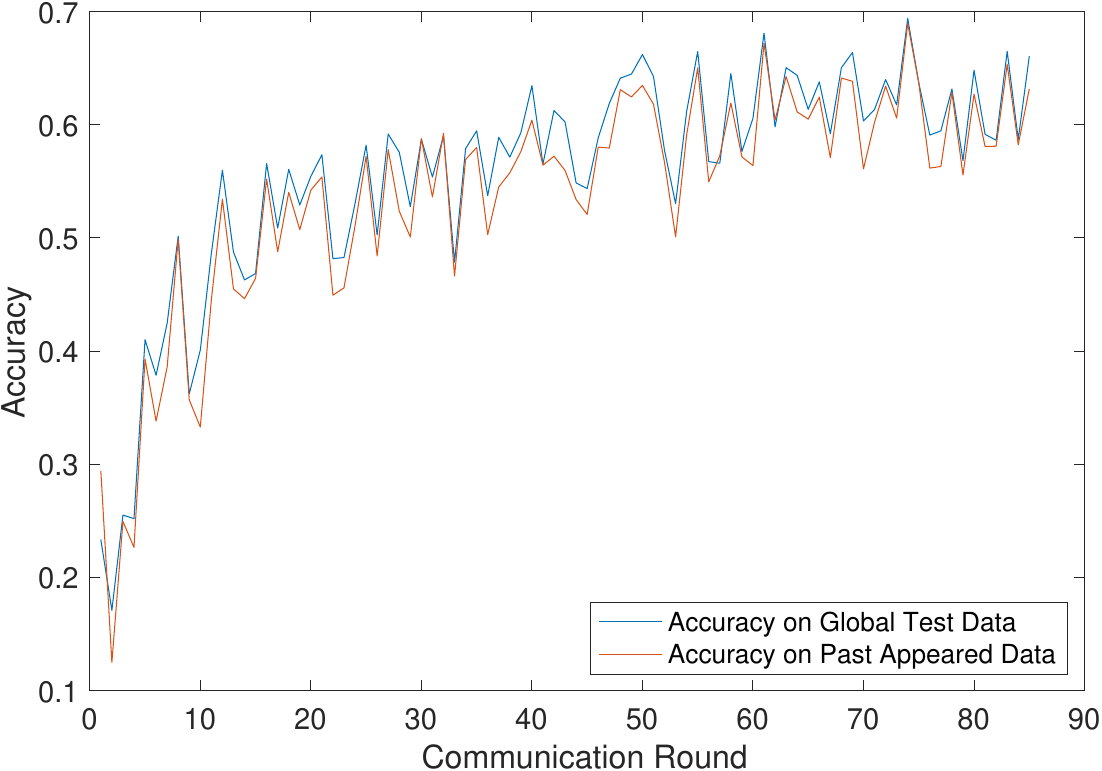}}
	\caption{Evaluation on global test data and past appeared training data. We trained ResNet18 on split-Cifar10 data set with $\alpha = 0.2$ for 85 rounds.}
	\label{Difference between train and test loss}
\end{figure*}

\begin{figure*}[!t]
	\centering
	\subfigure[Loss on Fashion-MNIST]{ \includegraphics[width=.23\textwidth,]{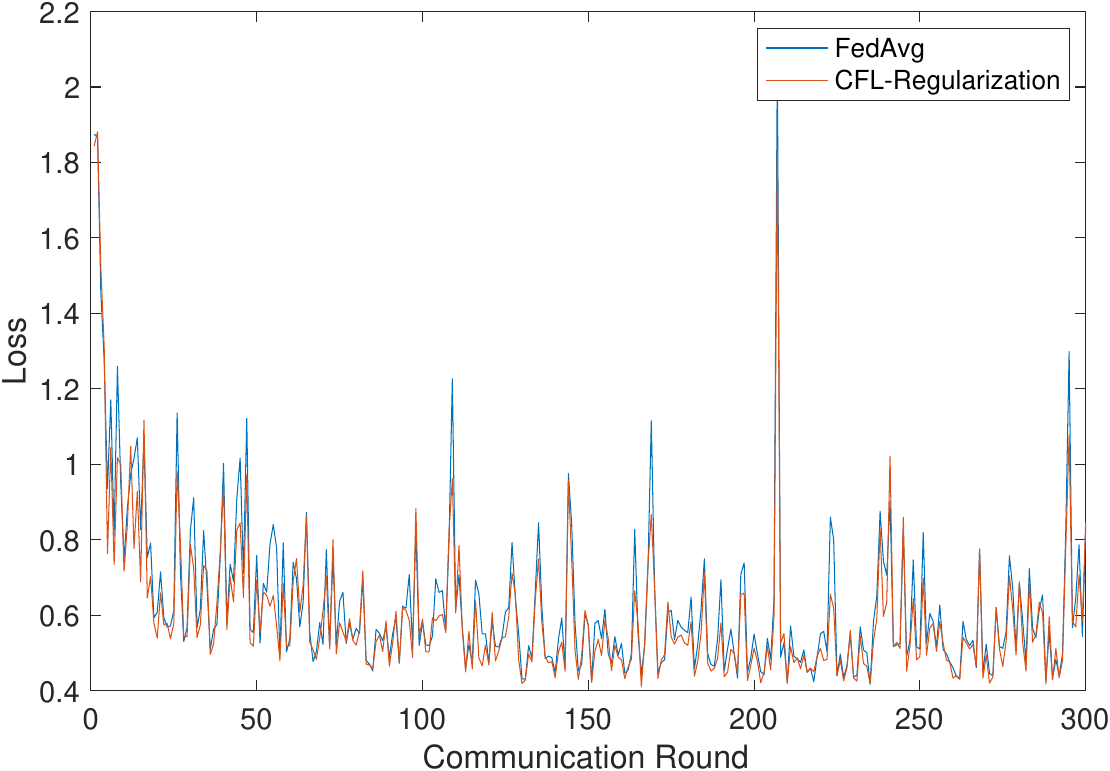}}
	\subfigure[Loss on Fashion-MNIST (Local)]{ \includegraphics[width=.23\textwidth,]{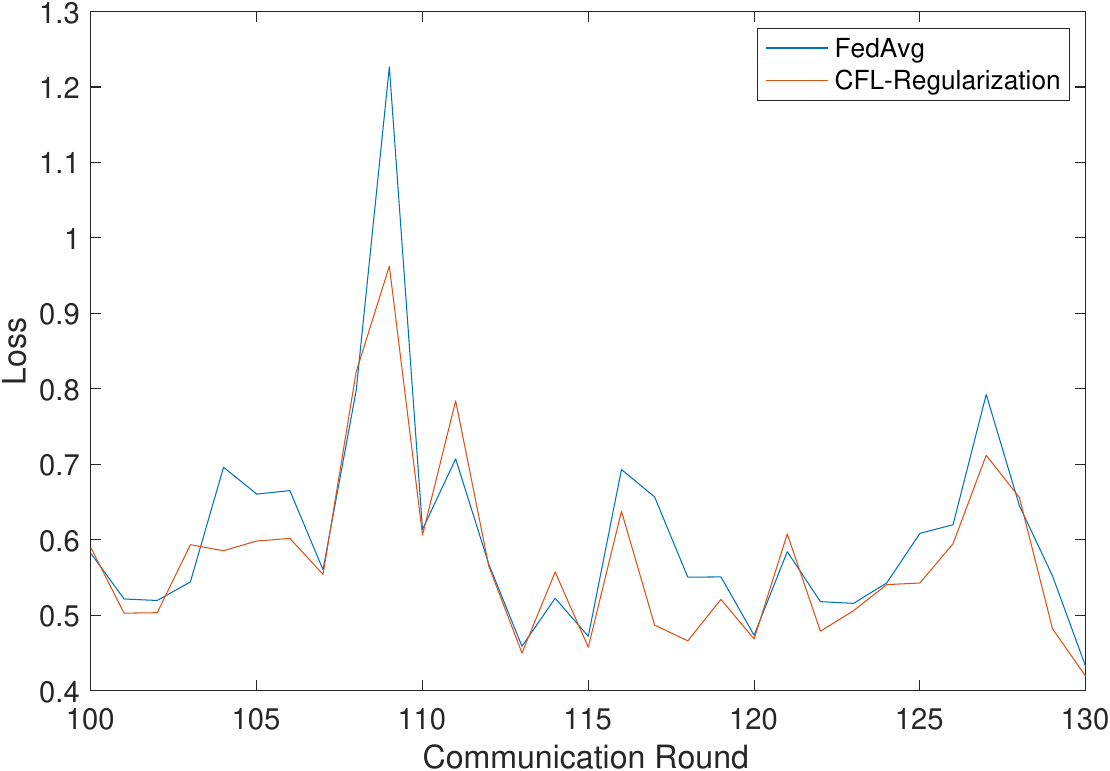}}
	\subfigure[Loss on Cifar10]{ \includegraphics[width=.23\textwidth,]{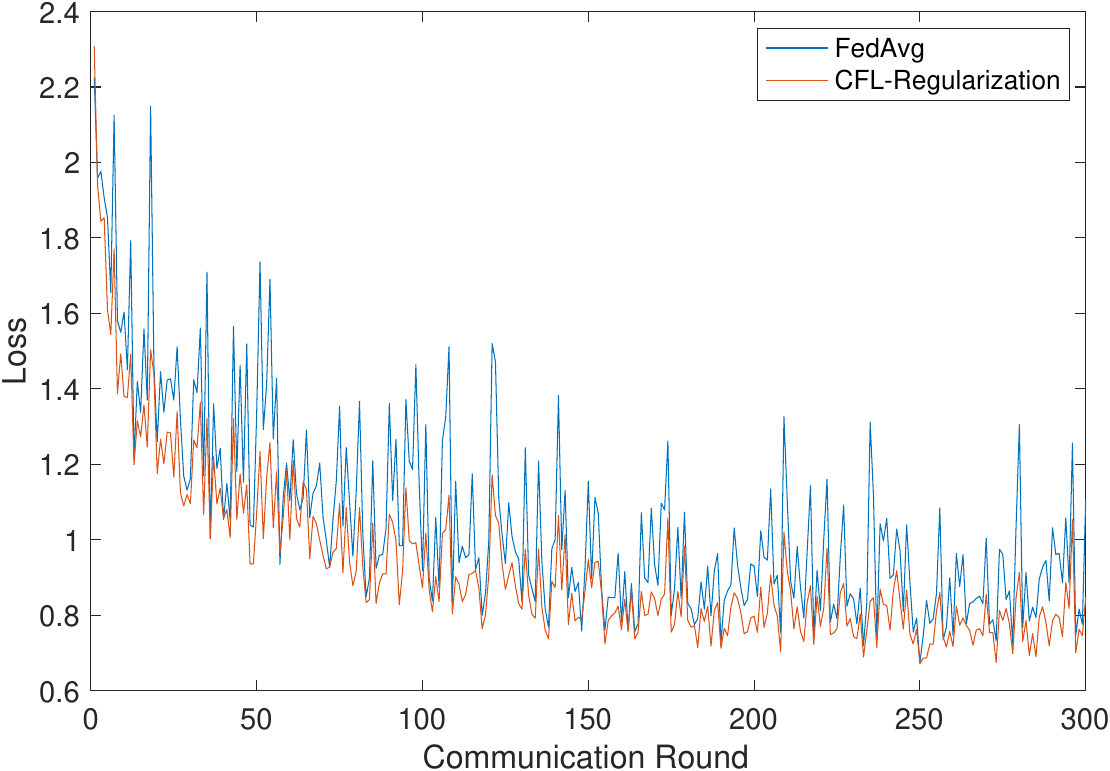}}
	\subfigure[Loss on Cifar10 (Local)]{ \includegraphics[width=.23\textwidth,]{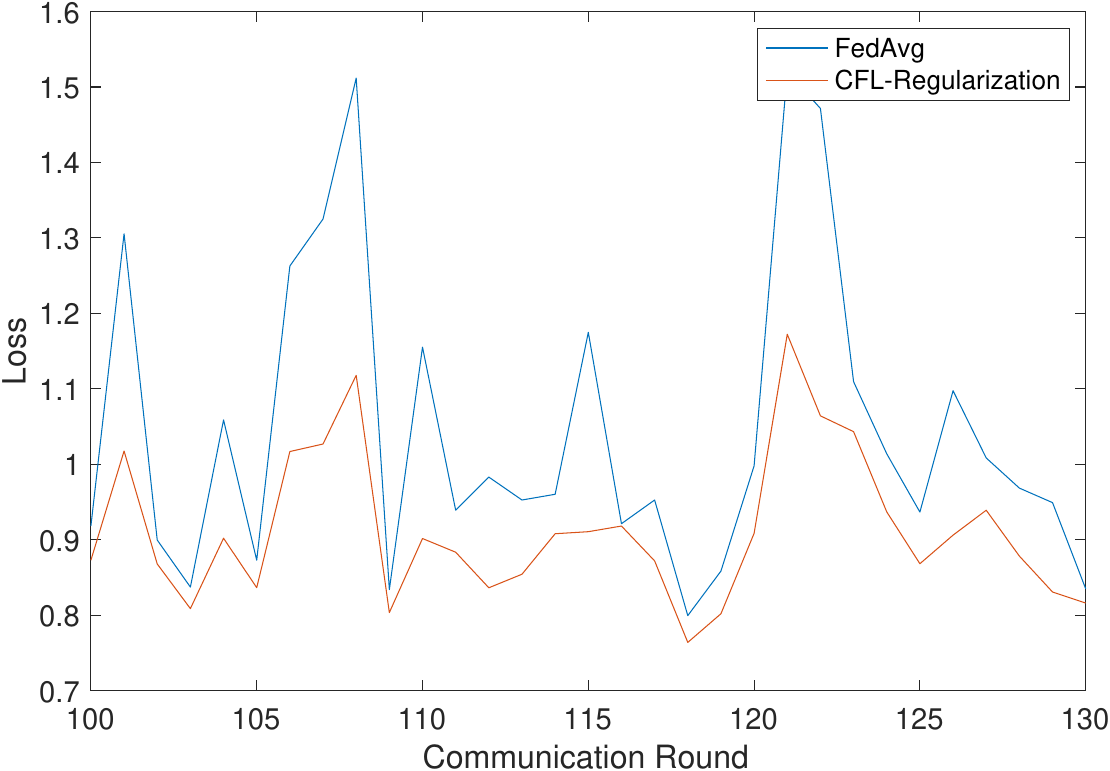}}
	\caption{Models are trained on split-Fashion-MNIST and split-Cifar10 data sets with $\alpha = 0.1$. The loss is evaluated on global test data sets. Left column is the full curve of 300 rounds, and figures in right column are partially enlarged curves.}
	\label{Loss on Different data sets}
\end{figure*}

\subsection{Algorithms}
\label{sec:algorithms}

\begin{algorithm}[!t]
	\begin{algorithmic}[1]
		\Require{$S$: total data set split by class, $T$: rounds, $K$, $N$, $\alpha$, $\beta$}
		\Ensure{$D_{final}$: split data set}

		\myState{$D \gets splitData(S, K, N, \alpha)$}
		\myState{$D_{final} \gets []$}
		\For{$i = 1, 2, ..., K$}
		\myState{$S_i \gets splitByClass(D[i])$}
		\myState{$N_{local} \gets len(D_m) / T$}
		\myState{$D_{i} \gets splitData(S_i, cluster\_num, N_{local}, \beta)$}
		\myState{$D_{final} \gets D_{final} \cup D_{i}$}
		\EndFor
	\end{algorithmic}
	\mycaptionof{algorithm}{\small SplitdataMain}
	\label{SplitdataMain}
\end{algorithm}

\begin{algorithm}[!t]
	\begin{algorithmic}[1]
		\Require{$\theta$: weights for different classes, $i$: the class that should be deleted}
		\Ensure{$\theta$: renormalized weights}

		\For{$j = 1,2,...,len(\theta), j \not = i$}
		\myState{$\theta[j] \gets \theta[j] / sum(theta / theta[i])$}
		\EndFor
	\end{algorithmic}
	\mycaptionof{algorithm}{\small renormalize}
	\label{renormalize}
\end{algorithm}

\begin{algorithm}[!t]
	\begin{algorithmic}[1]
		\Require{$S$: a list of data sets split by labels, $M$: the number of clients, $N$: the size of local data sets, $\alpha$: the Dirichlet distribution parameter}
		\Ensure{$D$: split data sets}

		\myState{$D \gets []$}
		\myState{$G \gets [0, 1, 2, ..., len(S)]$}
		\For{$m = 1, 2, ... , M$}
		\myState{$p = [p_1, p_2, ..., p_{len(S)}]$, where $p_{t, i}$ denotes the fraction of class $i$ in total data set.}
		\myState{$\theta \gets Dirichlet(\alpha, p)$}
		\myState{$D_m \gets \emptyset$}
		\While{$len(D_m) < N$}
		\myState{$i \gets multinomial(\theta, 1)$}
		\myState{$y \gets G[i]$}
		\myState{$data \gets uniform(S[y])$}
		\myState{$D_m \gets D_m \cup \{data\}$}
		\myState{$S[y] \gets S[y] / data$}
		\If{len(S[y]) == 0}
		\myState{$G \gets G / y$}
		\myState{$\theta \gets renormalize(\theta, i)$}
		\EndIf
		\EndWhile
		\myState{$D.append(D_m)$}
		\EndFor
	\end{algorithmic}
	\mycaptionof{algorithm}{\small Data splitting}
	\label{Data spliting algorithm}
\end{algorithm}




\end{document}